\documentclass{article}
\usepackage[utf8]{inputenc} 
\usepackage[T1]{fontenc}    
\usepackage{url}            
\usepackage{booktabs}       
\usepackage{amsfonts}       
\usepackage{nicefrac}       
\usepackage{microtype}      
\usepackage{xcolor}         
\usepackage{grffile}        
\usepackage[shortlabels]{enumitem}
\usepackage[round]{natbib}
\usepackage{multirow}

\ifdefined\NeurIPS  
\usepackage{hyperref}       
\else 
\usepackage[colorlinks=true,citecolor=blue,urlcolor=blue,linkcolor=blue]{hyperref}
\usepackage[margin=1in]{geometry}
\fi

\usepackage{amsmath}

\usepackage{amsthm}
\newtheorem{theorem}{Theorem}
\newtheorem{definition}{Definition}

\newtheorem{lemma}{Lemma}
\newtheorem{corollary}{Corollary}
\newtheorem{proposition}{Proposition}

\usepackage{dsfont}
\def\one{\mathds{1}}

\usepackage{amssymb}
\usepackage{makecell}
\usepackage{caption}
\usepackage{graphicx}

\usepackage{relsize}
\usepackage{subcaption}

\def\th{^\text{th}}

\def\T{\mathsf{T}}

\def\de{{\rm d}} 

\newcommand{\R}{\mathbb{R}} 
\renewcommand{\S}{\mathbb{S}} 
\def\ball{{\mathsf B}} 
\newcommand{\G}{\mathbb{G}} 

\newcommand{\rarrow}{\rightarrow}

\renewcommand{\P}{\mathbb{P}} 
\newcommand{\E}{\mathbb{E}} 
\newcommand{\Var}{\mathrm{Var}} 
\newcommand{\Cov}{\mathrm{Cov}} 
\newcommand{\simiid}{\overset{\text{i.i.d.}}{\sim}} 
\newcommand{\convd}{\overset{d}{\to}} 
\def\reals{\mathbb{R}}
\def\de{\mathrm{d}}

\def\cR{{\mathcal R}}
\def\cF{{\mathcal F}}
\def\cS{{\mathcal S}}
\def\cM{{\mathcal M}}
\def\cN{{\mathcal N}}

\def\cG{{\mathcal G}}

\def\cL{{\mathcal L}}

\def\bcdot{{\boldsymbol \cdot}}

\def\sf{\mathsf{f}}

\newcommand{\TV}{{\rm TV}}

\newcommand{\RTV}{{\mathrm{RTV}}}
\newcommand{\RBV}{{\mathrm{RBV}}}
\def\RTVk{\RTV^k}
\def\RBVk{\RBV^k}

\newcommand{\RadonDomain}{{\S^{d-1}\times\R}}
\newcommand{\nonnegRD}{{\S^{d-1}\times[0,\infty)}}
\newcommand{\zeroRD}{{\S^{d-1}\times\{0\}}}
\newcommand{\negativeRD}{{\S^{d-1}\times(-\infty, 0)}}
\newcommand{\positiveRD}{{\S^{d-1}\times(0, \infty)}}

\def\bracket{{[\;]}}

\newcommand{\popP}{P}
\newcommand{\popQ}{Q}
\newcommand{\empP}{P_m}
\newcommand{\empQ}{Q_n}

\newcommand{\Rd}{\Reals^d}
\newcommand{\Reals}{\mathbb{R}} 

\usepackage{algorithm}
\usepackage{algpseudocode}

\title{Integral Probability Metrics Meet Neural Networks: \\ The
  Radon-Kolmogorov-Smirnov Test}

\author{Seunghoon Paik$^{1}$, Michael Celentano$^{1}$, Alden Green$^{2}$,
  Ryan J. Tibshirani$^{1}$} 
\date{$^1$University of California, Berkeley, $^2$Stanford University}

\begin{document}
\maketitle

\begin{abstract}
Integral probability metrics (IPMs) constitute a general class of nonparametric
two-sample tests that are based on maximizing the mean difference between
samples from one distribution $\popP$ versus another $\popQ$, over all choices
of data transformations $f$ living in some function space $\cF$. Inspired by
recent work that connects what are known as functions of \emph{Radon bounded
variation} (RBV) and neural networks \citep{parhi2021banach, parhi2023near}, we
study the IPM defined by taking $\cF$ to be the unit ball in the RBV space of a
given smoothness degree $k \geq 0$. This test, which we refer to as the
\emph{Radon-Kolmogorov-Smirnov} (RKS) test, can be viewed as a generalization of
the well-known and classical Kolmogorov-Smirnov (KS) test to multiple dimensions
and higher orders of smoothness. It is also intimately connected to neural
networks: we prove that the witness in the RKS test---the function $f$ achieving
the maximum mean difference---is always a ridge spline of degree $k$, i.e., a
single neuron in a neural network. We can thus leverage the power of modern
neural network optimization toolkits to (approximately) maximize the criterion
that underlies the RKS test. We prove that the RKS test has asymptotically full
power at distinguishing any distinct pair $\popP \not= \popQ$ of distributions,
derive its asymptotic null distribution, and carry out experiments to elucidate
the strengths and weaknesses of the RKS test versus the more traditional kernel
MMD test.   
\end{abstract}

\section{Introduction}
\label{sec:intro}

In this paper, we consider the fundamental problem of nonparametric two-sample
testing, where we observe independent samples $x_i \sim \popP$, $i = 1,\ldots,m$
and $y_i \sim \popQ$, $i = 1,\ldots,n$, all samples assumed to be in $\Rd$, and
we use the data to test the hypothesis
\[
	H_0: \popP = \popQ, \quad \text{versus} \quad H_1: \popP \neq \popQ. 
\]
This is not only a problem of core interest in the traditional literature on
hypothesis testing in statistics, it also has a key role in more modern
applications in machine learning, such as out-of-distribution detection
\citep{hendrycks2017baseline}, transfer learning \citep{long2015learning},
generative modeling \citep{goodfellow2014generative}, among
others. \emph{Nonparametric} two-sample testing in particular refers to the
problem of testing $\popP = \popQ$ as above, without parametric assumptions on
the forms of the distributions $\popP, \popQ$. Although there are many
nonparametric two-sample tests, we will focus on a class of tests called
\emph{integral probability metrics} (IPMs) \citep{muller1997integral}, which
measure the discrepancy between two sets of samples by taking the maximum
difference in sample means over a function class $\cF$:
\begin{equation}
\label{eq:ipm}
\rho(\empP, \empQ; \cF) = \sup_{f \in \cF} \, | \empP(f) - \empQ(f) |. 
\end{equation}
Here \smash{$\empP = \frac{1}{m} \sum_{i = 1}^{m} \delta_{x_i}$} is the empirical
distribution and \smash{$\empP(f) = \frac{1}{m} \sum_{i = 1}^{m} f(x_i)$} the
corresponding empirical expectation operator based on $x_i$, $i = 1,\ldots,m$;    
and likewise for $\empQ$ and $\empQ(f)$. 

The discrepancy in \eqref{eq:ipm} is quite general and certain choices of
$\mathcal{F}$ recover well-known distances that give rise to two-sample tests. A
noteworthy univariate ($d = 1$) example is the \emph{Kolmogorov-Smirnov} (KS)
distance \citep{kolmogorov1933sulla, smirnov1948table}. Though it is more
typically defined in terms of cumulative distribution functions (CDFs) or ranks,
the KS distance can be seen as the IPM that results when $\mathcal{F}$ is taken
to be the class of functions with at most unit \emph{total variation} (TV).

Some attempts have been made to define a multivariate KS distance based on
multivariate CDFs \citep{bickel1969distribution} or ranks
\citep{friedman1979multivariate}, but the resulting tests have a few
limitations, such as being expensive to compute or exhibiting poor power in
practice. In fact, even the univariate KS test is known to be insensitive to
certain kinds of alternatives, such as those $\popP \neq \popQ$ which differ in
the tails \citep{bryson1974heavy}. This led \citet{wang2014falling} to recently
propose a higher-order generalization of the KS test, which uses the IPM based
on a class of functions whose \emph{derivatives} are of bounded TV.

In this work we propose and study a new class of IPMs based on a kind of
multivariate total variation known as \emph{Radon total variation} (RTV). The
definition of RTV is necessarily technical, and we delay it until Section
\ref{sec:radon-ks-test}. For now, we remark that if $d = 1$ then RTV reduces to 
the usual notion of total variation, which means that our RTV-based IPM---the
central focus of this paper---directly generalizes the KS distance to multiple
dimensions. We refer to this RTV-based IPM as the
\emph{Radon-Kolmogorov-Smirnov} (RKS) distance, and the resulting 
two-sample test as the RKS test. We also will consider the IPM defined using a 
class of functions whose derivatives are of bounded RTV, which in turn
generalizes the higher-order KS distance of \citet{wang2014falling} to multiple
dimensions. \citet{sadhanala2019higher} recently showed that the
higher-order KS test can be more sensitive to departures in the tails, and we
will give empirical evidence that our higher-order RKS test can be similarly
sensitive to tail behavior. 

Our test statistic also has a simple but revealing characterization in terms
of \emph{ridge splines}. A ridge spline of integer degree $k \geq 0$ is defined
for a direction \smash{$w \in \S^{d - 1}$} and an offset $b \in \R$ as the
truncated polynomial \smash{$f(x) = (w^\T x - b)_+^{k}$}. Here we use 
$\S^{d-1}$ to denote the unit $\ell_2$ sphere in $\R^d$ (the set of vectors with
unit $\ell_2$ norm), and we use the abbreviation $t_+ = \max\{t,0\}$. We prove
(Theorem \ref{thm:test-stat-is-relu-power}) that a $k\th$ degree ridge spline 
\emph{witnesses} the $k\th$ degree RKS distance: there exists a $k\th$ degree
ridge spline $f$ with unit Radon total variation for which $\empP(f) - \empQ(f)$
equals the RKS distance. Thus the RKS test statistic can be written as  
\begin{equation}
\label{eq:test-stat}
T_{d,k} = \max_{(w,b)\in\nonnegRD} \, \bigg| \frac{1}{m}\sum_{i = 1}^{m} 
(w^\T x_i -  b)_{+}^{k} - \frac{1}{n}\sum_{i = 1}^{n} (w^\T y_i - b)_{+}^{k}
\bigg|. 
\end{equation}
Ridge splines---and in particular the first-degree ridge spline $f(x) = (w^\T x 
- b)_+$, also known as the rectified linear (ReLU) unit---are fundamental
building blocks of modern neural networks. From this perspective, the
representation \eqref{eq:test-stat} says that the RKS distance is achieved by a
single neuron in a two-layer neural network. In fact, the connections between
Radon total variation and neural networks run deeper. The RKS distance can be
viewed as finding the maximum discrepancy between sample means over all
two-layer neural networks (of arbitrarily large width) subject to a weight decay
constraint (Section \ref{sec:radon-ks-test}). Although we focus throughout on
two-layer neural networks, the IPM defined over deeper networks can be viewed as 
the RKS distance in the representation learned by the hidden layers.

The representation \eqref{eq:test-stat} also hints at some interesting
statistical properties. Let $X \sim \empP$ and $Y \sim \empQ$ be random
variables drawn from the empirical distributions of \smash{$\{x_i\}_{i=1}^{m}$}
and \smash{$\{y_i\}_{i=1}^{n}$}, respectively. The $0\th$ degree RKS distance
finds the direction $w \in \S^{d - 1}$ along which the usual Kolmogorov-Smirnov
distance---maximum gap in CDFs---between the univariate random variables $w^\T
X$ and $w^\T Y$ is as large as possible. For $k \geq 1$, the $k\th$ degree
statistic finds the direction $w \in \S^{d-1}$ along which the higher-order KS
distance---maximum gap in truncated $k\th$ moments---between $w^\T X$ and
$w^\T Y$ is maximized.

Thus, speaking informally, we expect the RKS test to be particularly sensitive
to the case when the laws of $w^\T X$ and $w^\T Y$ differ primarily in
just a few directions $w$ (since we are taking a maximum over all such 
directions in \eqref{eq:test-stat}). In other words, we expect the RKS test to be
sensitive to \emph{anisotropic} differences between $\popP$ and $\popQ$. On the  
other hand, taking a larger value of $k$ results in higher-order truncated
sample moments, which can be much more sensitive to small differences in the 
tails. Figure \ref{fig:illustrative} gives an illustrative example.  We consider 
two normal distributions with equal means and covariances differing in just one 
direction (along the x-axis). The figure displays the witnesses (ridge splines)
for the RKS test when $k = 0,1,2$. For larger $k$, the witness is more aligned 
with the direction in which $\popP$ and $\popQ$ vary, and also places more
weight on the tails.  

\begin{figure}[htb]
\centering
\includegraphics[width=\textwidth]{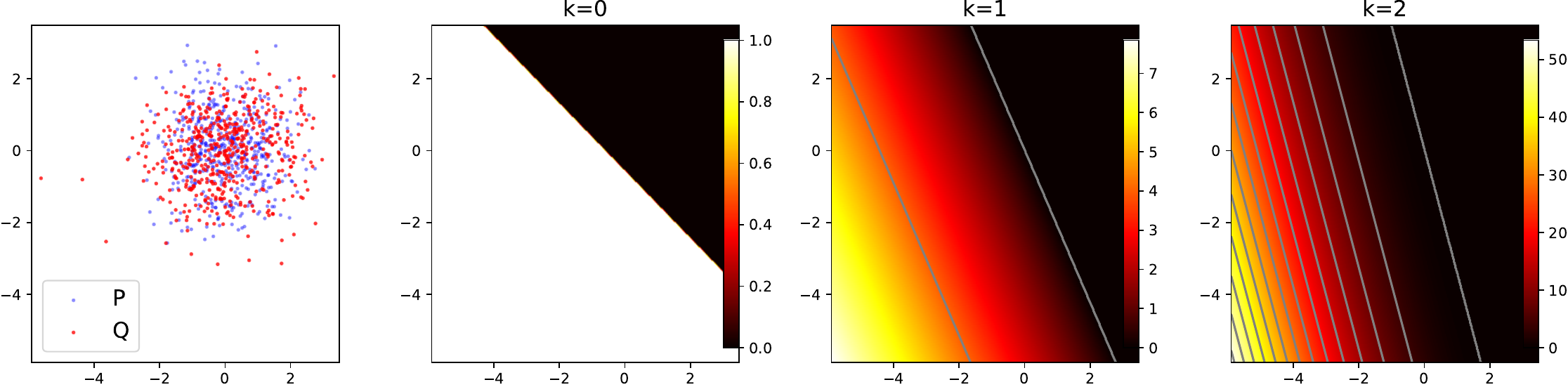}
\caption{Illustration of RKS tests for $\popP = \cN_2(0,I)$ and $\popQ = 
  \cN_2(0,D)$, where $D = \text{diag}(1.4, 1)$.}  
\label{fig:illustrative}
\end{figure}

\paragraph{Summary of contributions.} 

Our main contributions are as follows.

\begin{itemize}
\item We propose a new class of two-sample tests using a $k\th$ degree Radon 
  total variation IPM. We prove a representer theorem (Theorem
  \ref{thm:test-stat-is-relu-power}) which shows that the IPM is witnessed by a
  ridge spline (establishing the representation in \eqref{eq:test-stat}). This
  connects our proposal to two-sample testing based on neural networks, and
  helps with optimization.     

\item Under the null $\popP = \popQ$, we derive the asymptotic distribution of
  the test statistic (Theorem \ref{thm:null-distribution}).  

\item Under the alternative $\popP \neq \popQ$, we give upper bounds on the rate
  at which the test statistic concentrates around the population-level IPM.
  Using the fact that the population-level IPM is a metric, we then show that
  the RKS test is \emph{consistent} against any fixed $\popP \neq \popQ$: it has
  asymptotic error (sum of type I and type II errors) tending to zero (Theorem
  \ref{thm:asymptotic-power}).     

\item We complement our theory with numerical experiments to explore the
  operating characteristics of the RKS test compared to other popular
  nonparametric two-sample tests.  
\end{itemize}

\paragraph{Related work.}

When $\cF$ is the unit ball in a reproducing kernel Hilbert space (RKHS), the 
IPM in \eqref{eq:ipm} is known as a \emph{maximum mean discrepancy} (MMD).
Recently, there has been interest in studying multivariate nonparametric
two-sample testing using kernel MMDs \citep{gretton2012kernel}, as well as
energy distances \citep{baringhaus2004new, szekely2005new}. The latter are in
many cases equivalent to a kernel MMD \citep{sejdinovic2013equivalence}. It is
worth emphasizing that the Radon bounded variation space is \emph{not} an RKHS, 
and as such, our test cannot be seen as a special case of a kernel MMD. There
are also many other kinds of multivariate nonparametric two-sample tests, such
as graph-based tests using kNN graphs \citep{schilling1986multivariate,
henze1988multivariate} or spanning-trees \citep{friedman1979multivariate}. The
RKS test statistic \eqref{eq:test-stat} avoids the need to construct a graph
over the samples. However, we note that one could also use a graph to measure 
(approximate) multivariate total variation of a different variety (often called
\emph{measure-theoretic} TV, which is not the same as Radon TV), and one could
then form a related two-sample test, accordingly. 

From a different point of view, based on \eqref{eq:test-stat}, the RKS test can
be understood as a projection test which uses the higher-order KS test statistic   
\citep{wang2014falling, sadhanala2019higher} as its base univariate tool. In
general, projections are a natural way to lift univariate tools for statistical
modeling (whether in testing, estimation, or prediction) to multivariate data
settings. The literature is rich with contributions studying projection tests in
two-sample settings, including \citet{mueller2015principal,
  ghosh2016distribution, wei2016direction, paty2019subspace,
  lin2020projection, niles2022estimation}. Our paper complements this line of 
work by connecting projection tests to IPMs over bounded variation spaces,
which, to our knowledge, has not been done before.    

In machine learning, two-sample tests play a fundamental role in generative
adversarial networks (GANs) \citep{goodfellow2014generative}, which have proven
to be an effective way to generate new (synthetic) draws that adhere to the
distribution of a given high-dimensional set of samples. In short, GANs are
trained to produce synthetic data which a two-sample test cannot distinguish
from ``real'' samples. However, optimizing a GAN is notoriously unstable. It has
been suggested that the stability of GANs can be improved if the underlying
two-sample test is based on an IPM, with constraints on, e.g., the Lipschitz
constant \citep{arjovsky2017wasserstein}, RKHS norm \citep{li2017mmd}, or
Sobolev norm \citep{mroueh2018sobolev} of the witness. In practice, a neural
network is used to train the discriminator in the GAN, which only approximates
the witness (the function achieving the supremum) in IPMs over RKHS or Sobolev
spaces. Interestingly, for RBV spaces, this is not an approximation and a
two-layer neural network captures the witness exactly, as shown in
\eqref{eq:test-stat}.

At a technical level, our paper falls into a line of work exploring neural
networks from the functional analytic perspective. The equivalences between
weight decay of parameters in a two-layer neural network and first-degree ($k = 
1$) Radon total variation were discovered in \citet{savarese2019infinite, 
  ongie2020function}, and further formalized and extended by
\citet{parhi2021banach} to the higher-degree case ($k > 1$). We make use of
these equivalences to prove results on the Radon TV IPM. Lastly, we note that
the sensitivity of neural networks to variation in only a few directions was
observed even earlier, in \citet{bach2017breaking}.     

\section{The Radon-Kolmogorov-Smirnov test}
\label{sec:radon-ks-test}

For independent samples $x_i \sim \popP$, $i = 1,\ldots,m$ and $y_i \sim \popQ$, 
$i = 1,\ldots,n$, and a given integer $k \geq 0$, the RKS test statistic
\smash{$T_{d,k}$} is defined as in \eqref{eq:test-stat}. This statistic
searches for a marginal (defined by projection onto the direction $w$) with the
largest discrepancy in a truncated $k\th$ moment. As with any two-sample test,
we can use permutations to compute a finite-sample valid p-value,
\begin{equation}
\label{eq:test-perm}
p = \frac{\sum_{b=1}^{B} \one\big\{ T_{d,k}(z_{\pi_b}) \geq T_{d,k}(z) \big\} +  
  1}{B+1}.
\end{equation}
Here we use $z = (x_1,\ldots,x_m,y_1,\ldots,y_n) \in \R^{m+n}$ to denote the
original data sequence, and we use \smash{$T_{d,k}(z)$} to emphasize that the 
statistic in \eqref{eq:test-stat} is computed on $z$. Moreover, each $\pi_b$ is
a permutation of $\{1,\ldots,m+n\}$, drawn uniformly at random (from the set of
all possible permutations), and we use \smash{$T_{d,k}(z_{\pi_b})$} to denote
the test statistic computed on the permuted sequence \smash{$z_{\pi_b} = 
  (z_{\pi_b(1)},\ldots,z_{\pi_b(m+n)})$}. Given any user-chosen level $\alpha
\in [0,1]$, we reject $H_0: \popP = \popQ$ if $p \leq \alpha$. By well-known
results on permutation tests (see, e.g., \citet{hemerik2018exact}, and
references therein), this gives finite-sample type I error control, 
\smash{$\P_{H_0} (p \leq \alpha) \leq \alpha$}. 

Finding the value of $w$ and $b$ achieving the supremum in \eqref{eq:test-stat}
is a nonconvex problem. However, it bears a connection to neural network
optimization, which is of course among the most familiar and widely-studied
nonconvex problems in machine learning. We show in Corollary
\ref{cor:Tdk-by-N-neurons} that problem \eqref{eq:test-stat} can be equivalently 
cast as finding the maximum discrepancy between sample means over all two-layer
neural networks (of arbitrarily large width) under a weight decay-like
constraint. In particular, if we denote by \smash{$f_{(a_j,w_j,b_j)_{j=1}^N}$}
the $N$-neuron two-layer neural network defined by
\begin{equation}
\label{eq:f-param}
f_{(a_j,w_j,b_j)}(x) = a_j (w_j^\T x - b_j)_+^k
\quad \text{and} \quad
f_{(a_j,w_j,b_j)_{j=1}^N}(x) = \sum_{j=1}^N \, f_{(a_j,w_j,b_j)}(x),
\end{equation}
then the RKS test statistic is equivalently
\begin{equation} 
\label{eq:Tdk-by-N-neurons}
\begin{aligned}
T_{d,k} \,=\,
&\max_{f = f_{(a_j,w_j,b_j)_{j=1}^N}}
&&\empP(f) - \empQ(f) \\
&\;\;\;\; \text{subject to} \;\;\;
&& {\textstyle \sum_{j=1}^N \, |a_j| \|w_j\|_2^k \leq 1}, \; \text{and} \; b_j \geq 0, \; j = 1,\ldots,N.
\end{aligned}
\end{equation}
This equivalence holds for any number of neurons $N \geq 1$, which is why we say
that the network can be of ``arbitrarily large width''.

In Section \ref{sec:experiments}, we use a Lagrangian form of this constrained
optimization, allowing us to leverage existing deep learning toolkits to compute
the RKS distance. As with neural network optimization in general, we cannot
prove that any practical first-order scheme---such as gradient descent or any
its variants---is able to find the global optimum of
\eqref{eq:Tdk-by-N-neurons}, or its Lagrange form. However, we find empirically
that the test statistic resulting from first-order optimization has favorable
behavior in practice, such as stability (with respect to the choice of learning
rate) and strong power (with respect to certain classes of alternatives).
Further, we recall that in practice we would typically use the permutation
p-value \eqref{eq:test-perm}, and as long as the same first-order optimization
algorithm is applied to the original data sequence and to every permuted
sequence, this still provides finite-sample type I error control. In effect, we
can view this as \emph{redefining} the RKS test to be the output of the given
first-order algorithm we have at our disposal.

\subsection{Functions of Radon bounded variation}

The RKS distance, like the univariate KS distance, can be viewed as the maximum
mean discrepancy with respect to a function class defined by a certain notion of
smoothness. Whereas the univariate KS distance is the IPM over functions with
total variation is bounded by 1, the RKS distance is the IPM over functions with 
Radon total variation bounded by 1, subject to a boundary condition. This will
be shown a bit later; first, we must cover preliminaries needed to understand
Radon total variation. 

Heuristically, the $k\th$ degree Radon total variation of  $f: \reals^d
\to \reals$ measures the average smoothness of the ``marginals'' of 
$f$. To be more concrete, consider a function $f: \reals^d \to \reals$
which is infinitely differentiable, and whose derivatives of any order decay 
super-polynomially: \smash{$f \in C^\infty(\reals^d)$} and\smash{ $\sup_{x \in
    \reals^d} |x^\alpha \partial^\beta f(x)| < \infty$}. Here $\alpha =
(\alpha_1,\ldots,\alpha_d)$ and $\beta = (\beta_1,\ldots,\beta_d)$ are 
multi-indices (with nonnegative integer elements), and we use the abbreviations     
\[
x^\alpha = \prod_{i=1}^d x_i^{\alpha_i}, 
\quad \text{and} \quad
\partial^\beta f = \partial_{x_1}^{\beta_1} \cdots \partial_{x_d}^{\beta_d} f. 
\]
The set of such functions is called the \emph{Schwartz class}, and denoted by
$\cS(\reals^d)$. The \emph{Radon transform} of a function $f \in \cS(\reals^d)$
is itself another function $\cR\{f\} : \RadonDomain \to \reals$ defined as   
\[
\cR\{f\}(w,b) = \int_{w^\T x = b} f(x) \, \de x,
\]
where the integral is with respect to the Lebesgue surface measure on the
hyperplane $\{w^\T x = b\}$. As a function of $b$, the Radon transform can be
viewed as the marginal of the function $f$ in the direction $w$. The $k\th$
degree Radon total variation of $f$ is then defined to be
\citep{parhi2021banach, parhi2023near}:\footnote{A remark on notation and  
  nomenclature: we use \smash{$\RBVk$} to refer to the space which
  \citet{parhi2021banach, parhi2022ridge, parhi2023near} instead call
  \smash{$\RBV^{k+1}$}. That is, we index based on the \emph{degree} of the
  underlying ridge spline, whereas the aforementioned papers index based on the
  \emph{order}, traditionally defined as the degree plus one, $m = k+1$.}  
\[
\| f \|_{\RTVk} =  c_d \int_{\RadonDomain} \big|\partial_b^{k+1} \Lambda^{d-1} 
\cR\{f\}(w,b)\big| \, \de \sigma(w,b),  
\]
where $\sigma$ is the Hausdorff measure on $\RadonDomain$, \smash{$c_d =
  1/(2(2\pi)^{d-1})$}, and $\Lambda^{d-1}$ is the ``ramp filter'' defined by
\smash{$\Lambda^{d-1} = (-\partial_b^2)^{\frac{d-1}2}$} with fractional
derivatives interpreted in terms of Riesz potentials (see
\citet{parhi2021banach, parhi2023near} for details). Equivalently, if we define
an operator by \smash{$R_k\{f\}(w,b) = c_d\partial_b^{k+1}
  (-\partial_b^2)^{\frac{d-1}{2}} \cR\{f\}(w,b)$}, then we can write $k\th$
degree Radon TV of $f$ simply as \smash{$\| f \|_{\RTVk} = \| R_k \{ f \}
  \|_{L^1}$}. The mapping \smash{$f \mapsto \| f \|_{\RTVk}$} is a seminorm,
which we call the \smash{$\RTVk$} seminorm. 

\citet{parhi2021banach, parhi2023near} and \citet{parhi2022ridge} extend the
\smash{$\RTVk$} seminorm beyond Schwartz functions using functional analytic
tools based on duality. The space of functions with bounded \smash{$\RTVk$}
seminorm is referred to as the \smash{$\RBVk$} space (where RBV stands for
``Radon bounded variation''). After this extension, one can still think of
\smash{$\RTVk$} seminorm as measuring average higher-order smoothness of the 
marginals of $f$, but it does not require that all derivatives and integrals
exist in a classical sense. We refer the reader especially to
\citet{parhi2022ridge} for the complete functional analytic definition of the
\smash{$\RBVk$} space. For our purposes, it is only important that
\smash{$\RBVk$} functions enjoy the representation established by the following
theorem. Note that in order to define the \smash{$\RBVk$} space,
\citet{parhi2021banach, parhi2022ridge} must view it as a space of equivalence
classes of functions under the equivalence relation $f \sim g$ if $f(x) = g(x)$
for Lebesgue almost every $x \in \R^d$. We denote elements of \smash{$\RBVk$} as
$\sf$, and the functions they contain as $f \in \sf$.  

\begin{proposition}[Adaptation of Theorem 22 in \citet{parhi2021banach}; Theorem
  3.8 in \citet{parhi2022ridge}] 
\label{prop:Parhi-Nowak-repr} 
Fix any $k \geq 0$. For each $\sf \in \RBVk$, there exists a representative $f
\in \sf$ which satisfies for all $x \in \reals^d$
\[
  f(x) = \int_{\nonnegRD} (w^\T x-b)_+^k \, \de \mu(w,b) + q(x), 
\]
for some finite signed Borel measure $\mu$ on \smash{$\nonnegRD$} satisfying
\smash{$\| \mu \|_{\TV} = \| \sf \|_{\RTVk}$}, and some polynomial $q$ of degree
at most $k$, where we use \smash{$\| \cdot \|_{\TV}$} for the total variation
norm in the sense of measures.  
\end{proposition}

Proposition \ref{prop:Parhi-Nowak-repr} is the key result connecting the
\smash{$\RBVk$} space to two-layer neural networks. It states that any
\smash{$\RBVk$} functions is equal almost everywhere to the sum of a (possibly 
infinite-width) two-layer neural network and a polynomial. Furthermore, the
$\ell_1$ norm of coefficients in the last layer of this neural network
(formally, the total variation norm \smash{$\| \mu \|_{\TV}$} of the measure
$\mu$ defined over the coefficients) equals the \smash{$\RBVk$} seminorm of the
function in question. As stated, Proposition \ref{prop:Parhi-Nowak-repr} is a
slight refinement (and simplification) of results in \citet{parhi2021banach, 
  parhi2022ridge}, which we show in Appendix \ref{supp:Parhi-Nowak-repr}.  

\citet{parhi2021banach} consider the nonparametric regression problem defined 
by minimizing, over all functions $f$, the squared loss incurred by $f$ with
respect to a given finite data set plus a penalty on \smash{$\| f
  \|_{\RTVk}$}. They show this is solved by the sum of a \emph{finite-width} 
two-layer neural network and a polynomial, and hence provide a representation
theorem for \smash{$\RTVk$}-penalized nonparametric regression analogous to that
for kernel ridge regression \citep{wahba1990spline}. In a coming subsection, we 
will establish a similar connection in the context of two-sample testing: the
IPM under an \smash{$\RTVk$} constraint is realized by a single neuron.      

Before moving on, it is worth providing some intuition for the results just
discussed. First, the operator $R_k$ should be understood as transforming a 
function $f$ into the Radon domain: namely, it specifies $f$ via the
higher-order derivatives of its marginals. Then, the \smash{$\RTVk$} acts as the 
\smash{$L^1$} norm in the Radon domain. Lastly, the reason two-layer neural
networks and ridge splines emerge as the solutions to nonparametric regression 
and maximum mean discrepancy problems under an \smash{$\RTVk$} constraint is
that ridge splines are \emph{sparse} in the Radon domain.  In particular,
for \smash{$f(x) = (w^\T x - b)_{+}^{k}$}, we have \citep{parhi2021banach}:
\[
R_k\{f\} (w,b) = \frac{1}{2} \big( \delta_{(w,b)} + (-1)^{k+1} \delta_{(-w,-b)}
\big), 
\]
where \smash{$\delta_{(\pm  w,\pm b)}$} denote point masses at $(\pm w,\pm
b)$. Thus, \smash{$\RTVk$}-penalized optimization problems are solved by
two-layer neural networks due to the tendency of the \smash{$L^1$} norm to
encourage sparse solutions.      

\subsection{Pointwise evaluation of RBV functions}

Recall, the \smash{$\RBVk$} space contains equivalence classes of functions
whose members differ on sets of Lebesgue measure zero. Therefore, at face value,
point evaluation is meaningless for elements of the \smash{$\RBVk$} space.  This
poses a problem for defining an IPM with respect to RBV functions, since the IPM
depends on empirical averages, which require function evaluations over
samples. Fortunately, we show that there is a natural way to resolve point
evaluation over the \smash{$\RBVk$} space: each equivalence class in
\smash{$\RBVk$} has, possibly subject to a mild boundary condition, a unique
representative which satisfies a suitable notion of continuity. We can then
define point evaluation with respect to this representative.

For the case $k = 0$, we need to introduce a new notion of continuity, which we
call \emph{radial cone continuity}. For $k \geq 1$, we rely on the standard
notion of continuity. Radial cone continuity is defined as follows. 

\begin{definition}
A function $f$ is said to be \emph{radially cone continuous} if it is
continuous at the origin, and for any $x \in \reals^d \setminus \{0\}$, 
we have $f(x + \epsilon v) \to f(x)$ whenever (i) $\epsilon \to 0$ and 
(ii) $v \in \S^{d-1}$ and $v \to x / \|x\|_2$. 
\end{definition}

Radial cone continuity can be viewed as a generalization of right-continuity to
higher dimensions. Indeed a consequence of radial cone continuity is that for
every \smash{$v \in \reals^d$}, the restriction of the function $f$ to the ray
emanating from the origin in the direction $v$ is right-continuous (i.e., $t
\mapsto f(tv)$, $t > 0$ is right-continuous). The concept of radial cone
continuity, however, is stronger than right-continuity along rays---it
additionally requires continuity along any path which is tangent to and
approaches $x$ from the same direction as the ray that points from the origin to 
$x$. The following theorem shows that (subject to a boundary condition for the 
case $k = 0$) every equivalence class \smash{$\sf \in \RBVk$} contains a unique    
continuous representative when $k \geq 1$, and radially cone continuous
representative when $k = 0$.

\begin{theorem}
\label{thm:RBVk-continuous-representatives}
For $k = 0$, if \smash{$\sf \in \RBV^0$} contains a function that is continuous
at the origin, then it contains a unique representative that is radially cone
continuous. Moreover, for $k \geq 1$, each \smash{$\sf \in \RBVk$} contains a 
unique representative that is continuous and this representative is in fact
$(k-1)$-times continuously differentiable.  
\end{theorem}

Theorem \ref{thm:RBVk-continuous-representatives}, which as far as we can tell
is a new result for RBV spaces, is proved in Appendix
\ref{supp:RBVk-continuous-representatives}. This theorem motivates the definition
of the function spaces:   
\begin{equation}
\label{eq:rbv-cont}
\begin{aligned}
\RBV_c^0 &= \Big\{ f \,:\, \text{$f \in \sf$ for some $\sf \in \RBV^0$, $f$ 
  radially cone continuous} \Big\}, \\
\RBV_c^k &= \Big\{ f \,:\, \text{$f \in \sf$ for some $\sf \in \RBV^k$, $f$
  continuous} \Big\}, \quad k \geq 1.
\end{aligned}
\end{equation}

Unlike \smash{$\RBVk$}, the space \smash{$\RBV_c^k$} contains functions rather
than equivalence classes of functions, so that point evaluation is
well-defined. The Radon total variation of a member $f \in \sf$ is defined in
the natural way as the Radon total variation of the function class in
\smash{$\RBVk$} to which it belongs.

\subsection{The RKS distance is an IPM}

We are ready to show that the RKS distance is equivalent to the IPM over
functions $f$ with \smash{$\|f\|_{\RTVk} \leq 1$}, subject to a boundary 
condition. In this sense, it can be viewed as a generalization of the classical 
univariate KS distance, which, as we have mentioned, is the IPM over the space
of functions with total variation at most 1. Precisely, the RKS distance is the
IPM over the function space   
\begin{equation}
\label{eq:fk-space}
\cF_k = \Big\{ f \in \RBV_c^k \,:\, \|f\|_{\RTVk} \leq 1, \;
\text{$\partial^\alpha f(0)$ exists and equals 0 for all $|\alpha| \leq k$} 
\Big\}.   
\end{equation}
The boundary condition \smash{$\partial^\alpha f(0) = 0$} for all multi-indices
$\alpha = (\alpha_1,\ldots,\alpha_d)$ with \smash{$|\alpha| = \sum_{j=1}^d 
\alpha_j \leq k$} is needed to guarantee that the IPM is finite. We can
interpret this condition intuitively as requiring the polynomial $q$ to be zero
in the representation in Proposition \ref{prop:Parhi-Nowak-repr}. Otherwise, $q$
would be able to grow without bound under the constraint \smash{$\|f\|_{\RTVk}
  \leq 1$}, which we could then use to drive the IPM criterion to $\infty$,
provided that $\popP$ and $\popQ$ had any moment differences (in their first $k$
moments).     

The following result proves an important representation for functions in
\eqref{eq:fk-space}. Its proof is in Appendix \ref{supp:Fk-rep-poly0}.  

\begin{theorem}
\label{thm:Fk-rep-poly0}
Fix any $k \geq 0$. For any $f  \in \cF_k$, there exists a finite, signed Borel
measure $\mu$ on \smash{$\nonnegRD$} such that for all $x \in \reals^d$, 
\[
f(x) = \int_\nonnegRD (w^\T x-b)_+^k \, \de \mu(w,b),
\]
and \smash{$\|\mu\|_{\TV} = \|f\|_{\RTVk}$}. When $k = 0$, the measure $\mu$ may
be taken to be supported on \smash{$\positiveRD$}.
\end{theorem}

Note that the representation provided by Theorem \ref{thm:Fk-rep-poly0} is
similar to that in Proposition \ref{prop:Parhi-Nowak-repr}, except that the
former explicitly deals with point evaluation and sets the polynomial term to
$q = 0$, which is a consequence of the boundary condition imposed in the
definition of $\cF_k$. (Theorem \ref{thm:Fk-rep-poly0} also appears similar to
results in \citet{parhi2021banach}, but differs from their results for the same
reasons.) 

The next theorem establishes the equivalence between \smash{$T_{d,k}$} as
defined in \eqref{eq:test-stat} and the IPM over $\cF_k$. Its proof is in
Appendix \ref{supp:test-stat-is-relu-power}.   

\begin{theorem}
\label{thm:test-stat-is-relu-power}
Fix any $k \geq 0$. Define 
\[
\cG_k = \Big\{ (w^\T \cdot -\, b)_+^k \,:\, (w,b)\in\nonnegRD \Big\},   
\]
where by convention we take \smash{$t_+^0 = \one\{t \geq 0\}$}. Then for any 
$\popP$ and $\popQ$ with finite $k\th$ moments, we have 
\[
\rho(\popP, \popQ; \cF_k) = \rho(\popP, \popQ; \cG_k).
\]
In particular, this implies that for the empirical distributions $\empP$ and
$\empQ$ over any sets of samples \smash{$\{x_i\}_{i=1}^{m}$} and
\smash{$\{y_i\}_{i=1}^{n}$}, we have \smash{$\rho(\empP, \empQ; \cF_k) = 
  T_{d,k}$}, where the latter is as defined in \eqref{eq:test-stat}.
\end{theorem}

Lastly, as a consequence of the above result, we establish the representation
claimed earlier in \eqref{eq:Tdk-by-N-neurons}. Its proof is simple enough that 
we present it here.

\begin{corollary}
\label{cor:Tdk-by-N-neurons}
Fix any $k \geq 0$. For the empirical distributions $\empP$ and $\empQ$ over any
sets of samples \smash{$\{x_i\}_{i=1}^{m}$} and \smash{$\{y_i\}_{i=1}^{n}$}, the
RKS test statistic $T_{d,k}$ defined in \eqref{eq:test-stat} is equivalent to
that in \eqref{eq:f-param}, \eqref{eq:Tdk-by-N-neurons}, for any $N \geq 1$.
\end{corollary}

\begin{proof}
For $\|w\|_2 = 1$ and $b \in \R$, \citet{parhi2021banach} show that 
\smash{$\|(w^\T \cdot - \, b)_+^k\|_{\RTVk} = 1$}. Then, for any $w \in \R^d$
and $a \in \R$, by the \smash{$k\th$} order homogeneity of the ridge splines,  
\smash{$\|a (w^\T \cdot - \, b)_+^k\|_{\RTVk} = |a| \|w\|_2^k$}. By the
triangle inequality,  
\[
\big\| f_{(a_j, w_j, b_j)_{j=1}^N} \|_{\RTVk} \leq \sum_{j=1}^N | a_j | \|w_j
\|_2^k.
\]
Thus, defining
\[
\cF_{k,N} = \bigg\{ f_{(a_j, w_j, b_j)_{j=1}^N} \,:\, (a_j, w_j, b_j) \in
\R\times\R^d\times(0,\infty), \; \sum_{j=1}^N |a_j|\|w_j\|_2^k\leq 1 \bigg\},
\]
we have \smash{$\cF_{k,N} \subseteq \cF_k$}. Defining also \smash{$\cG_k^+ = 
  \{(w^\T \cdot - \, b)_+^k \,:\, (w,b) \in \positiveRD\}$}, we have
\smash{$\cG_k^+ \subseteq \cF_{k,N}$}, so  
\[
	\sup\nolimits_{f\in\cG_k^+} \, |\empP(f) - \empQ(f)|
	\leq \sup\nolimits_{f\in\cF_{k,N}} \, |\empP(f) - \empQ(f)|
	\leq \sup\nolimits_{f\in\cF_k} \, |\empP(f) - \empQ(f)|.
\]
The leftmost supremum is unchanged if we include the origin, i.e., take the
suprermum over $f \in \cG_k$, which precisely defines $T_{d,k}$. By Theorem
\ref{thm:test-stat-is-relu-power}, this is equal to \smash{$\sup_{f \in \cF_k}
  |\empP(f) - \empQ(f)|$}, the rightmost supremum above, so we conclude that all 
inequalities are equalities, thus \smash{$T_{d,k} = \sup_{f \in \cF_{k,N}}
  |\empP(f) - \empQ(f)|$}. Finally, we can remove the absolute value sign in the 
criterion because $\cF_{k,N}$ is symmetric ($f \in \cF_{k,N} \iff -f \in
\cF_{k,N}$), which completes the proof of \eqref{eq:Tdk-by-N-neurons} as an
equivalent representation.
\end{proof}

\subsection{The RKS distance identifies the null hypothesis}

An important property of the RKS distance for the purposes of two-sample testing 
is its ability to distinguish any two distinct distributions. In particular, we
have the following result, whose proof is in Appendix \ref{proof:rtv-ipm-metric}.

\begin{theorem}
\label{thm:rtv-ipm-metric}
For any $\popP,\popQ$ with finite $k\th$ moments, $\rho(\popP, \popQ; \cF_k) =
0$ if and only if $\popP = \popQ$, and is positive otherwise. 
\end{theorem}

Theorem \ref{thm:rtv-ipm-metric} states that the RKS distance, in the
population, perfectly distinguishes the null $H_0: \popP = \popQ$ from the
alternative $H_1 : \popP \neq \popQ$. Thus, the function space $\cF_k$ is rich
enough to make the IPM a \emph{metric} at the population level. As an
interesting contrast, we remark that this is not true of the kernel MMD with a
polynomial kernel of degree $k$, as studied in \citet{gretton2012kernel}. 
Indeed, if $\popP$ and $\popQ$ have the same moments up to order $k$ (where
``moments'' is to be interpreted in the appropriate multivariate sense) then the
kernel MMD in the population between $\popP$ and $\popQ$ is $0$. In other words,
we can see that the use of \emph{truncated} moments, as given by averages of
ridge splines, is the key to the discrepancy being a metric.         

\section{Asymptotics}

Given any probability distribution $\popP$ on $\reals^d$ with finite moments of
order $2k + \Delta$ for any fixed $\Delta > 0$, let us define the corresponding
Gaussian process \smash{$\G_{\popP} = \{ G_{w,b} : (w,b) \in \nonnegRD \}$} by 
\begin{equation}
\label{eq:gwb-process}
G_{w,b} \sim \cN\Big( 0, \, \E_{x \sim \popP} \big[ (w^\T x - b)_+^{2k} \big]
\Big), \quad
\Cov(G_{w,b}, G_{w',b'}) = \E_{x \sim \popP} \Big[ (w^\T x - b)_+^k (w'^\T x 
- b')_+^k \Big].  
\end{equation}
The importance of this Gaussian process is explained in the next result, which
shows that its supremum provides the asymptotic null distribution of the RKS
test statistic. 

\begin{theorem}
\label{thm:null-distribution}
Fix any $k \geq 0$. Assume that $\popP$ has finite moments of order $2k +
\Delta$, for any fixed $\Delta>0$. If $k = 0$, then additionally assume $\popP$  
is absolutely continuous with respect to Lebesgue measure, and satisfies
\begin{equation}
\label{eq:density-bounded}
\sup_{w \in \S^{d-1}} \big\| p_{w^\T x} \big\|_\infty <
\infty, 
\end{equation}
where \smash{$p_{w^\T x}$} denotes the density of $w^\T x$ for $x \sim P$, and
\smash{$\|f\|_\infty = \sup_{t \in \R} |f(t)|$} denotes the sup norm of
$f$. The condition \eqref{eq:density-bounded} holds when $\popP$ has a bounded
density and bounded support, but also holds for many distributions without
bounded support, such as the multivariate Gaussian. Under these conditions,
when $\popP = \popQ$, we have         
\[
\sqrt{\frac{mn}{m+n}} \, T_{d,k} \,\convd\, \sup_{(w,b) \in \nonnegRD} \, 
|G_{w,b}|, \quad \text{as $m,n \to \infty$}.
\]
\end{theorem}

Moreover, with a proper rejection threshold, the RKS test can have
asymptotically zero type I error and full power against any fixed $\popP \neq
\popQ$.  

\begin{theorem}
\label{thm:asymptotic-power}
Fix any $k \geq 0$. Assume that $\popP, \popQ$ each have finite moments of order
$2k + \Delta$, for any fixed $\Delta>0$. If $k = 0$, then additionally assume
$\popP, \popQ$ each have densities which satisfy \eqref{eq:density-bounded}. Let  
$t_{m,n}$ be such that $t_{m,n} \to 0$ and \smash{$t_{m,n} \sqrt{m+n} \to
  \infty$} as $m,n \to \infty$. Then the test which rejects when \smash{$T_{d,k}
  > t_{m,n}$} rejects with asymptotic probability 0 if $\popP = \popQ$ and
asymptotic probability 1 if $\popP \neq \popQ$, as $m,n \to \infty$.  
\end{theorem}

We prove Theorems \ref{thm:null-distribution} and \ref{thm:asymptotic-power} in
Appendices \ref{supp:uclt-prelim} and \ref{sec:proof-asymptotic-power}.
Roughly, Theorem \ref{thm:null-distribution} follows from general uniform
central limit theorem results in \citet{dudley2014uniform}, after controlling
the bracketing number of the function class $\cG_k$. Theorem
\ref{thm:asymptotic-power} follows from tail bounds that we establish for the
concentration of the RKS distance around its population value, combined with the
fact that the population distance is a metric (Theorem
\ref{thm:rtv-ipm-metric}).

Let us pause to discuss the above theorems and how they fit into our
understanding of the RKS test at large. While these results answer (what in our
view are) standard and yet fundamental questions about the asymptotic behavior
of the RKS test, neither really guides practical usage of the test. If the
asymptotic null distribution in Theorem \ref{thm:null-distribution} had been
\emph{pivotal}, i.e., not depending on $P$, then we could have used this (at
least in principle, ignoring computational issues) to approximately control type
I error in large samples. However, the asymptotic null appears to depend on $P$
except in the special case when $k = 0$ and $d = 1$, which recall, coincides
with classical KS test, whose asymptotic null distribution is well-known to be 
pivotal (defined in terms of the supremum of a Brownian bridge).

In practice, as discussed at the start of Section \ref{sec:radon-ks-test}, we
would typically use permutations to calibrate the RKS test, controlling type I 
error in finte-sample. Our next result shows that calibration via permutations
can indeed also achieve asymptotically zero type I error and full power against
any fixed $\popP \neq \popQ$.    

\begin{theorem}
\label{thm:asymptotic-perm} 
Fix any $k \geq 0$. Assume the conditions of Theorem \ref{thm:asymptotic-power}
on $\popP, \popQ$, and assume further that each has bounded support. Consider
the permutation p-value in \eqref{eq:test-perm}, for $B+1 = (m+n)!$, and let 
\smash{$\alpha_{m,n}$} be such that \smash{$\alpha_{m,n} \to 0$} and
\smash{$\log(1/\alpha_{m,n}) \sqrt{1/m+1/n}$} as $m,n \to \infty$. Then the test 
which rejects when \smash{$p > \alpha_{m,n}$} rejects with asymptotic
probability 0 if $\popP = \popQ$ and asymptotic probability 1 if $\popP \neq
\popQ$, as $m,n \to \infty$.      
\end{theorem}

The proof of Theorem \ref{thm:asymptotic-perm} is given in Appendix
\ref{app:asymptotic-perm}. It follows techniques developed in \citet{green2024two}. 

\section{Experiments}
\label{sec:experiments}

We conduct experiments to compare the power of the RKS test with other
multivariate nonparametric tests in multiple settings (different pairs of
$\popP$ and $\popQ$), both to investigate the strengths and weaknesses of the
RKS test, and to investigate its behavior as we vary the smoothness degree $k$.  

\subsection{Computation of the RKS distance}
\label{subsec:computation}

First, we discuss computation of the RKS distance, as initially defined in
\eqref{eq:test-stat}. This a difficult optimization problem because of the
nonconvexity of the criterion and the unit $\ell_2$ norm constraint on
$w$. Though we will not be able to circumvent the challenge posed by
nonconvexity entirely, we can ameliorate it by moving to the overparametrized
formulation in \eqref{eq:Tdk-by-N-neurons}: this now casts the same computation
as an optimization over the $N$-neuron neural network $f$ in \eqref{eq:f-param},
subject to a constraint on \smash{$\sum_{j=1}^N \, |a_j| \|w_j\|_2^k$}
(which is sometimes called the ``path norm'' of $f$). We call this problem 
``overparametrized'' since any individual pair $(w_j, b_j)$ would be
sufficient to obtain the global optimum in \eqref{eq:test-stat}, and yet we
allow the optimization to search over $N$ such pairs $(w_j, b_j)$, $j =
1,\ldots,N$.  Perhaps unsurprisingly, we find that overparametrization (larger
$N$) generally makes optimization easier---more stable with respect to the
initialization and choice of learning rate, discussed in more detail in
Appendix \ref{supp:sensitivity-N}. 

Problem \eqref{eq:Tdk-by-N-neurons} is still more difficult than we would like
it to be, i.e., not within scope for typical first-order optimization
techniques, due to the path norm constraint \smash{$\sum_{j=1}^N \, |a_j|
  \|w_j\|_2^k \leq 1$}. Fortunately, it is easy to see that any bound here
suffices to compute the RKS distance: we can instead use 
\smash{$\sum_{j=1}^N \, |a_j| \|w_j\|_2^k \leq C$}, for any $C>0$, rescale the
criterion $\empP(f) - \empQ(f)$ by $1/C$, and then the value of the criterion at
its optimum (the IPM value) is unchanged. This motivates us to instead solve a 
Lagrangian version of \eqref{eq:Tdk-by-N-neurons}:   
\begin{equation}
\label{eq:opt-problem}
\min_{\substack{f = f_{(a_j,w_j,b_j)_{j=1}^N} \\ b_j \geq 0, \, j = 1,\ldots,N}}
\, - \frac{1}{k} \log\bigg(\frac{1}{N} \big| \empP(f) - \empQ(f) \big| \bigg)   
+ \frac{\lambda}{kN} \sum_{i=1}^N |a_i| \|w_i\|_2^k,
\end{equation}
which is the focus in all of our experiments henceforth. The log transform for
the mean difference in \eqref{eq:opt-problem} is used because we find that it
leads to greater stability with respect to the choice of learning rate, which we
cover in Appendix \ref{supp:sensitivity-log}. Importantly, the choice of
Lagrangian parameter $\lambda > 0$ in \eqref{eq:opt-problem} can be arbitrary,
by the same rescaling argument as explained above: given any (local) optimizer
$f^*$ in \eqref{eq:opt-problem}, we can simply return the mean difference over
the rescaled witness \smash{$f^* / \|f^*\|_{\RTVk}$}, where
\smash{$\|f^*\|_{\RTVk} = \sum_{i=1}^N |a^*_i| \|w^*_i\|_2^k$}.    

For $k \geq 1$, we apply the \texttt{torch.optim.Adam} optimizer (a variation on
gradient descent), as implemented in PyTorch, to \eqref{eq:opt-problem}. We use 
a \texttt{betas} parameter $(0.9, 0.99)$, learning rate 0.5, number of
iterations $T = 200$, penalty parameter $\lambda = 1$, and number of neurons $N
= 10$. To enforce the nonnegativity condition on $b$, we project $b$ to
$[0,\infty)$ after each gradient step. Rather than take the last iterate, we
choose the maximal IPM values among the iterates (after rescaling by the
\smash{$\RTV^k$} seminorm of each iterate so that it lies in the unit seminorm   
ball). Further, we repeat this over three random initializations, and select the
best resultant IPM value to be the final output. While this optimization scheme
does not come with guarantees on convergence to a minimizer of the criterion
\eqref{eq:opt-problem}, we reiterate that if we calibrate this test statistic
using the permutation null (i.e., we use permutations to calculate a p-value as
in \eqref{eq:test-perm}, where in each permutation we apply the same
optimization scheme), then we are still guaranteed to control type I error in
finite-sample.   

For $k = 0$, such a first-order scheme is not applicable due to the fact that
the gradient of the $0\th$ degree ridge spline \smash{$(w^\T x - b)_+^0 = \one\{ 
  w^\T x \geq b \}$} (with respect to $w$ and $b$) is almost everywhere zero. As a
surrogate, we directly approximate the optimum $(w^*,b^*)$ in
\eqref{eq:test-stat} using logistic regression where the class labels identify
samples from $\popP$ versus $\popQ$, as implemented in
\texttt{sklearn.linear\_model.LogisticRegression} in Python. 

For concreteness, we summarize our computational approach below in Algorithm 
\ref{alg:rks-test-statistic}.

\begin{algorithm}
\caption{RKS test statistic} \label{alg:rks-test-statistic}
\begin{algorithmic}[]
  \State \textbf{Input:} Samples \smash{$\{x_i\}_{i=1}^m,  \{y_i\}_{i=1}^n
    \subseteq \R^d$}; smoothness degree $k \geq 0$; number of neurons $N \geq
  1$; number of iterations $T \geq 1$; learning rate $\eta > 0$; regularization
  parameter $\lambda > 0$  
  
  \smallskip
  \State \textbf{Output:} RKS test statistic $T_{d,k}$ 
  \smallskip
  \State \textbf{Run:}
  \State \underline{Step 1} (optional). Center the data to have sample mean
  zero:    
  \State \indent $\hat\mu = \frac{1}{m+n} ( \sum_{\ell=1}^m x_\ell +
  \sum_{\ell=1}^n y_\ell )$; 
  $x_i \leftarrow x_i-\hat\mu$;
  $y_i \leftarrow y_i-\hat\mu$

  \smallskip	
  \State \underline{Step 2} (optional). Scale the data to have average $\ell_2$ 
  norm of 1: 
  \State \indent $\hat{L}^2 = \frac{1}{m+n} (\sum_{\ell=1}^m \|x_\ell\|_2^2 +
  \sum_{\ell=1}^n \|y_\ell\|_2^2 )$; 
  $x_i \leftarrow x_i/\hat{L}$;
  $y_i \leftarrow y_i/\hat{L}$

  \smallskip
  \State \underline{Step 3}. Apply optimization procedure: 

  \smallskip
  \If {$k\geq 1$} 

  \State \indent Run Adam on the criterion in \eqref{eq:opt-problem} for $T$
  iterations with learning rate $\eta$ to compute $f^*$ 

  \smallskip
  \Else 
  \State \indent Run logistic regression on data $\{z_i\}_{i=1}^{m+n} =
  \{x_i\}_{i=1}^m \cup \{y_i\}_{i=1}^n$, with class labels indicating whether
  each $z_i \in \{x_\ell\}_{\ell=1}^m$ or $z_i \in \{y_\ell\}_{\ell=1}^n$. Given
  the solution $(w^*,b^*)$, define $f^*(x) = \one\{ w^{*\T} x \geq b \}$   
  \EndIf
		
  \smallskip
  \State \underline{Step 4}. Rescale $f^* / \|f^*\|_{\RTVk}$, where
  $\|f^*\|_{\RTVk} = \sum_{i=1}^N |a^*_i| \|w^*_i\|_2^k$, and compute   
  \[
  T_{d,k} = \big| \empP(f^*) - \empQ(f^*) \big|
  \]

  \smallskip
  \State \textbf{return} $T_{d,k}$
	\end{algorithmic}
\end{algorithm}

\subsection{Experimental setup}

We compare the RKS tests of various degrees across several experimental
settings, both to one another and to other all-purpose nonparametric tests: the
energy distance test \citep{szekely2005new} as well as kernel MMD
\citep{gretton2012kernel} with a Gaussian kernel. Python code to replicate all
of our experimental results is available at \url{https://github.com/100shpaik/}.      

We fix the sample sizes to $m = n = 512$ throughout, and study four choices of
dimension: $d=2,4,8,16$. For each dimension $d$, we consider five settings for
$\popP,\popQ$, which are described in Table \ref{table:experiments-setup}. In
each setting, the parameter $v$ controls the discrepancy between $\popP$ and
$\popQ$, but its precise meaning depends on the setting. The settings were
broadly chosen in order to study the operating characteristics of the RKS test
when differences between $\popP$ and $\popQ$ occur in one direction (settings
1--4), and in all directions (setting 5). Among the settings in which the
differences occur in one direction, we also investigate different varieties
(settings 1 and 2: mean shift under different geometries, setting 3: tail
difference, setting 4: variance difference). Figure \ref{fig:scatter-task}
visualizes samples from drawn each task in $d=2$ dimensions, whereas Figure 
\ref{fig:extreme-scatter-task} exaggerates the deviation between $\popP,\popQ$
(larger values of $v$) to better illustrate the geometry. Finally, we note that
since the RKS test is \emph{rotationally invariant},\footnote{Meaning,
  $\rho(\empP, \empQ; \cF_k)$ is invariant to rotations of the underlying
  samples \smash{$\{x_i\}_{i=1}^{m}$} and \smash{$\{y_i\}_{i=1}^{n}$} by any
  arbitrary orthogonal transformation $U \in \R^{d \times d}$. This is true
  because the $\RTV^k$ seminorm is itself invariant to rotations of the
  underlying domain, as can be seen directly from the representation in
  Proposition \ref{prop:Parhi-Nowak-repr}.}    
the fact that the chosen differences in Table \ref{table:experiments-setup} are 
axis-aligned is just a matter of convenience, and the results would not change 
if these differences instead occurred along arbitrary directions in $\R^d$.    

\begin{table}[htb]
\centering
\renewcommand{\arraystretch}{1.8}
\renewcommand{\arraycolsep}{1pt}
\begin{tabular}{|c||c|c|c|}
\hline
Setting & $\popP$ & $\popQ$ & $v$ \\
\hline
\makecell{Pancake mean shift \\ (``pancake-shift'')} & 
\makecell{one axis $\sim \cN(0, 1)$ \\ the others $\sim \cN_{d-1}(0, 9 I)$} & 
\makecell{one axis $\sim \cN(v, 1)$ \\ the others $\sim \cN_{d-1}(0, 9 I)$} &
0.3 \\ \hline
\makecell{Ball mean shift \\ (``ball-shift'')} & 
$\cN_d(0, I)$ &
\makecell{one axis $\sim \cN(v, 1)$ \\ the others $\sim \cN_{d-1}(0, I)$} &
0.2 \\ \hline
\makecell{$t$ coordinate \\ (``t-coord'')} & 
$\cN_d(0, I)$ & 
\makecell{one axis $\sim t(v)$ \\ the others $\sim \cN_{d-1}(0, I)$} & 
3 \\ \hline
\makecell{Variance inflation \\ in one direction \\ (``var-one'')} & 
$\cN_d(0, I)$ &
\makecell{one axis $\sim \cN(0, v)$ \\ the others $\sim \cN_{d-1}(0, I)$} &
1.4 \\ \hline
\makecell{Variance inflation \\ in all directions \\ (``var-all'')} &
$\cN_d(0, I)$ & $\cN_d(0, vI)$ & 1.2 \\
\hline
\end{tabular}
\caption{Experimental settings. Here $\cN_d(\mu, \Sigma)$ means the
  $d$-dimensional normal distribution with mean $\mu$ and covariance $\Sigma$, 
  and $t(v)$ means the $t$ distribution with $v$ degrees of freedom.} 
\label{table:experiments-setup}
\end{table}

\begin{figure}[H]
\centering
\begin{subfigure}[b]{0.19\textwidth}
  \includegraphics[width=\textwidth]{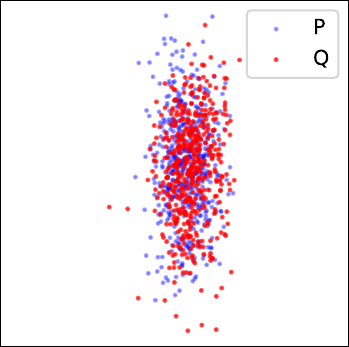}
  \caption{pancake-shift}
\end{subfigure}
\hfill
\begin{subfigure}[b]{0.19\textwidth}
  \includegraphics[width=\textwidth]{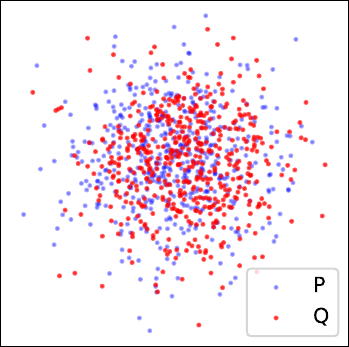}
  \caption{ball-shift}
\end{subfigure}
\hfill
\begin{subfigure}[b]{0.19\textwidth}
  \includegraphics[width=\textwidth]{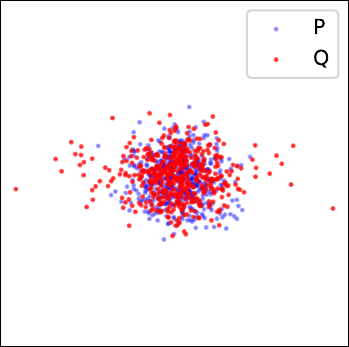}
  \caption{t-coord}
\end{subfigure}
\hfill
\begin{subfigure}[b]{0.19\textwidth}
  \includegraphics[width=\textwidth]{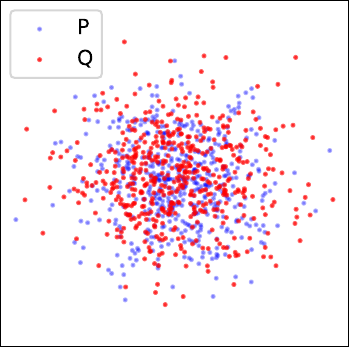}
  \caption{var-one}
\end{subfigure}
\hfill
\begin{subfigure}[b]{0.19\textwidth}
  \includegraphics[width=\textwidth]{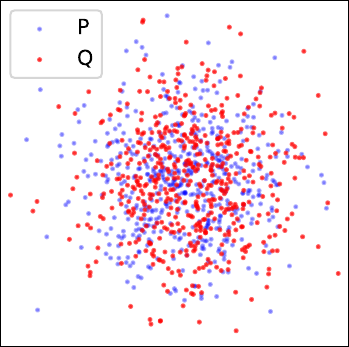}
  \caption{var-all}
\end{subfigure}
\caption{Samples drawn from the settings in Table \ref{table:experiments-setup}, 
  with $d = 2$.}
\label{fig:scatter-task} 

\vspace{15pt}
\begin{subfigure}[b]{0.19\textwidth}
  \includegraphics[width=\textwidth]{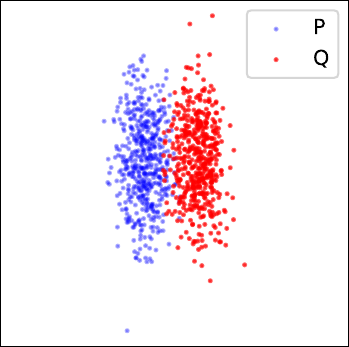}
  \caption{pancake-shift}
\end{subfigure}
\hfill
\begin{subfigure}[b]{0.19\textwidth}
  \includegraphics[width=\textwidth]{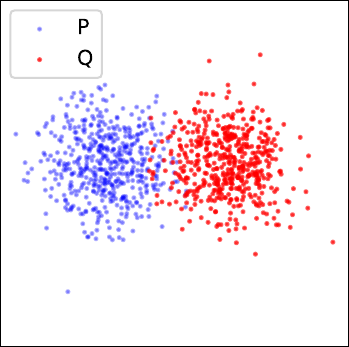}
  \caption{ball-shift}
\end{subfigure}
\hfill
\begin{subfigure}[b]{0.19\textwidth}
  \includegraphics[width=\textwidth]{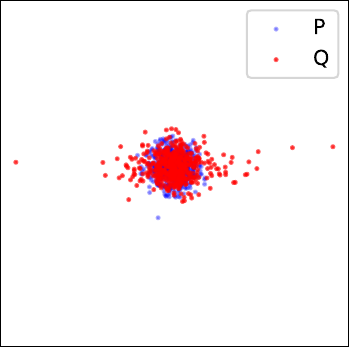}
  \caption{t-coord}
\end{subfigure}
\hfill
\begin{subfigure}[b]{0.19\textwidth}
  \includegraphics[width=\textwidth]{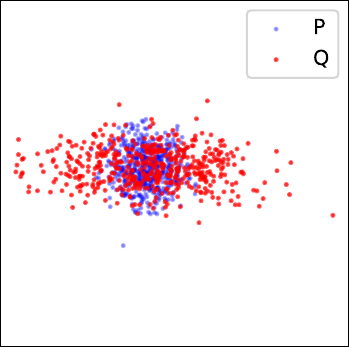}
  \caption{var-one}
\end{subfigure}
\hfill
\begin{subfigure}[b]{0.19\textwidth}
  \includegraphics[width=\textwidth]{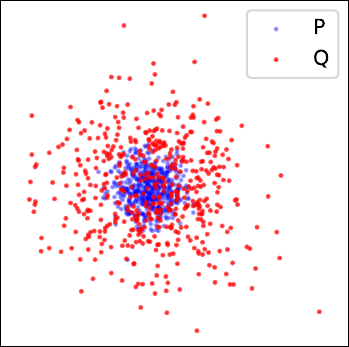}
  \caption{var-all}
\end{subfigure}
\caption{Samples drawn from the settings in Table \ref{table:experiments-setup}, 
    with $d = 2$, and exeggerated values of $v$.}
  \label{fig:extreme-scatter-task} 
\end{figure}

\clearpage 

For the RKS tests, we examine smoothness degrees $k = 0,1,2,3$, and we center  
the input data to have the sample mean zero jointly across both samples. That
is, each $x_i$ is replaced by \smash{$x_i - \frac{1}{m+n} \big(\sum_{\ell=1}^m
  x_\ell + \sum_{\ell=1}^n y_\ell\big)$}, and likewise for $y_i$. This makes the RKS test
\emph{translationally-invariant}. For kernel MMD, we set the Gaussian kernel
bandwidth according to a standard heuristic, based on the median of all pairwise 
squared distances between the input samples \citep{caputo2002appearance}. (The
energy distance test has no tuning parameters.) For each setting, we compute
these test statistics under the null where each $x_i$ and $y_i$ are sampled
i.i.d.\ from the mixture \smash{$\frac{m}{m+n} \popP + \frac{n}{m+n} \popQ$},
and under the alternative where $x_i$ are i.i.d.\ from $\popP$ and $y_i$ from
$\popQ$. We then repeat this 100 times (draws of samples, and computation of
test statistics), and trace out ROC curves---true positive versus false positive
rates---as we vary the rejection threshold for each test. 

In Appendix \ref{supp:mmd-poly}, we also compare the RKS test to kernel MMD with
a polynomial kernel. An important remark is that the latter is not actually a 
nonparametric test, for any finite polynomial degree, because the population IPM
does not define a metric (it cannot distinguish distributions $\popP, \popQ$
that agree up to some number of moments, interpreted in the multivariate sense,
but differ otherwise). We include comparisons to polynomial kernels for
completeness nonetheless.  

\subsection{Experimental results}

\begin{figure}[p]
\vspace{-10pt}
\includegraphics[width=\textwidth]{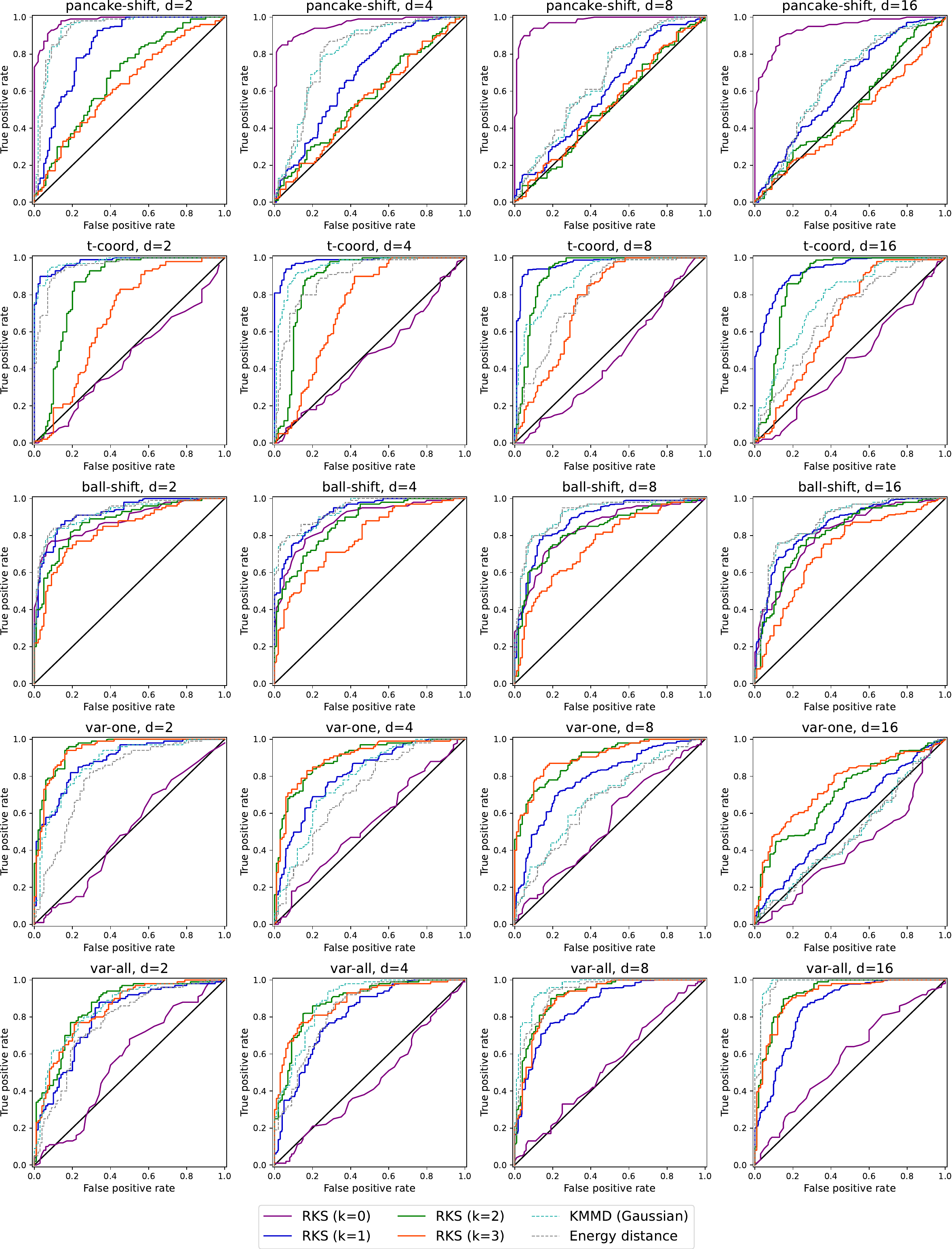}
\caption{ROC curves across the experimental settings described in Table
  \ref{table:experiments-setup}. Each row represents a different setting, and
  each column a different dimension. The RKS, kernel MMD (``KMMD'') with a
  Gaussian kernel, and energy distance tests are compared. Each is coded by a 
  combination of color and line type.}     
\label{fig:roc-setup-per-dimension}
\end{figure}

Figure \ref{fig:roc-setup-per-dimension} displays the ROC curves from all tests
across all settings (rows) and all dimensions (columns). We see that in the
first four settings (first four rows), the RKS tests exhibit generally strong
performance: in each case, one of the $k\th$ degree RKS tests performs either
the best or near-best among the methods considered. This supports the intuition
discussed earlier, because in each of these settings, $\popP$ and $\popQ$ differ 
in only a small number (in fact, one) direction. Conversely, in the last
setting, ``var-all'', the distributions $\popP$ and $\popQ$ differ by a small
amount in all directions, and the Gaussian kernel MMD and energy distance tests
outperform the RKS tests, especially for larger $d$. 
  
Recalling that larger choices of $k$ correspond to comparing higher-order
truncated moments, we can look further into the first four tasks (and their
underlying data models) in order to try to learn something about why and when
different degree RKS tests perform better or worse. In the ``pancake-shift''
setting, the mean shift occurs in a low-variance direction, which leads to
nearly separable samples. The smallest choice $k = 0$ performs best (and also
better than kernel-based tests). In the ``var-one'' and ``t-coord'' settings,
$\popP,\popQ$ have the same mean, but different variance or tails, respectively,
in one direction. We see that larger choices of $k$ perform better: $k = 1$ for
``t-coord'', and $k = 2,3$ for ``var-one''. In the ``ball-shift'' setting, which
is a mean shift between isotropic Gaussians, all degrees $k = 0,1,2$ are
reasonably competitive.

As the dimension $d$ increases (across the columns), the performance of each
test generally gets worse, except for ``var-all'', which will be explained
shortly. However, when performance decreases, it does so at a rate which depends
on the task and the test. For example, in the ``pancake-shift'' task, the
performance of the RKS test with $k = 0$ decreases quite slowly as $d$
increases, whereas all other methods quickly lose power. (In experiments not
shown, RKS with $k = 0$ continues to have nontrivial power in this setting even
when $d$ is in the hundreds.) On the other hand, in the ``var-one'' task, all
tests---including the leading RKS tests with $k = 2,3$---noticeably lose power
by the last column, $d = 16$. Finally, in ``var-all'', the problem actually gets 
\emph{easier} as $d$ increases, recalling that $v$ parametrizes a variance
inflation in each of the $d$ total directions. Indeed, we can see that the
kernel MMD and energy distance tests improve as $d$ increases.  

\subsection{Aggregating RKS tests}

Now we discuss how to aggregate the RKS test statistics, across various degrees
$k$. In an ideal case, such an aggregate RKS test could leverage the strengths
of different degrees for different settings, and have more robust power than any
individual degree itself. 

Suppose we consider RKS tests for degrees $k = 0,\ldots,K$, for some choice $K
\geq 0$ of maximum degree; in most practical scenarios, it seems reasonable to
choose to $K = 3$. A classical approach for aggregating these RKS tests would be 
Bonferroni correction, which takes p-values $p_k$, $k = 0,\ldots,K$, and returns   
\[
p_{\min} = (K+1) \cdot \min_{k=0,\ldots,K} \, p_k 
\]
as an aggregate p-value. The controls type I error by the union bound:
\smash{$\P_{H_0} ((K+1) \cdot \min_{k = 0,\ldots,K} p_k \leq \alpha) \leq$}
\smash{$\sum_{k=0}^K \P_{H_0} (p_k \leq \alpha / (K+1)) \leq \alpha$}. 

A variety of other aggregation strategies are available; for example,
\[
p_{\mathrm{avg}} = \frac{2}{K+1} \sum_{k=0}^K p_k 
\]
is also a valid p-value. Table 1 of \citet{vovk2020combining} lists a number
of other examples. Interestingly, if the p-values for individual RKS tests are
computed via permutations (as have advocated for generally in this paper), then
we can take any aggregation function $G$, and calibrate the corresponding
p-value   
\[
G(p_0,\ldots,p_K)
\]
itself via permutations at no additional computational cost, provided the 
same permutations $\pi_b$, $b = 1,\ldots,B$ were used to compute each 
p-value $p_k$, and these permutations form a group. This can be explained as
follows (similar ideas appear in \citet{solmi2014combining}). Abbreviating
\smash{$T(b,k) = T_{d,k}(z_{\pi_b})$} for the test statistic for each $b,k$, and
\smash{$T(0,k) = T_{d,k}(z)$} for  the test statistic on the original sequence,
it helps to visualize the matrix of available test statistics: 
\begin{center}
\begin{tabular}{|c|c|c|c|}
\hline
$T(0,0)$ & $T(0,1)$ & $\cdots$ & $T(0,K)$ \\
\hline
$T(1,0)$ & $T(1,1)$ & $\cdots$ & $T(1,K)$ \\
\hline
$\vdots$ & & & \\
\hline
$T(B,0)$ & $T(B,1)$ & $\cdots$ & $T(B,K)$ \\
\hline
\end{tabular}
\end{center}
The permutation p-value, per column (per degree), is given by the following
computation:
\[
p_k = \frac{\sum_{b=0}^{B} \one\{ T(b,k) \geq T(0,k) \}}{B+1}, \quad k =
0,\ldots,K. 
\]
The aggregator, $G(p_0,\ldots,p_K)$, combines this information across the
columns. To calibrate it, we can treat this as yet another test statistic of the
data, writing \smash{$\tilde{T}(0) = G(p_0(z), \ldots, p_K(z))$}, and writing 
\[
\tilde{T}(b) = G\big( p_0(z_{\pi_b}), \ldots, p_K(z_{\pi_b}) \big) 
\] 
for the result when the original sequence $z$ is replaced by
\smash{$z_{\pi_b}$}. Now comes the key realization: each
\smash{$p_k(z_{\pi_b})$} can be computed from the same matrix of  
available test statistics as above, just given by swapping the $b\th$ row with 
the first row:   
\[
p_k(z_{\pi_b}) = \frac{\sum_{b'=0}^{B} \one\{ T(b',k) \geq T(b,k) \}}{B+1}, 
  \quad k = 0,\ldots,K.
\]
This is true because the set of permutations generated by applying $\pi_{b'}$,
$b' = 0,\dots,B$ to \smash{$z_{\pi_b}$} is the same as that generated by
applying $\pi_{b'}$, $b' = 0,\dots,B$ to $z$ (due to the group structure). Therefore
each \smash{$\tilde{T}(b)$} can be computed from the same matrix of available
test statistics, as can the permutation p-value corresponding to the aggregate 
statistic: 
\[
p_G = \frac{\sum_{b=0}^{B} \one\{ \tilde{T}(b) \geq \tilde{T}(0) \}}{B+1}.
\]
This is a valid p-value in finite-sample, for the same reason that permutation
p-values are valid whenever the set of permutations used forms a group (see,
e.g., \citet{hemerik2018exact}).

Figure \ref{fig:agg-and-grayout} examines two aggregation methods over the
same setup as that in Figure \ref{fig:roc-setup-per-dimension}. As before, we
use a mixture null as a proxy for the permutation null (for the sake of
computational efficiency; our simulations use 100 repetitions). The aggregation
rules considered are the min-rule: 
\smash{$G(p_0,\ldots,p_K) = \min_{k=0,\ldots,K} p_k$}, 
and the so-called Fisher-rule: 
\smash{$G(p_0,\ldots,p_K) = -\sum_{k=0}^K \log p_k$}. 
For each task (row), Figure \ref{fig:agg-and-grayout} displays the ROC curves
corresponding to these rules as dotted colored lines, and the ROC curve from the
best RKS (single $k$) test as a solid colored line; ROC curves from the other
RKS tests (other values of $k$) are grayed-out as a visual aid. We can see that
either aggregation method frequently competes with the best RKS test, with the
min-rule performing slightly better overall. 

\begin{figure}[p]
\vspace{-10pt}
\includegraphics[width=\textwidth]{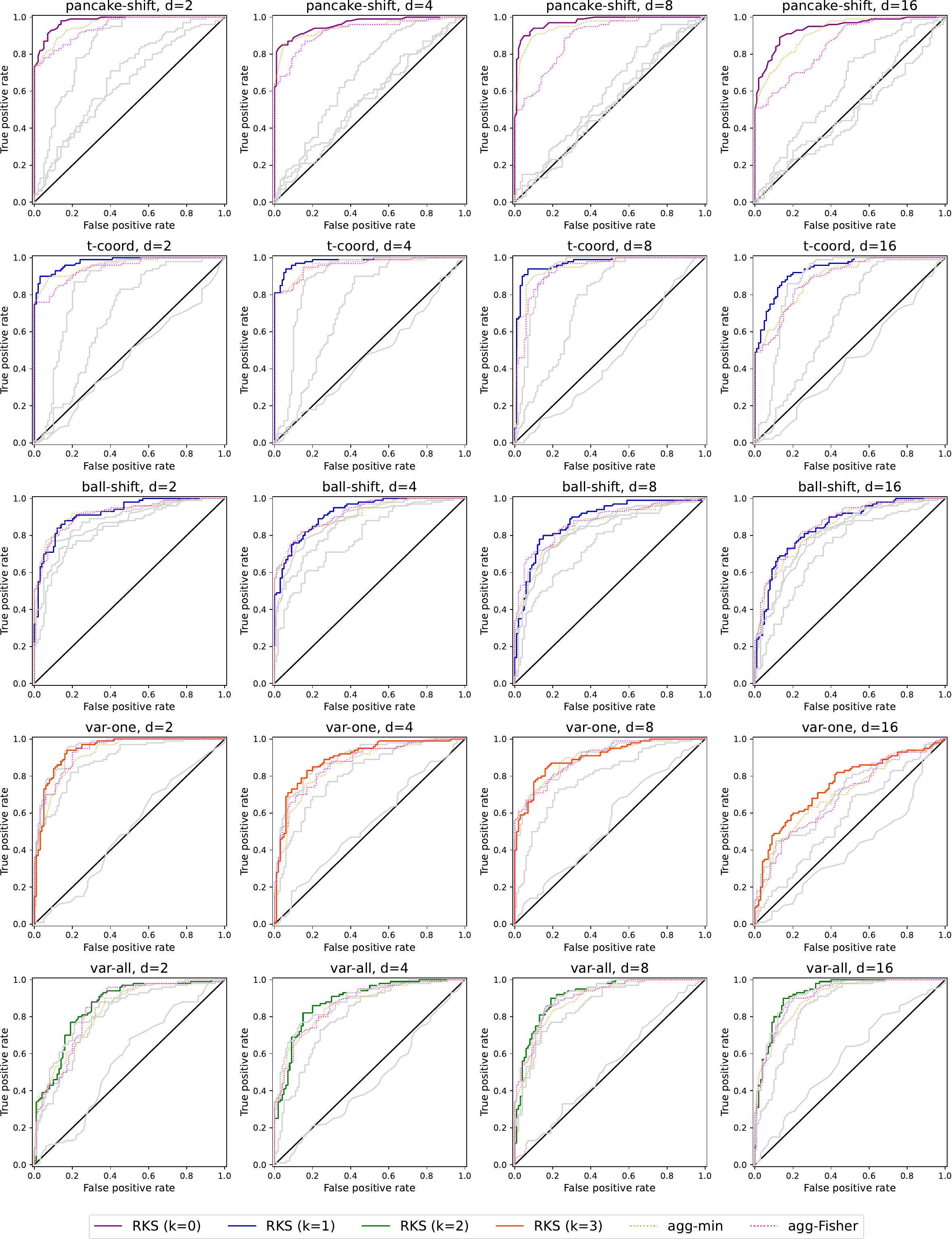}
\caption{ROC curves across the same experimental settings as in Figure
  \ref{fig:roc-setup-per-dimension}, along with ROC curves from the min
  aggregation rule (``agg-min'') and Fisher aggregation rule (``agg-Fisher''). 
  In each row, only the ROC curve from the best-performing RKS test is drawn in
  color, and the rest are drawn in gray.}
\label{fig:agg-and-grayout}
\end{figure}

\section{Discussion}

We proposed and analyzed a new multivariate nonparametric two-sample test,
which is an integral probability metric (IPM) test with respect to a class of
functions having unit Radon total variation of a given smoothness degree $k
\geq 0$. We call this the Radon-Kolmogorov-Smirnov (RKS) test, and it can be
seen as a generalization of the Kolmogorov-Smirnov (KS) test to multiple
dimensions and higher orders of smoothness. 

The RKS test, like any statistical test, is not expected to be most powerful in
all scenarios. Roughly, we expect the RKS test to be favorable in settings in
which the given distributions $\popP$ and $\popQ$ differ in only a few
directions in the ambient space; and higher orders of smoothness $k$ can be more
sensitive to tail differences along these few directions. When $\popP$ and
$\popQ$ differ in many or all directions, the more traditional kernel MMD test
with (say) Gaussian kernel will probably perform better.

It is worth revisiting computation of the RKS distance and discussing some
limitations of our work and possible future directions. In this paper, we
considered \emph{full batch} gradient descent applied to \eqref{eq:opt-problem}
(we actually used Adam, but the differences for this discussion are
unimportant), i.e., in each update the entire data set
\smash{$\{x_i\}_{i=1}^{m} \cup \{y_i\}_{i=1}^{n}$} is used to calculate the 
gradient. This results in the following complexity: $T$ iterations of gradient
descent costs $O((n+m)dNT)$ operations, where recall $N$ is the number of
neurons. However, \emph{mini batch} gradient descent should be also
investigated, and is likely to be preferred in large problem sizes. That said,
implementing mini batch gradients for \eqref{eq:opt-problem} is nontrivial due
to the nonseparable nature of the criterion across observations (due to the log
and the absolute value in the first term), and pursuing this carefully is an
important direction for future work. For now, we note that by just subsetting
$\empP, \empQ$ to a mini batch of size $b \leq n+m$, this results in a
$T$-iteration complexity of $O(bdNT)$ operations. To compare kernel MMD: 
this costs $O((n+m)^2 d)$ operations, and thus we can see that for large enough
problem sizes computation of the RKS statistic via full batch and especially
mini batch gradient descent can be more efficient, provided that small or
moderate values of $N,T$ still suffice to obtain good performance.

Finally, another direction worth investigating is the use of deeper networks,
which could be interpreted as learning the feature representation in which there
is the largest discrepancy in projected truncated moments. This would help
further distinguish the RKS test from kernel MMD, because the latter is ``stuck
with'' the representation provided by the choice of kernel, and it cannot learn
the representation based on the data. A deep RKS test could dovetail nicely with
existing architectures (and even pre-trained networks) for related tasks in
areas such as generative modeling, and out-of-distribution detection.

\subsection*{Acknowledgments}

We thank Rahul Parhi and Jonathan Siegel and helpful discussions about Radon
bounded variation spaces. This work was supported by NSF SCALE MoDL grant
2134133, ONR MURI grant N00014-20-1-2787, and the Miller Institute for Basic
Research in Science at the University of California, Berkeley.

\bibliographystyle{plainnat}
\bibliography{rks}

\begin{thebibliography}{42}
\providecommand{\natexlab}[1]{#1}
\providecommand{\url}[1]{\texttt{#1}}
\expandafter\ifx\csname urlstyle\endcsname\relax
  \providecommand{\doi}[1]{doi: #1}\else
  \providecommand{\doi}{doi: \begingroup \urlstyle{rm}\Url}\fi

\bibitem[Albert(2019)]{albert2019concentration}
M{\'e}lisande Albert.
\newblock Concentration inequalities for randomly permuted sums.
\newblock In \emph{High Dimensional Probability VIII: The Oaxaca Volume}, 2019.

\bibitem[Arjovsky et~al.(2017)Arjovsky, Chintala, and
  Bottou]{arjovsky2017wasserstein}
Martin Arjovsky, Soumith Chintala, and L{\'e}on Bottou.
\newblock Wasserstein generative adversarial networks.
\newblock In \emph{International Conference on Machine Learning}, 2017.

\bibitem[Bach(2017)]{bach2017breaking}
Francis Bach.
\newblock Breaking the curse of dimensionality with convex neural networks.
\newblock \emph{Journal of Machine Learning Research}, 18\penalty0
  (19):\penalty0 1--53, 2017.

\bibitem[Baringhaus and Franz(2004)]{baringhaus2004new}
Ludwig Baringhaus and Carsten Franz.
\newblock On a new multivariate two-sample test.
\newblock \emph{Journal of Multivariate Analysis}, 88\penalty0 (1):\penalty0
  190--206, 2004.

\bibitem[Bickel(1969)]{bickel1969distribution}
Peter~J. Bickel.
\newblock A distribution free version of the {Smirnov} two sample test in the
  p-variate case.
\newblock \emph{Annals of Mathematical Statistics}, 40\penalty0 (1):\penalty0
  1--23, 1969.

\bibitem[Bryson(1974)]{bryson1974heavy}
Maurice~C. Bryson.
\newblock Heavy-tailed distributions: {Properties} and tests.
\newblock \emph{Technometrics}, 16\penalty0 (1):\penalty0 61--68, 1974.

\bibitem[Caputo et~al.(2002)Caputo, Sim, Furesj{\"o}, and
  Smola]{caputo2002appearance}
Barbara Caputo, K.~Sim, Fredrik Furesj{\"o}, and Alexander~J. Smola.
\newblock Appearance-based object recognition using {SVMs}: {Which} kernel
  should {I} use?
\newblock In \emph{Advances in Neural Information Processing Systems Workshop
  on Statistical Methods for Computational Experiments in Visual Processing and
  Computer Vision}, 2002.

\bibitem[Dudley(2014)]{dudley2014uniform}
Richard~M. Dudley.
\newblock \emph{Uniform Central Limit Theorems}.
\newblock Cambridge University Press, 2014.
\newblock Second edition.

\bibitem[Dumer(2007)]{dumer2007covering}
Ilya Dumer.
\newblock Covering spheres with spheres.
\newblock \emph{Discrete \& Computational Geometry}, 38:\penalty0 665--679,
  2007.

\bibitem[Friedman and Rafsky(1979)]{friedman1979multivariate}
Jerome~H. Friedman and Lawrence~C. Rafsky.
\newblock Multivariate generalizations of the {Wald-Wolfowitz} and {Smirnov}
  two-sample tests.
\newblock \emph{Annals of Statistics}, 7\penalty0 (4):\penalty0 697--717, 1979.

\bibitem[Ghosh and Biswas(2016)]{ghosh2016distribution}
Anil~K. Ghosh and Munmun Biswas.
\newblock Distribution-free high-dimensional two-sample tests based on
  discriminating hyperplanes.
\newblock \emph{Test}, 25:\penalty0 525--547, 2016.

\bibitem[Goodfellow et~al.(2014)Goodfellow, Pouget-Abadie, Mirza, Xu,
  Warde-Farley, Ozair, Courville, and Bengio]{goodfellow2014generative}
Ian Goodfellow, Jean Pouget-Abadie, Mehdi Mirza, Bing Xu, David Warde-Farley,
  Sherjil Ozair, Aaron Courville, and Yoshua Bengio.
\newblock Generative adversarial nets.
\newblock In \emph{Advances in Neural Information Processing Systems}, 2014.

\bibitem[Green et~al.(2024)Green, Balakrishnan, and Tibshirani]{green2024two}
Alden Green, Sivaraman Balakrishnan, and Ryan~J. Tibshirani.
\newblock Two-sample testing with a graph-based total variation integral
  probability metric.
\newblock \emph{arXiv preprint arXiv:2409.15628}, 2024.

\bibitem[Gretton et~al.(2012)Gretton, Borgwardt, Rasch, Sch{\"o}lkopf, and
  Smola]{gretton2012kernel}
Arthur Gretton, Karsten~M. Borgwardt, Malte~J. Rasch, Bernhard Sch{\"o}lkopf,
  and Alexander~J. Smola.
\newblock A kernel two-sample test.
\newblock \emph{Journal of Machine Learning Research}, 13:\penalty0 723--773,
  2012.

\bibitem[Hemerik and Goeman(2018)]{hemerik2018exact}
Jesse Hemerik and Jelle Goeman.
\newblock Exact testing with random permutations.
\newblock \emph{Test}, 27:\penalty0 811--825, 2018.

\bibitem[Hendrycks and Gimpel(2017)]{hendrycks2017baseline}
Dan Hendrycks and Kevin Gimpel.
\newblock A baseline for detecting misclassified and out-of-distribution
  examples in neural networks.
\newblock In \emph{International Conference on Learning Representations}, 2017.

\bibitem[Henze(1988)]{henze1988multivariate}
Norbert Henze.
\newblock A multivariate two-sample test based on the number of nearest
  neighbor type coincidences.
\newblock \emph{Annals of Statistics}, 16\penalty0 (2):\penalty0 772--783,
  1988.

\bibitem[Kolmogorov(1933)]{kolmogorov1933sulla}
Andrey Kolmogorov.
\newblock Sulla determinazione empirica di una legge di distribuzione.
\newblock \emph{Giornale dell'Istituto Italiano degli Attuari}, 4:\penalty0
  83--91, 1933.

\bibitem[Li et~al.(2017)Li, Chang, Cheng, Yang, and P{\'o}czos]{li2017mmd}
Chun-Liang Li, Wei-Cheng Chang, Yu~Cheng, Yiming Yang, and Barnab{\'a}s
  P{\'o}czos.
\newblock {MMD GAN}: {Towards} deeper understanding of moment matching network.
\newblock In \emph{Advances in Neural Information Processing Systems}, 2017.

\bibitem[Lin et~al.(2020)Lin, Fan, Ho, Cuturi, and Jordan]{lin2020projection}
Tianyi Lin, Chenyou Fan, Nhat Ho, Marco Cuturi, and Michael Jordan.
\newblock Projection robust {Wasserstein} distance and {Riemannian}
  optimization.
\newblock In \emph{Advances in Neural Information Processing Systems}, 2020.

\bibitem[Long et~al.(2015)Long, Cao, Wang, and Jordan]{long2015learning}
Mingsheng Long, Yue Cao, Jianmin Wang, and Michael Jordan.
\newblock Learning transferable features with deep adaptation networks.
\newblock In \emph{International Conference on Machine Learning}, 2015.

\bibitem[Mroueh et~al.(2018)Mroueh, Li, Sercu, Raj, and
  Cheng]{mroueh2018sobolev}
Youssef Mroueh, Chun~Liang Li, Tom Sercu, Anant Raj, and Yu~Cheng.
\newblock Sobolev {GAN}.
\newblock In \emph{International Conference on Learning Representations}, 2018.

\bibitem[Mueller and Jaakkola(2015)]{mueller2015principal}
Jonas~W. Mueller and Tommi Jaakkola.
\newblock Principal differences analysis: {Interpretable} characterization of
  differences between distributions.
\newblock In \emph{Advances in Neural Information Processing Systems}, 2015.

\bibitem[M{\"u}ller(1997)]{muller1997integral}
Alfred M{\"u}ller.
\newblock Integral probability metrics and their generating classes of
  functions.
\newblock \emph{Advances in Applied Probability}, 29\penalty0 (2):\penalty0
  429--443, 1997.

\bibitem[Niles-Weed and Rigollet(2022)]{niles2022estimation}
Jonathan Niles-Weed and Philippe Rigollet.
\newblock Estimation of {Wasserstein} distances in the spiked transport model.
\newblock \emph{Bernoulli}, 28\penalty0 (4):\penalty0 2663--2688, 2022.

\bibitem[Ongie et~al.(2020)Ongie, Willett, Soudry, and
  Srebro]{ongie2020function}
Greg Ongie, Rebecca Willett, Daniel Soudry, and Nathan Srebro.
\newblock A function space view of bounded norm infinite width {ReLU} nets:
  {The} multivariate case.
\newblock In \emph{International Conference on Learning Representations}, 2020.

\bibitem[Parhi(2022)]{parhi2022ridge}
Rahul Parhi.
\newblock \emph{On Ridge Splines, Neural Networks, and Variational Problems in
  Radon-Domain {BV} Spaces}.
\newblock PhD thesis, Department of Electrical and Computer Engineering,
  University of Wisconsin-Madison, 2022.

\bibitem[Parhi and Nowak(2021)]{parhi2021banach}
Rahul Parhi and Robert~D. Nowak.
\newblock Banach space representer theorems for neural networks and ridge
  splines.
\newblock \emph{Journal of Machine Learning Research}, 22\penalty0
  (43):\penalty0 1--40, 2021.

\bibitem[Parhi and Nowak(2023)]{parhi2023near}
Rahul Parhi and Robert~D. Nowak.
\newblock Near-minimax optimal estimation with shallow {ReLU} neural networks.
\newblock \emph{Transactions on Information Theory}, 69\penalty0 (2):\penalty0
  1125--1140, 2023.

\bibitem[Paty and Cuturi(2019)]{paty2019subspace}
Fran{\c{c}}ois-Pierre Paty and Marco Cuturi.
\newblock Subspace robust {Wasserstein} distances.
\newblock In \emph{International Conference on Machine Learning}, 2019.

\bibitem[Sadhanala et~al.(2019)Sadhanala, Wang, Ramdas, and
  Tibshirani]{sadhanala2019higher}
Veeranjaneyulu Sadhanala, Yu-Xiang Wang, Aaditya Ramdas, and Ryan~J.
  Tibshirani.
\newblock A higher-order {Kolmogorov-Smirnov} test.
\newblock In \emph{International Conference on Artificial Intelligence and
  Statistics}, 2019.

\bibitem[Savarese et~al.(2019)Savarese, Evron, Soudry, and
  Srebro]{savarese2019infinite}
Pedro Savarese, Itay Evron, Daniel Soudry, and Nathan Srebro.
\newblock How do infinite width bounded norm networks look in function space?
\newblock In \emph{Annual Conference on Learning Theory}, 2019.

\bibitem[Schilling(1986)]{schilling1986multivariate}
Mark~F. Schilling.
\newblock Multivariate two-sample tests based on nearest neighbors.
\newblock \emph{Journal of the American Statistical Association}, 81\penalty0
  (395):\penalty0 799--806, 1986.

\bibitem[Sejdinovic et~al.(2013)Sejdinovic, Sriperumbudur, Gretton, and
  Fukumizu]{sejdinovic2013equivalence}
Dino Sejdinovic, Bharath Sriperumbudur, Arthur Gretton, and Kenji Fukumizu.
\newblock Equivalence of distance-based and {RKHS}-based statistics in
  hypothesis testing.
\newblock \emph{Annals of Statistics}, 41\penalty0 (5):\penalty0 2263--2291,
  2013.

\bibitem[Smirnov(1948)]{smirnov1948table}
Nikolai Smirnov.
\newblock Table for estimating the goodness of fit of empirical distributions.
\newblock \emph{Annals of Mathematical Statistics}, 19\penalty0 (2):\penalty0
  279--281, 1948.

\bibitem[Solmi and Onghena(2014)]{solmi2014combining}
Francesca Solmi and Patrick Onghena.
\newblock Combining p-values in replicated single-case experiments with
  multivariate outcome.
\newblock \emph{Neuropsychological Rehabilitation}, 24\penalty0
  (3--4):\penalty0 607--633, 2014.

\bibitem[Sz{\'e}kely and Rizzo(2005)]{szekely2005new}
G{\'a}bor~J. Sz{\'e}kely and Maria~L. Rizzo.
\newblock A new test for multivariate normality.
\newblock \emph{Journal of Multivariate Analysis}, 93\penalty0 (1):\penalty0
  58--80, 2005.

\bibitem[van~der Vaart and Wellner(1996)]{vandervaart1996weak}
Aad van~der Vaart and Jon Wellner.
\newblock \emph{Weak Convergence and Empirical Processes: With Applications to
  Statistics}.
\newblock Springer, 1996.

\bibitem[Vovk and Wang(2020)]{vovk2020combining}
Vladimir Vovk and Ruodu Wang.
\newblock Combining p-values via averaging.
\newblock \emph{Biometrika}, 107\penalty0 (4):\penalty0 791--808, 2020.

\bibitem[Wahba(1990)]{wahba1990spline}
Grace Wahba.
\newblock \emph{Spline Models for Observational Data}.
\newblock Society for Industrial and Applied Mathematics, 1990.

\bibitem[Wang et~al.(2014)Wang, Smola, and Tibshirani]{wang2014falling}
Yu-Xiang Wang, Alexander~J. Smola, and Ryan~J. Tibshirani.
\newblock The falling factorial basis and its statistical applications.
\newblock In \emph{International Conference on Machine Learning}, 2014.

\bibitem[Wei et~al.(2016)Wei, Lee, Wichers, and Marron]{wei2016direction}
Susan Wei, Chihoon Lee, Lindsay Wichers, and J.~S. Marron.
\newblock Direction-projection-permutation for high-dimensional hypothesis
  tests.
\newblock \emph{Journal of Computational and Graphical Statistics}, 25\penalty0
  (2):\penalty0 549--569, 2016.

\end{thebibliography}

\clearpage
\appendix


\allowdisplaybreaks

\section{Proof of Proposition \ref{prop:Parhi-Nowak-repr}}  
\label{supp:Parhi-Nowak-repr}

    By Theorem 22 in \citet{parhi2021banach}, every \smash{$\sf \in \RBV^k$} has a representative $f \in \sf$ of the form
    \begin{equation}
    \label{eq:ridge-mixture-representation-original}
        f(x) = \int_{\S^{d-1} \times \reals} 
            \Big[
            	(w^\T x - b)_+^k - \sum\nolimits_{|\alpha|\leq k} c_\alpha(w,b)x^\alpha
            \Big]
            \de \mu(w,b) + q_0(x),
    \end{equation}
    for a finite Radon measure $\mu$ on \smash{$\RadonDomain$} such that: 
    \textit{(i)} $\mu$ is an odd (respectively even) measure if $k$ is even (respectively odd);
    \textit{(ii)} the integrand is integrable in $(w,b)$ with respect to $\mu$ for any $x$;
    \textit{(iii)} \smash{$\| \mu \|_{\TV} = \|f\|_{\RTVk}$};
    \textit{(iv)} \smash{$q_0(x)$} is a polynomial of degree at most $k$;
    \textit{(v)} \smash{$\sum_{|\alpha|\leq k}\big(c_\alpha(w, b) x^\alpha + (-1)^k c_\alpha(-w, -b)x^\alpha\big) = (w^\T x-b)^k$}.
    
    Using the conditions \textit{(i)} and \textit{(v)}, 
    we can show the following with a manual calculation and the change of
    variables $(w,b) \to (-w, -b)$:
    \begin{align}
    \label{eqn:mu-flipping-trick}
    	&\int_\negativeRD \Big[(w^\T x - b)_+^k - \sum\nolimits_{|\alpha|\leq k} c_\alpha(w,b)x^\alpha \Big]\;\de\mu(w,b) \nonumber \\
        &=\int_{\S^{d-1} \times (0,\infty)} \Big[(-w^\T x + b)_+^k - \sum\nolimits_{|\alpha|\leq k} c_\alpha(-w,-b)x^\alpha \Big]\;\de\mu(-w,-b) \nonumber \\
        &=\int_{\S^{d-1} \times (0,\infty)} (-1)^{2k+1}\Big[(w^\T x - b)_-^k - \sum\nolimits_{|\alpha|\leq k} (-1)^k c_\alpha(-w,-b)x^\alpha \Big]\;\de\mu(w,b) \nonumber \\
        &=\int_{\S^{d-1} \times (0,\infty)} (-1)^{2k+1}\Big[(w^\T x - b)_-^k - (w^\T x - b)^k +  \sum\nolimits_{|\alpha|\leq k}  c_\alpha(w,b)x^\alpha \Big]\;\de\mu(w,b) \nonumber \\
    	&=\int_\positiveRD \Big[(w^\T x - b)_+^k - \sum\nolimits_{|\alpha|\leq k} c_\alpha(w,b)x^\alpha \Big]\;\de\mu(w,b).
    \end{align}
    Define a measure $\tilde{\mu}$ by {$\tilde{\mu}(E) = \mu(E_0) +2\mu(E_+)$}
    where for an event $E$ we denote {$E_0= (\zeroRD)\cap E$} and {$E_+=
      (\positiveRD)\cap E$}. Then \eqref{eqn:mu-flipping-trick} rewrites 
      \eqref{eq:ridge-mixture-representation-original} as  
    \begin{equation}
    \label{eq:ridge-mixture-representation-intermediate}
    	f(x) = \int_\nonnegRD \Big[(w^\T x - b)_+^k - \sum_{|\alpha|\leq k} c_\alpha(w,b)x^\alpha \Big]\;\de\tilde{\mu}(w,b) + q_0(x)
    \end{equation}
    where {$\|\tilde{\mu}\|_\TV = \|\mu\|_\TV = \|f\|_{\RTVk}$}.
    Note that {$(w^\T x - b)_+^k$} is integrable in $(w,b)$ with respect to $\tilde{\mu}$ because $\|w\|_2= 1$ and $b \geq 0$.
    Thus, $\int_\nonnegRD (w^\T x-b)_+^k \de\tilde{\mu}< \|x\|_2^k \|\tilde{\mu}\|_\TV$, hence $\sum_{|\alpha|\leq k} c_\alpha(w,b) x^\alpha$ is integrable with respect to $(w,b)$ for every $x$ as well. Thus, 
    \begin{equation*}
    	q_1(x) =\int_\nonnegRD\Big(\sum\nolimits_{|\alpha|\leq k} c_\alpha(w, b)x^\alpha \Big) \de\tilde{\mu}(w, b)
    \end{equation*}
    is well-defined due to the condition \textit{(ii)} and \eqref{eq:ridge-mixture-representation-intermediate}.
    Therefore \eqref{eq:ridge-mixture-representation-intermediate} can be rewritten as
    \begin{equation*}
    	f(x) = \int_\nonnegRD (w^\T x - b)_+^k \;\de\tilde{\mu}(w,b) + \tilde{q_0}(x)
    \end{equation*}
    where $\tilde{q_0}(x) = q_0(x) + q_1(x)$ is a polynomial of degree at most $k$.
    Also \smash{$\tilde{\mu}|_{\zeroRD} = \mu|_{\zeroRD}$} by definition of $\tilde{\mu}$, and hence, \smash{$\tilde{\mu}|_{\zeroRD}$} is odd when $k$ is even, and even when $k$ is odd by the condition \textit{(i)}.
    	Thus, we have arrived at the desired representation.

\section{Proof of Theorem \ref{thm:RBVk-continuous-representatives}}
\label{supp:RBVk-continuous-representatives}

We rely on the representation established by Proposition
\ref{prop:Parhi-Nowak-repr} and treat the cases $k = 0$ and $k \geq 1$
separately. 

\paragraph{Case $k = 0$.}

For this case, the boundary condition (continuity at the origin) implies that
the integral may be restricted to $b > 0$. 

\begin{lemma}
\label{lem:continuity-at-0}
    Consider \smash{$\sf \in \RBV^0$}.
    There exists a representative $f \in \sf$ which is continuous at $0\in\R^d$ if and only if $f$ admits a representation
    \begin{equation}
    \label{eq:k=0-pos-b-mixture}
        f(x) = \int_{\positiveRD} \one\{w^\T x \geq b\}\de \mu(w,b) + q,
    \end{equation}
    where $q$ is a constant and $\mu$ is a signed measure with \smash{$\|\mu\|_{\TV} < \infty$}.
\end{lemma}

\begin{proof}[Proof of Lemma \ref{lem:continuity-at-0}]
    By Proposition \ref{prop:Parhi-Nowak-repr}, there exists $f \in \sf$ such that for all $x$,
    \begin{equation*}
    	\begin{aligned}
    	    f(x) = \int_{\nonnegRD} \one\{w^\T x \geq b\} \de \mu(w,b) + q
    	         = f^{=0}(x) + f^{>0}(x),
    	\end{aligned}
    \end{equation*}
    where
    \begin{equation*}
	    \begin{gathered}
	        f^{>0}(x) =\int_{\positiveRD} \one\{w^\T x \geq b\} \de \mu(w,b) + q, \\
        	f^{=0}(x) = \int_{\zeroRD} \one\{w^\T x \geq 0\} \de \mu(w,b).
	    \end{gathered}
    \end{equation*}

    If $b > 0$, then \smash{$\lim_{x \to 0} \one\{w^\T x \geq b\} = 0$}.
    Thus, by dominated convergence,
    \begin{equation*}
        \lim_{x \to 0} f^{>0}(x) = q(x) = f^{>0}(0).
    \end{equation*}
    Thus, \smash{$f^{>0}$} is continuous at $0$.
    We conclude that $f$ is almost-everywhere equal to a function which is continuous at 0 if and only if \smash{$f^{=0}$} is almost-everywhere equal to a function which is continuous at 0.

    For $x \neq 0$,
    \smash{$w^\T x \geq 0$} if and only if \smash{$w^\T x / \| x \| \geq 0$}.
    Thus, \smash{$f^{=0}(x) = f^{=0}(x/\|x\|)$}.
    Thus, \smash{$f^{=0}$} is almost everywhere equal to a function which is continuous at $0$ if and only if there exists some $c \in \reals$ such that for almost every \smash{$x \in \S^{d-1}$} (almost every with respect to Lebesgue surface measure),
    we have \smash{$f^{=0}(x) = c$}.
    But this implies that \smash{$f^{=0}(x) = c$} for almost every \smash{$x \in \reals^d$}.
    Thus,
    we get that for almost every $x \in \reals^d$,
    \begin{equation*}
        f(x) = f^{>0}(x) + c
               = \int_{\positiveRD} \one\{w^\T x \geq b\} \de \mu(w,b) + q + c.
    \end{equation*}
    The right-hand side is the desired representative.
\end{proof}

Theorem \ref{thm:RBVk-continuous-representatives} in the case $k = 0$ is now a consequence of the following lemma, which additionally gives an explicit form for the radially cone continuous representative of $\sf$.

\begin{lemma}
\label{lem:radial-cone-continuity}
    Assume $\sf \in \RBV^0$ contains a representative which is continuous at $0$.
	    Then it contains a unique representative which is radially cone continuous.
	    This representative can be written in the form \eqref{eq:k=0-pos-b-mixture}
	    where $\mu$ is a signed measure with \smash{$\| \mu \|_{\TV} < \infty$}.
    Conversely, any function of the form \eqref{eq:k=0-pos-b-mixture}
    	where $\mu$ is a signed measure with \smash{$\| \mu \|_{\TV} < \infty$}
    	is the unique radial cone continuous representative of some \smash{$\sf \in \RBV^0$}.
\end{lemma}

\begin{proof}[Proof of Lemma \ref{lem:radial-cone-continuity}]
    Consider \smash{$\sf \in \RBV^0$} which contains a representative which is
    continuous at the origin.  By Lemma \ref{lem:continuity-at-0}, there exists
    $f \in \sf$ which satisfies \eqref{eq:k=0-pos-b-mixture} for all $x\in\R^d$,
    for some constant $q$ and signed measure $\mu$ with \smash{$\| \mu \|_{\TV}
      < \infty$}.  
    
    Fix $x \neq 0$.
    Consider any sequence \smash{$\epsilon_j\in\R$} such that \smash{$\epsilon_j \downarrow 0$}, and \smash{$v_j \in \S^{d-1}$} such that \smash{$v_j \to x / \|x\|$}. Consider any \smash{$(w,b) \in \positiveRD$}. If \smash{$w^\T x < b$}, then \smash{$w^\T (x + \epsilon_j v_j) < b$} eventually because \smash{$x + \epsilon_j v_j \to x$}. Alternatively, \smash{$w^\T x \geq b$} gives \smash{$w^\T x / \|x\| > 0$}, whence \smash{$w^\T v_j > 0$} eventually. In particular, \smash{$w^\T (x + \epsilon_j v_j) > w^\T x > 0$} eventually. Thus, in both cases with have \smash{$\one\{w^\T (x + \epsilon_j v_j) \geq b\} \to \one\{w^\T x \geq b\}$}. By dominated convergence, we get that 
    \begin{equation*}
	    \begin{aligned}
    	    f(x + \epsilon_j v_j) &= \int_{\positiveRD} \one\{w^\T (x+\epsilon_j v_j) \geq b\}\de \mu(w,b) + q
	        \\
            &\to \int_{\positiveRD} \one\{w^\T x \geq b\}\de \mu(w,b) + q = f(x).
    \end{aligned}
    \end{equation*}
    Thus, $f$ is radially cone continuous and can be written in the form of \eqref{eq:k=0-pos-b-mixture}.

    By Theorem 22 in \citet{parhi2021banach} and reversing the construction of $\tilde{\mu}$ from $\mu$ in the proof of Proposition \ref{prop:Parhi-Nowak-repr},
    any function of the form \eqref{eq:k=0-pos-b-mixture} where $\mu$ is a signed measure with $\| \mu \|_{\TV} < \infty$ is a represenative of an equivalence class $\sf \in \RBV^0$. The argument above also shows that it is radially cone continuous. Thus, we have shown the converse statement as well.
\end{proof}

\paragraph{Case $k = 1$.} 

We begin with a lemma.

\begin{lemma}
\label{lem:k-1-derivative}
    For $k \geq 1$ and every \smash{$\sf \in \RBV^k$}, there exists a unique representative $f \in \sf$ which is continuous. This representative is in fact $(k-1)$-times continuously differentiable and can be written in the form in Proposition \ref{prop:Parhi-Nowak-repr}. Moreover, for any multi-index $\alpha$ with \smash{$|\alpha| \leq k-1$}, we have for all \smash{$x \in \reals^d$} that
    \begin{equation}
    \label{eq:deriv-explicit}
        D^\alpha f(x){}
            =
            \frac{k!}{(k-|\alpha|)!}
            \int_{\nonnegRD} w^\alpha (w^\T x - b)_+^{k - |\alpha|} \de \mu(w,b)
            + 
            D^\alpha q(x).
    \end{equation}
    where \smash{$w^\alpha = \prod_{i=1}^d w_i^{\alpha_i}$}.
\end{lemma}

\begin{proof}[Proof of Lemma \ref{lem:k-1-derivative}]
    For any $\sf \in \RBVk$, let $f \in \sf$ be a representative of the form in
    Proposition \ref{prop:Parhi-Nowak-repr}. We will show that this
    representative $f$ is $(k-1)$-times continuously differentiable by inducting
    on the order $m = |\alpha|$ of the multi-index $\alpha$. 

    The base case is $m = 0$. In this case, the representation of the
    \smash{$0\th$} derivative is equivalent to the representation in the
    proposition. The continuity of $f$ follows by Lemma
    \ref{lem:rbv-ftn-continuous} applied to this representation.  

    Now consider $m \geq 1$, and assume we have established the result for
    $m-1$. Consider $\alpha$ with \smash{$|\alpha| = m$}, and without loss of
    generality, assume $\alpha_1 \geq 1$ and let \smash{$\tilde\alpha =
      (\alpha_1-1,\alpha_2,\ldots,\alpha_d)$}. By the inductive hypothesis,
    \[
        D^{\alpha}f(x)
            = \frac{k!}{(k-|\alpha|+1)!} \partial_{x_1}
              \int_{\nonnegRD} w^{\tilde{\alpha}} (w^\T x - b)_+^{k - |\alpha| + 1} \de \mu(w,b)
            +  D^\alpha q(x).
    \]
    Note that for $x\in\R^d$, we have \smash{$\partial_{x_1} \big(w^{\tilde{\alpha}} (w^\T x -b)_+^{k-|\alpha| +1}\big) = (k-|\alpha|+1) w^\alpha(w^\T x - b)^{k-|\alpha|}$}. This function is integrable in $(w,b)$ with respect to $\mu$ for every $x\in\R^d$, since \smash{$\|w\|_2=1$} and \smash{$b\geq 0$}.
    Moreover, the partial derivative is uniformly bounded on $\{x' : \|x'\| \leq 2\|x\|\}$ by $2^{k-|\alpha|+1}(k-|\alpha|+1) \big(\prod_{i=1}^d |w_i|^{\alpha_i}\big)(2\|x\|)^{k-|\alpha|}$, which is integrable with respect to $\mu$. These conditions allow us to apply the Leibniz rule to conclude that we can exchange differentiation and integration, which yields \eqref{eq:deriv-explicit}.
    The continuity of \smash{$D^\alpha f$} then follows from Lemma \ref{lem:rbv-ftn-continuous}, and hence, the induction is complete.

    Thus we have shown that $f$ is $(k-1)$-times continuously differentiable. All that is left to show is that $f$ is the unique element of $\sf$ which is $(k-1)$-times continuously differentiable. This holds because any two continuous functions which are equal almost everywhere are in fact equal everywhere.
\end{proof}

Theorem \ref{thm:RBVk-continuous-representatives} in the case $k \geq 1$ is now a consequence of the following lemma,
which additionally gives an explicit form for the radially cone continuous representative of $\sf$.

\begin{lemma}
\label{lem:repr-with-boundary-condition}
    For $k \geq 1$, let \smash{$\sf \in \RBVk$} be such that the unique
    continuous representative $f \in \sf$ is $k$-times classically
    differentiable at \smash{$0\in\R^d$} with \smash{$D^\alpha f(0) = 0$} for
    all \smash{$|\alpha| \leq k$}. Then $f$ has representation of the form in
    Proposition \ref{prop:Parhi-Nowak-repr} with $q(x) \equiv 0$. 
\end{lemma}

\begin{proof}[Proof of Lemma \ref{lem:repr-with-boundary-condition}]
    Because $w^\alpha(w^\T 0 - b)_+^{k-|\alpha|} = 0$ for $|\alpha| \leq k-1$ and any $(w,b) \in \nonnegRD$,
    \eqref{eq:deriv-explicit} implies that \smash{$D^\alpha f(0) = D^\alpha q(0)$} for all $|\alpha|\leq k-1$.

    Next, we evaluate $D^\alpha f(0)$ for $|\alpha| = k$, which, by assumption, exists. Without loss of generality, assume $\alpha_1 \geq 1$, and let \smash{$\tilde\alpha = (\alpha_1-1,\alpha_2,\ldots,\alpha_d)$}. Note that for $b > 0$, \smash{$\partial_{x_i} w^{\tilde\alpha} (w^\T x -b)_+\big|_{x=0} = 0$}. By \eqref{eq:deriv-explicit}, we can write
    \begin{equation*}
        D^\alpha f(0)
        = k!\partial_{x_1}\int_{\nonnegRD} w^{\tilde\alpha} (w^\T x -b)_+\; \de\mu(w,b) \Big|_{x=0}
          + D^\alpha q(0).
    \end{equation*}
    We compute the partial derivative in the first term manually:
    \begin{equation*}
    	\begin{aligned}
    	    \partial_{x_1}\int_{\nonnegRD} w^{\tilde\alpha} (w^\T x - b)_+\; \de\mu(w,b) \Big|_{x=0}
    	        &=
    	        \lim_{h \to 0} \frac1h \int_{\nonnegRD} w^{\tilde\alpha} (w_1 h - b)_+\; \de \mu(w,b).
    	\end{aligned}
    \end{equation*}
    For $b > 0$ and \smash{$w \in \S^{d-1}$}, we have \smash{$(w_1h - b)_+ \leq |h|\cdot \one\{w_1h > b\}$} because $|w_1| \leq 1$. Also $|w^\alpha| \leq 1$ because $|w_j| \leq 1$ for all $j$. Thus,
    \begin{equation*}
	    \begin{aligned}
	        \bigg| \frac1h \int_\positiveRD w^{\tilde\alpha} (w_1h - b)_+ \; \de\mu(w,b) \bigg|
	            &\leq \frac1{|h|} \int_{\positiveRD} w^{\tilde\alpha} |h| \one\{w_jh - b\} \;\de\mu(w,b)
    	    \\
    	        &\leq |\mu|\Big((w,b)\in\positiveRD: w_j h > b \Big).
    	\end{aligned}
    \end{equation*}
    Because $|\mu|$ is a finite positive measure and \smash{$\bigcap_{h > 0} \big\{(w,b)\in\positiveRD: w_j h > b\big\} = \emptyset$}, we conclude $|\mu|\big((w,b)\in\positiveRD: w_j h > b\big) \to 0$ as $h \to 0$. Thus, we have that 
    \begin{equation}
    \label{eq:deriv-limit-on-b=0}
        D^\alpha f(0)
            = k! \lim_{h \to 0} \Big[\frac1h \int_{\zeroRD} w^{\tilde\alpha} (w_1h - b)_+ \; \de\mu(w,b)\Big]
            + D^\alpha q(0),            
    \end{equation}
    and in particular, the limit in the first term exists. For $b = 0$, \smash{$(w_1 h - b)_+ = w_1h\cdot \one\{w_1 > 0\}$} when $h > 0$, and \smash{$(w_1 h - b)_+ = w_1h\cdot \one\{w_1 < 0\}$} when $h < 0$. Thus, for $h > 0$,
    \begin{equation*}
        \frac1h \int_{\zeroRD} w^{\tilde\alpha} (w_1h - b)_+ \;\de\mu(w,b)
        = \int_{\substack{\zeroRD \\ w_1 > 0}} w^{\alpha} \;\de\mu(w,b),
    \end{equation*}
    and for $h < 0$,
    \begin{equation*}
    	\begin{aligned}
        	\frac1h \int_{\zeroRD} w^{\tilde\alpha} (w_1h - b)_+ \;\de\mu(w,b)
        	    &= \int_{\substack{\zeroRD \\ w_1 < 0}} w^{\alpha} \;\de\mu(w,b)
        	\\
        	    &= (-1)^{|\alpha|}(-1)^{k+1}\int_{\substack{\zeroRD \\ w_1 > 0}} w^{\alpha} \;\de\mu(w,b),
    	\end{aligned}
    \end{equation*}
    where the second equality uses the change of variables \smash{$w \mapsto - w$} and that \smash{$\mu|_{\S^{d-1}\times \{0\}}$} is odd when $k$ is even and even when $k$ is odd (see Proposition \ref{prop:Parhi-Nowak-repr}). By differentiability, the quantities in the two previous displays must be equal. Because \smash{$(-1)^{|\alpha|}(-1)^{k+1} = -1$}, this implies that the quantities in the two previous displays must be 0. That is, the limit in \eqref{eq:deriv-limit-on-b=0} must be 0, and we have \smash{$0 = D^\alpha f(0) = D^\alpha q(0)$}.

    We have thus shown that \smash{$D^\alpha q(0) = 0$} for all $|\alpha|\leq k$, which, because $q(x)$ is a polynomial of degree at most $k$, implies $q(x) \equiv 0$.
\end{proof}

\section{Proof of Theorem \ref{thm:Fk-rep-poly0}}
\label{supp:Fk-rep-poly0}

Theorem \ref{thm:Fk-rep-poly0} is an immediate consequence of Lemmas
\ref{lem:radial-cone-continuity} and \ref{lem:repr-with-boundary-condition}. 

\section{Proof of Theorem \ref{thm:test-stat-is-relu-power}}
\label{supp:test-stat-is-relu-power}

We again treat the cases $k = 0$ and $k \geq 1$ separately.

\paragraph{Case $k = 0$.}

Let
    \begin{equation*}
        \cG_0^+ = \Big\{\one\{w^\T \cdot \geq b\}:\;(w,b)\in\S^{d-1} \times (0,\infty)\Big\}.
    \end{equation*}
    Because \smash{$\cG_0^+ \subseteq \cG_0$}, for any signed measure $\mu$ on \smash{$\positiveRD$} with $\| \mu \|_{\TV} = 1$ and $q \in \reals$,
    \begin{equation*}
    \begin{aligned}
        \sup_{f \in \cG_0} \popP(f) - \popQ(f)
            &\geq \sup_{f \in \cG_0^+} \popP(f) - \popQ(f)
             \geq \int_{\positiveRD} \int_{\reals^d} \one\{w^\T x\geq b\} \,\de(\popP-\popQ)(x) \,\de\mu(w,b)
        \\
            &= \int_{\reals^d}
               \Big( \int_{\S^{d-1} \times (0,\infty)} \one\{w^\T x\geq b\} \, \de\mu(w,b) + q  \Big)
               \, \de(\popP-\popQ)(x).
    \end{aligned}
    \end{equation*}
    Taking the supremum on the right-hand side over measure $\mu$,
    we get that
    \begin{equation*}
        \sup\nolimits_{f \in \cG_0} \popP(f) - \popQ(f)
        \geq 
        \sup\nolimits_{f \in \cF_0} \popP(f) - \popQ(f).
    \end{equation*}
    Because $\lim_{b \to \infty} P(w^\T x \geq b) - Q(w^\T x \geq b)$, 
    $\S^{d-1}$ is compact, 
    and $P(w^\T x \geq b) - Q(w^\T x \geq b) $ is continuous, 
    the supremum on the left-hand side is achieved.
    Then, if the supremum is achieved at some $(w,b)$ with $b > 0$, since \smash{$x \mapsto \one\{ w^\T x \geq b\} \in \cF_0$}, we also have 
    \begin{equation*}
        \sup_{f \in \cG_0}
            \popP(f) - \popQ(f)
            \leq 
            \sup_{f \in \cF_0}
            \popP(f) - \popQ(f).
    \end{equation*}
    On the other hand, consider that the supremum is achieved at some $(w,0)$.
    Note that
    \[
        \lim_{b \uparrow 0}\; \popQ(w^\T x \leq b) - \popP(w^\T x \leq b)
        = \popP(w^\T x \geq 0) - \popQ(w^\T x \geq 0).
    \]
    But \smash{$\popQ(w^\T x \leq b) - \popP(w^\T x \leq b) = \popP(f)_b - \popQ(f)_b$} for \smash{$f_b(x) = \one\{-w^\T x \geq -b \}$}. Note that \smash{$f_b \in \cG_0^+$}, and therefore, \smash{$f_b \in \cF_0$}. Thus, we have in this case also that $\sup_{f \in \cG_0}\popP(f) - \popQ(f)\leq \sup_{f \in \cF_0}\popP(f) - \popQ(f)$.
    
Having shown both directions of the inequality, we conclude $\sup_{f \in
  \cG_0} \popP(f) - \popQ(f) = \sup_{f \in \cF_0} \popP(f) - \popQ(f)$. The
same result holds with the roles of $\popP$ and $\popQ$ reversed, which
establishes the result with the supremum over $|\popP(f) - \popP(Q)|$.

\paragraph{Case $k \geq 1$.} 

Because $\popP$ and $\popQ$ have finite \smash{$k\th$} moments, by dominated convergence
    \begin{equation*}
        (w,b) \mapsto \E_{\popP}[(w^\T x - b)_+^k] - \E_{\popQ}[(w^\T x - b)_+^k]
    \end{equation*}
    is continuous on \smash{$\nonnegRD$}. Thus,
    \begin{equation*}
    	\begin{aligned}
    	    \sup_{f \in \cG_k} \popP(f) - \popQ(f)
    	    = \sup_{(w,b)\in\positiveRD} \E_{\popP}[(w^\T x - b)_+^k] - \E_{\popQ}[(w^\T x - b)_+^k]
    	    \leq \sup_{f \in \cF_k} \popP(f) - \popQ(f),
    	\end{aligned}
    \end{equation*}
    where the last inequality holds because \smash{$x \mapsto (w^\T x - b)_+^k \in \cF_k$} for all \smash{$w \in\S^{d-1}$} and $b > 0$ \citep{parhi2021banach}.
    On the other hand, for any $f \in \cF_k$,
    \begin{equation*}
    \begin{aligned}
        \popP(f) - \popQ(f)
            &= \int_{\reals^d} \int_\nonnegRD (w^\T x-b)_+^k \,\de\mu(w,b) \,\de(\popP-\popQ)(x)
        \\
            &= \int_\nonnegRD \int_{\reals^d} (w^\T x-b)_+^k \,\de(\popP-\popQ)(x)\, \de\mu(w,b) 
            \leq \sup_{f' \in \cG_k} |\popP(f') - \popQ(f')|,
    \end{aligned}
    \end{equation*}
    where the second equality uses Fubini's theorem and the final inequality uses \smash{$\|\mu\|_{\TV} = \|\sf\|_{\RTVk} \leq 1$}.
    Because the previous display holds for all $f \in \cF_k$,
    we have $\sup_{f \in \cF_k} \popP(f) - \popQ(f) = \sup_{f \in \cG_k} |\popP(f) - \popQ(f)|$.
    Since $f \in \cG_k$ implies $-f \in \cG_k$, we have
    \smash{$\sup_{f \in \cG_k} |\popP(f) - \popQ(f)| = \sup_{f \in \cG_k} \popP(f) - \popQ(f)$}.
    
    Thus, we have shown both that \smash{$\sup_{ f \in \cG_k } |\popP(f) - \popQ(f)| \leq \sup_{ f \in \cF_k } |\popP(f) - \popQ(f)|$} and the reverse inequality, so that these are in fact equal.


\section{Proof of Theorem \ref{thm:rtv-ipm-metric}}
\label{proof:rtv-ipm-metric}

If $\popP = \popQ$, then for any \smash{$f \in \cF_k$, $\popP(f) - \popQ(f) =
  0$}, whence \smash{$\rho(\popP,\popQ;\cF_k) = 0$}. Alternatively, consider
$\popP \neq \popQ$.  Then, by the uniqueness of the characteristic function,
there exists \smash{$w \in \reals^d$} such that the distribution of
\smash{$w^\T X$} and \smash{$w^\T Y$}, where $X \sim \popP$ and $Y \sim
\popQ$, are different. This implies that there exists some \smash{$t \in
  \reals$} such that \smash{$\popP(w^\T X \geq t) \neq \popQ(w^\T Y \geq
  t)$}. 

In what follows, we treat the cases $k = 0$, $k = 1$, and $k \geq 2$
separately. 

\paragraph{Case $k = 0$.}

If $\popP(w^\T X \geq t) > \popQ(w^\T Y \geq t)$ for some $t \geq 0$, then
taking \smash{$f(x) = \one\{w^\T x \geq t\}$} gives 
\[
\rho(\popP,\popQ;\cF_0) = \rho(\popP,\popQ;\cG_0) \geq \popP(f) - \popQ(f) > 0, 
\]
where we have used Theorem \ref{thm:test-stat-is-relu-power}.
If $\popP(w^\T X \geq t) > \popQ(w^\T Y \geq t)$ for some $t \leq 0$, then
because \smash{$t \mapsto \popP(w^\T X \geq t)$} is left-continuous, for $s <
t$ sufficiently close to $t$, we will have \smash{$\popP(w^\T X > s) >
  \popQ(w^\T Y > s)$}. Then, taking \smash{$f(x) = -\one\{(-w)^\T x \geq -s\}$}
gives   
\[\rho(\popP,\popQ;\cF_0) = \rho(\popP,\popQ;\cG_0) \geq \popP(f) - \popQ(f) >
  0, 
\]
where we have used Theorem \ref{thm:test-stat-is-relu-power}.

Similarly, if \smash{$\popP(w^\T X \geq t) < \popQ(w^\T X \geq t)$}, then
we can take \smash{$f(x) = -\one\{w^\T x \geq t\}$} in the case $t \geq 0$ and 
\smash{$f(x) = \one\{(-w)^\T x \geq -s\}$} for some $s < t$ sufficiently close to
$t$ to conclude \smash{$\rho(\popP,\popQ;\cF_0) > 0$}. 

\paragraph{Case $k = 1$.}

Because \smash{$\E_{\popP}[\one\{w^\T X \geq s\}]$} and \smash{$\E_{\popQ}[\one\{w^\T Y \geq s\}]$} are left-continuous in $s$, if $w^\T X$ and \smash{$w^\T y$} have a different distribution, there is some open interval \smash{$(\alpha,\beta) \subseteq \reals$} on which \smash{$\E_{\popP}[\one\{w^\T X \geq s\}] > \E_{\popQ}[\one\{w^\T Y \geq s\}]$} or \smash{$\E_{\popP}[\one\{w^\T X \geq s\}] < \E_{\popQ}[\one\{w^\T Y \geq s\}]$}, uniformly over all $s\in(\alpha, \beta)$.
Because either $\beta > 0$ or $\alpha < 0$, this interval contains an open interval entirely contained in either $\reals_{> 0}$ or $\reals_{< 0}$.
In the former case, we have that for $b > 0$, \smash{$\E_\popP[(w^\T X - b)_+] \neq \E_\popQ(w^\T Y - b)_+]$} because
\[
	\E_\popP\big[(w^\T X - b)_+\big]
		= \int_b^\infty \E_{\popP}\big[\one\{w^\T X \geq s\}\big] \,\de s,
	\quad
	\E_\popQ\big[(w^\T Y - b)_+\big]
		= \int_b^\infty \E_{\popQ}\big[\one\{w^\T Y \geq s\}\big] \,\de s.
\]
We then get \smash{$\rho(\popP,\popQ;\cF_1) = \rho(\popP,\popQ;\cG_1) > 0$} by taking $f(x) = (w^\T x - b)_+$.
In the latter case, we have that for $b > 0$, $\E_\popP[((-w)^\T X -b)_+] \neq \E_\popQ[((-w)^\T Y -b)_+]$ because
\[
	\E_\popP\big[((-w)^\T X - b)_+\big]
		=
		\int_{-\infty}^{-b} \E_{\popP}\big[(1-\one\{w^\T X \geq s\})\big] \,\de s,
\]
and likewise for $\popQ$.
We then get \smash{$\rho(\popP,\popQ;\cF_1) = \rho(\popP,\popQ;\cG_1) > 0$} by taking $f(x) = ((-w)^\T x - b)_+$.

\paragraph{Case $k \geq 2$.}

Note that the derivatives of the mappings \smash{$b \mapsto \E_\popP[(w^\T X - b)_+^k]$} and \smash{$b \mapsto \E_\popQ[(w^\T Y - b)_+^k]$} are \smash{$b \mapsto k\E_\popP[(w^\T X - b)_+^{k-1}]$} and \smash{$b \mapsto k\E_\popQ[(w^\T Y - b)_+^{k-1}]$}, respectively.
If there is some $(w,b)$ with $b \geq 0$ such that \smash{$\E_\popP[(w^\T X - b)_+^{k-1}] \neq \E_\popP[(w^\T X - b)_+^{k-1}]$}, by continuity, $\E_\popP[(w^\T X - s)_+^{k-1}] \neq \E_\popP[(w^\T X - s)_+^{k-1}]$ for all $s$ in an open interval around $b$, whence, by integration, $\E_\popP[(w^\T X - s)_+^k] \neq \E_\popP[(w^\T X - s)_+^k]$ for some $s$ in this interval.
We can then conclude the result by induction, using $k = 1$ as a base case.


\section{Proof of Theorem \ref{thm:null-distribution}}
\label{supp:uclt-prelim}

For a function class $\cF$, a probability measure $\P$, and $\epsilon>0$,
we say that \smash{$\{[l_i,u_i] : i \in [N]\}$} is an $(\epsilon,\P)$-bracket 
of $\cF$ with cardinality $N$ if for all $i \in [N]$ we have $l_i(x) \leq
u_i(x)$ for all $x$ and 
\[
\|u_i(X)-l_i(X)\|_{L^2(\P)}^2 = \E_{X\sim\P}[(u_i(X)-l_i(X))^2]
  \leq\epsilon^2,
\]
and for all $f \in \cF$ there exists $i \in [N]$ such that  \smash{$l_i(x) \leq
  f(x) \leq u_i(x)$} for all $x$. The $(\epsilon,\P)$-bracketing number, which
we denote by \smash{$N_\bracket(\epsilon; \cF, \|\bcdot\|_{L^2(\P)})$},  
is the minimum cardinality of an $(\epsilon,\P)$-bracket of $\cF$.
Using standard techniques from empirical process theory,
we will be able to prove Theorem \ref{thm:null-distribution} by bounding the
bracketing number \smash{$N_\bracket(\epsilon; \cG_k, \|\bcdot\|_{L^2(\P)})$}. 

\subsection{Bracketing number for $\cG_k$, $k \geq 1$}

\smallskip
\begin{theorem}
\label{thm:bracketing-number}
    Let $k \geq 1$ and let $\P$ be a probability measure with finite
    moments of order $2k + \Delta$, for any fixed $\Delta>0$,
    i.e., $\E_{X\sim \P} \|X\|_2^{2k+\Delta} = M <\infty$.
    Then, for the class $\cG_k$, there is a constant \smash{$C_\bracket>0$}
    depending only on $M,d,\Delta$, and $k$ such that 
    \[
        \log N_\bracket(\epsilon; \cG_k, \|\bcdot\|_2) 
        \leq 
        C_\bracket 
        \log\Big(1 + \frac{1}{\epsilon}\Big).
    \]
\end{theorem}

\begin{proof}[Proof of Theorem \ref{thm:bracketing-number}]
    We denote by \smash{$\ball_r^d(x)$} the ball of radius $r$ centered at $x \in \reals^d$:
    \[
        \ball_r^d(x) = \{z\in\R^d\;|\;\|z-x\|_2^2 \leq r^2\}.
    \]
    We say a set $S \subseteq \reals^d$ is a $r$-covering of a set $A \subseteq \reals^d$ if $A$ is contained in the union of the balls of radius $r$ centered at points $x\in I$:
    \[
        A \subseteq \bigcup\nolimits_{x \in I}\ball_r^d(x).
    \]

    For a compact set $I \subseteq \reals^d$ and real number $b \in \reals$,
    define functions
    \begin{align*}
        \ell_{I, b}(x) &= \inf\Big\{(w^\T x - b)_+^k: w\in I\Big\},
        \quad\text{and}\quad
        u_{I, b}(x) = \sup\Big\{(w^\T x - b)_+^k: w\in I\Big\},
    \end{align*}
    where we adopt the convention that \smash{$\ell_{I,\infty}(x) = u_{I,\infty}(x) \equiv 0$}.

    \begin{lemma}
    \label{lem:ul-is-a-bracket}
        Consider $S$ is an $r$-covering of $\S^{d-1}$.
        Let $\mathcal{I}$ be the collection of sets $\{I = \S^{d-1} \cap
        \ball_r^d(x) : x \in S\}$, and let $b_j = j\delta $ for $j =
        -1,\ldots,N$, and $b_{N+1} = \infty$ for some $\delta > 0$ and integer
        $N > 0$. Then, the following is a bracket of $\cG_k$
        \begin{equation}
        \label{eqn:bracket}
            \big\{
                [l_{I,b_{j+1}},u_{I,b_j}]
                :
                I \in \mathcal{I},
                j = -1,\ldots,N
            \big\}.
        \end{equation}
    \end{lemma}

    \begin{proof}[Proof of Lemma \ref{lem:ul-is-a-bracket}]
        Consider any \smash{$(w,b)\in\nonnegRD$} and its corresponding function $ x\mapsto (w^\T x-b)_+^k\in\cG_k$. 
        Because $S$ is an $r$-covering,
        there exists $I\in \mathcal{I}$ such that $w\in I$. 
        Also clearly there exists $j\in\{-1, 0, 1, \ldots, N\}$ such that $b_j < b \leq b_{j+1}$.
        Then $\ell_{I, b_{j+1}}\leq (w^\T x-b)_+^k$ from the following observation:
        \begin{equation*}
            \ell_{I, b_{j+1}}(x)
                =
                \inf\Big\{(\tilde w^\T x - b_{j+1})_+^k: \tilde w\in I\Big\}
                \leq
                (w^\T x - b_{j+1})_+^k
                \leq 
                (w^\T x - b)_+^k.
        \end{equation*}
        Similarly, we can show $(w^\T x-b)_+^k \leq (w^\T x - b_{j})_+^k \leq u_{I, b_j}(x)$.
    \end{proof}

    \begin{lemma}
    \label{lem:ul-is-eps-bracket}
        In the setting of Lemma \ref{lem:ul-is-a-bracket},
        if we set $r = \delta = \frac{\epsilon/3}{k\sqrt{1+\max(1,M)}}$ 
        and $N=\lceil(\frac{M}{\epsilon})^{1/\Delta}\frac{1}{\delta}\rceil$, 
        then we have $\|u_{I,b_j}-l_{I,b_{j+1}}\|_{L^2(P)} \leq \epsilon$ for all $j = -1,\ldots,N$ and $I \in \mathcal{I}$.
    \end{lemma}

    \begin{proof}[Proof of Lemma \ref{lem:ul-is-eps-bracket}]
        We bound
        \begin{align*}
            |u_{I,b_j}(x)-l_{I,b_{j+1}}(x)| &\leq \max_{w,v\in I}|(w^\T x-b_j)_+^k - (v^\T x-b_{j+1})_+^k|\\
            &\leq 
            \max_{w,v\in I} |(w^\T x-b_{j})_+^k - (v^\T x-b_{j})_+^k|
            +
            \max_{v\in I}|(v^\T x-b_j)_+^k - (v^\T x-b_{j+1})_+^k|.
        \end{align*}
        Using a telescoping sum,
        we may bound for $j \leq N$, using Lemma \ref{lem:bound-power-of-relu}:
        \begin{align*}
            |(w^\T x-b_j)_+^k - (v^\T x-b_j)_+^k|
            &= |(w^\T x-b_j)_+ - (v^\T x-b_j)_+|\times\sum\nolimits_{i=0}^{k-1} (w^\T x-b_j)_+^i (v^\T x-b_j)_+^{k-1-i}\\
            &\leq \|w-v\|_2 \|x\|_2 \times k\cdot \max\{(w^\T x)_+, (v^T x)_+\}^{k-1}\\
            &\leq k \|w-v\|_2 \|x\|_2^k\\
            &\leq 2kr \|x\|_2^k.
        \end{align*}
        Likewise,
        we may bound for $j \leq N-1$
        \begin{align*}
            |(v^\T x-b_j)_+^k - (v^\T x-b_{j+1})_+^k|
            &= |(v^\T x-b_j)_+- (v^\T x-b_{j+1})_+|\times \sum\nolimits_{i=0}^{k-1}(v^\T x-b_j)_+^i(v^\T x-b_{j+1})_+^{k-1-i}\\
            &\leq |b_j-b_{j+1}|\times k\cdot (v^\T x-b_j)_+^{k-1}\\
            &\leq k\delta \|x\|_2^{k-1}.
        \end{align*}
        Therefore,
        for $j \leq N-1$,
        \[
            |u_{I,b_j}(x)-l_{I,b_{j+1}}(x)| \leq k\delta\|x\|_2^{k-1} + 2kr\|x\|_2^k \leq 
            \begin{cases}
                k(\delta + 2r) &: \|x\|_2\leq 1\\
                k(\delta+2r)\|x\|_2^k \leq k(\delta+2r)\|x\|_2^{k+\Delta/2}&: \|x\|_2>1
            \end{cases},
        \]
        which can be simplified to
        \[
            |u_{I,b_j}(x)-l_{I,b_{j+1}}(x)| 
            \leq 
            k(\delta+2r)\max\{1, \|x\|_2^{k+\Delta/2}\}.
        \]
        This gives for $j \leq N-1$,
        \[
            \|u_{I,b_j}-l_{I,b_{j+1}}\|_{L^2(\P)}^2 
            \leq 
            k^2(\delta+2r)^2(1+\E_{X\sim P}\|X\|_2^{2k+\Delta}) \leq k^2(\delta+2r)^2 (1 + M) \leq \epsilon^2,
        \]
        where the final inequality is obtained by plugging in the values of $r$
        and $\delta$ specified in the statement of the lemma.  Meanwhile, for $j
        = N$, we have $b_{j+1} = \infty$, so that $l_{I,b_{j+1}} = 0$. 
        Also, $b_j = N\delta \geq (M/\epsilon)^{1/\Delta}$.
        Thus,
        \begin{align*}
            \|u_{I,b_N}-l_{I,b_{N+1}}\|_{L^2(\P)}^2
            &= \int_{\R^d} u_{I,b_N}(x)^2 \de \P(x) \\
            &\leq \int_{\R^d} (\|x\|_2-b_N)_+^{2k} \de \P(x) 
            = \int_{\R^d} (\|x\|_2-N\delta)_+^{2k} \de \P(x) \\
            &\leq 
            \Big(\int_{\R^d} (\|x\|_2-N\delta)^{2k+\Delta} \de \P(x)\Big)^{\frac{2k}{2k+\Delta}} \Big(\int_{\R^d} 1(\|x\|_2\geq N\delta)\; \de \P(x)\Big)^{\frac{\Delta}{2k+\Delta}}\\
            &= 
            M^{\frac{2k}{2k+\Delta}} \times \P\Big(\|X\|_2\geq C\delta\Big)^{\frac{\Delta}{2k+\Delta}}\\
            &\leq 
            M^{\frac{2k}{2k+\Delta}} \times \Big(\frac{M}{(N\delta)^{2k+\Delta}}\Big)^{\frac{\Delta}{2k+\Delta}} 
            = 
            \frac{M}{(N \delta)^\Delta},
        \end{align*}
        where the second inequality uses H\"older's inequality, and the last inequality used Markov inequality. With the value of $\delta$ and $N$ specified by the lemma, we conclude \smash{$\|u_{I,b_N}-l_{I,b_{N+1}}\|_{L^2(\P)}^2\leq\epsilon^2$}.  
    \end{proof}
    
    We can now complete the proof of Theorem \ref{thm:bracketing-number}. By \citet{dumer2007covering}, there exists an $r$-covering of $\S^{d-1}$ with cardinality at most $(1 + 2/r)^d$. Thus, using the choices of $r,\delta$, and $N$ in Lemma \ref{lem:ul-is-eps-bracket} gives an $(\epsilon,\P)$-bracket of $\cG_k$ that has cardinality at most
    \[
    	\Big(1 + \frac{2k\sqrt{1 + \max(1,M)}}{\epsilon/3}\Big)^d\Big(\big(\frac{M}{\epsilon}\big)^{1/\Delta}\;\frac{k\sqrt{1+\max(1,M)}}{\epsilon/3}+1\Big) \leq C (1+1/\epsilon)^{d+1+1/\Delta}.
    \]
    with some constant $C$. The proof is completed.
\end{proof}

\subsection{Bracketing number for $\cG_0$}

\begin{theorem}
	\label{thm:bracketing-number-k0}
	Let $\P$ be a probability measure with finite $\Delta\th$ moments, i.e.,
  \smash{$\E_{X\sim \P} \|X\|_2^{\Delta} = M <\infty$}, for any fixed
  $\Delta>0$. 
	 Suppose additionally that for the random vector \smash{$X\sim P$} and any
   \smash{$w\in \S^{d-1}$}, the random variable \smash{$w^\T X\in\R$} has the
   density upper bounded by \smash{$M_\infty$}, i.e.,
   \begin{equation}
   	\label{eq:density-ub-M-infty}
   	\sup_{w \in \S^{d-1}} \big\| p_{w^\T x} \big\|_\infty <
   	M_\infty,
   \end{equation}
   where \smash{$p_{w^\T x}$} denotes the density of $w^\T x$ for $x \sim P$, and \smash{$\|f\|_\infty = \sup_{t \in \R} |f(t)|$} denotes the supremum norm of $f$. 
    Then, for the class $\cG_0$, there is a constant
  $C_\bracket>0$ depending only on $M$, $M_{\infty}$, $d$, and $\Delta$ such
  that 
	\[
		\log N_\bracket(\epsilon; \cG_0, \|\bcdot\|_2) 
		\leq  C_\bracket \log\Big(1 + \frac{1}{\epsilon}\Big).
	\]
\end{theorem}

\begin{proof}[Proof of Theorem \ref{thm:bracketing-number-k0}]
	We adopt the same notation used in the proof of Theorem~\ref{thm:bracketing-number}. The set \eqref{eqn:bracket} is a bracket of $\cG_0$ by Lemma~\ref{lem:ul-is-a-bracket}. Now we show that it is an $\epsilon$-bracket for appropriate choices of $r,\delta$, and $N$.
	
	
	Select $I \in \mathcal{I}, j = -1,\ldots,N - 1$. There exists a $w_0 \in I$ such that $\|w - w_0\|_2 \leq r$ for all $w \in I$, since $S$ is an $r$-cover of \smash{$\mathbb{S}^{d - 1}$}. Consequently, for any $x \in \Rd$
	\begin{align*}
	\max_{w \in I}\; 1(w^{\T}x \geq b_j) \;&\leq\; 1(w_0^{\T} x \geq b_j - r \|x\|_2),
	\\
	\min_{v \in I}\; 1(v^{\T}x \geq b_{j + 1}) \;&\geq\; 1(w_0^{\T} x \geq r \|x\|_2 + b_{j + 1})
  	\end{align*}
  	Note that the random variable $w_0^{\T}X \in \Reals$ has a density that is
    upper bounded by $M_{\infty}$ as \eqref{eq:density-ub-M-infty}. Therefore,
    for any $R > 0$, 
  	\begin{align*}
  	\|u_{I,b_j} - l_{I,b_j}\|_{L^2(\P)}^2 
  	& = \P\Bigl(\max_{w,v}\; \Bigl|1(w^{\T}x \geq b_j) - \min_{v \in I}\; 1(v^{\T}x \geq b_{j + 1}) \Bigr|\Bigr) \\
  	& \leq \P\Big(w_0^{\T}X \in [b_j - r \|X\|_2, b_{j + 1} + r \|X\|_2]\Bigr) \\
  	& \leq \P\Big(w_0^{\T}X \in [b_j - r  R, b_{j + 1} + r R]\Bigr) + \mathbb{P}\Bigl(\|X\|_2 \geq R\Bigr) \\
  	& \leq M_{\infty}\Bigl(\delta + 2 r R\Bigr) + \mathbb{P}\Bigl(\|X\|_2 \geq R\Bigr) \\
  	& \leq M_{\infty}\Bigl(\delta + 2 r R\Bigr) + \frac{M}{R^{\Delta}},
  	\end{align*}
  	with the last line following from Markov's inequality. Taking \smash{$R = (3M/\epsilon^2)^{1/\Delta}$, $r = \frac{\epsilon^2}{6RM_{\infty}} = \frac{\epsilon^{2(1 + 1/\Delta)}}{6(3M)^{1/\Delta}M_{\infty}}$} and \smash{$\delta = \frac{\epsilon^2}{3M_{\infty}}$} makes the above upper bound equal to $\epsilon^2$. 
  	
  	On the other hand for $I \in \mathcal{I}, j = N$, we have \smash{$l_{w,b_N + 1} = 0$}, and 
  	\begin{align*}
  	\|u_{I,b_N} - l_{I,b_{N + 1}}\|_P^2 
  	& = \mathbb{P}\Bigl(\max_{w \in I}w^{\T}X \geq b_N\Bigr) \\
  	& \leq \mathbb{P}\Bigl(\|X\|_2 \geq b_N\Bigr) \\
  	& \leq \frac{M}{(N\delta)^{\Delta}}.
  	\end{align*}
  	Taking \smash{$N = \frac{(M/\epsilon^2)^{1/\Delta}}{\delta}$} makes this upper bound equal to $\epsilon^2$. So for these choices of $r,\delta$ and $N$ the bracket is an $\epsilon$-bracket.
  	
  	Finally, arguing as in the proof of Theorem~7, we conclude that there exists an $\epsilon$-bracket of $\cG_0$ with cardinality at most 
  	$$
  	N \times \Bigl(1 + \frac{2}{r}\Bigr)^{d} \leq C (1 + \frac{1}{\epsilon})^{2d(1 + 1/\Delta) + 2(1 + 1/\Delta)}.
  	$$
  	This completes the proof.	
\end{proof}

\subsection{Proof of Theorem \ref{thm:null-distribution}}

Theorem \ref{thm:null-distribution} is a consequence of the
$(\epsilon,\P)$-bracketing number bound of Theorem \ref{thm:bracketing-number}
(or Theorem~\ref{thm:bracketing-number-k0}, for $k=0$) and Theorems 7.6 and 9.1
in  \citet{dudley2014uniform}, which we copy below for convenience. 

\begin{theorem}
\label{thm:dudley-empirical-process}
    Let $(\Omega, \Sigma, P)$ be a probability space and let $\cF\subseteq \cL^2(\Omega,\Sigma,P)$ be such that
    \[
        J_\bracket(\cF, \|\bcdot\|_2) = \int_0^1 \sqrt{\log N_\bracket(\epsilon; \cF, \|\bcdot\|_{L^2(P)})} \; \de \epsilon < \infty.
    \]
    Then, under the $P$-null hypothesis, as $m,n\rarrow\infty$,
    \[
        \sqrt{\frac1m + \frac1n}\, (P_m - Q_n)\convd \G_P,
        \;\;\text{and hence}\;\;
        \sqrt{\frac1m + \frac1n}\, \sup_{f\in \cF} \big| P_m f -Q_n f\big| \convd \sup_{f\in\cF} |\G_P(f)|.
    \]
\end{theorem}
Theorem \ref{thm:dudley-empirical-process} implies Theorem
\ref{thm:null-distribution} because $\int_0^1 \sqrt{\log(1+1/\epsilon)} \, \de
\epsilon < \infty$.

\section{Proof of Theorem \ref{thm:asymptotic-power}}
\label{sec:proof-asymptotic-power}

Our proof follows closely the proof of the analogous result for the univariate
higher-order KS test in \citet{sadhanala2019higher}. We begin by recalling the
following lemma, which appears as Theorem 2.14.2 and 2.14.5 in
\citet{vandervaart1996weak}, the statement of which we transcribe from
\citet{sadhanala2019higher}. 

\begin{lemma}
\label{lem:concentration}
    Let $\cF$ be a class of functions with an envelope function $F$; i.e., $f \leq F$ for all $f \in \cF$.
    Define
    \begin{equation*}
      W = \sqrt{n} \sup\nolimits_{f \in \cF} |P_n f - P f |,
    \end{equation*}
    and let $J = \int_0^1 \sqrt{\log N_{[]}(\epsilon; \cF,\|\cdot\|_{L^2(P)})} \de \epsilon$.
    If for $p \geq 2$, $\| F \|_{L^p(P)} < \infty$,
    then for a universal constant $c_1 > 0$,
    \begin{equation*}
      \E[|W|^p]^{1/p}
        \leq
        c_1 \big(
          \| F \|_{L^2(P)} J + n^{-1/2 + 1/p}\| F \|_{L^p(P)}
        \big).
    \end{equation*}
    If, for $0 < p \leq 1$, the exponential Orlicz norm of $F$, i.e., $\| F \|_{\Psi_p} = \inf\{ t > 0 : \E[\exp(|F|^p/t^p)] \leq 2 \}$, is finite, then
    \begin{equation*}
      \| W \|_{\Psi_p}
        \leq 
        c
        \big(
          \| F \|_{L^2(P)} J
          +
          n^{-1/2}(1 + \log n)^{1/2} \| F \|_{\Psi_p}
        \big).
    \end{equation*}
\end{lemma}

Using this lemma, we have the following theorem on the concentration of the the
RKS test statistic on its population value, from which the asymptotic result in
Theorem \ref{thm:asymptotic-power} can be deduced.

\begin{theorem}
\label{thm:tail-bounds}
    We have the following tail bounds.

    \begin{enumerate}[(i)]

      \item
      Consider $k = 0$.
      Assume $P$ and $Q$ have $\Delta\th$ moment bounded by $M$ for $M,\Delta > 
      0$, are absolutely continuous with respect to Lebesuge measure, and
      satisfies the condition in \eqref{eq:density-ub-M-infty}.
       Then, for some constant $c_0$ depending
      only on $M,\Delta,M_\infty$, and $d$, the following holds with probability
      at least $1-\alpha$:  
      \[
        |T_{d,k}-\rho(P, Q;\cF_k)|
          \leq
          c_0(\log(1/\alpha)) \Big(\frac{1}{\sqrt{m}}+\frac{1}{\sqrt{n}}\Big).        
      \]

      \item 
      Consider $k\geq 1$.
      Assume $P$ and $Q$ have finite $p\th$ moments bounded by $M$ for $p = 2k +
      \Delta$ for $M,\Delta > 0$. Then, for some constant $c_0$ depending only
      on $M,\Delta,k$, and $d$, the following holds with probability at least
      $1-\alpha$:  
      \[
          |T_{d,k}-\rho(P, Q;\cF_k)|
          \leq
          \frac{c_0}{\alpha^{1/p}} \Big(\frac{1}{\sqrt{m}}+\frac{1}{\sqrt{n}}\Big).
      \]
      Assume, on the other hand, that $X$ and $Y$ have Orlicz norm of order $0 <
      p \leq 1$ bounded by $M$ when $X \sim P$ and $Y \sim Q$. Then, for some
      constant $c_0$ depending only on $M,p,k$, and $d$, the following holds
      with probability at least $1-\alpha$:  
      \begin{equation*}
        |T_{d,k}-\rho(P, Q;\cF_k)|
          \leq
          c_0(\log(1/\alpha))^{1/p} \Big(\frac{1}{\sqrt{m}}+\frac{1}{\sqrt{n}}\Big).
      \end{equation*}
    \end{enumerate}
    
\end{theorem}

\begin{proof}[Proof of Theorem \ref{thm:tail-bounds}]
  First consider \smash{$k \geq 1$}.
  Note that a bound on any Orlicz norm of order \smash{$0 < p \leq q$} implies a bound 
  on the moments of order \smash{$2k + \Delta$ of $P,Q$ for $\Delta  = 1$}. 
  Thus, assuming either the moment bound or the Orlicz norm bound,
  Theorem \ref{thm:bracketing-number} implies \smash{$J = \int_0^1 \sqrt{\log
    N_{[]}(\epsilon; \cG_k,\|\cdot\|_{L^2(P)})} \de \epsilon$} has a finite upper
  bound depending only on the moments of order \smash{$2k + \Delta$} of $P,Q$, and on
  \smash{$d,k,\Delta$}. Take \smash{$F(x) = \| x \|_2^k$}.
  If we only assume a bound on the moments of order \smash{$2k + \Delta$} of $P,Q$, then 
  take \smash{$p = 2k + \Delta \geq 2$}. In this case, \smash{$n^{-1/2 + 1/p} \leq 1$},
  whence using Lemma \ref{lem:concentration} and Markov's inequality, 
  we see that \smash{$\sup_{f \in \cF} |P_n f - P f | \leq c_0 / (\alpha^{1/p}\sqrt{n})$} and $\sup_{f \in \cF} |Q_n f - Q f | \leq c_0 /(\alpha^{1/p} \sqrt{m})$ with probability at least \smash{$1-\alpha$}.
  The asserted bound then holds by the triangle inequality.
  On the other hand, if we assume a bound on the Orlicz norm of order $p$, Lemma \ref{lem:concentration} implies a constant upper bound on the Orlicz norm of order $p$ of \smash{$\sqrt{n} \sup_{f \in \cF} |P_n f - P f |$} and \smash{$\sqrt{m} \sup_{f \in \cF} |Q_n f - Q f |$} because \smash{$(1+\log n)^{1/2} / n^{1/2} \leq 1$}.
  This gives us the desired tail bound by taking \smash{$t = c_0(\log(1/\alpha))^{1/p}$}  in the supremum defining the Orlicz norm and Markov's inequality.

  Now consider $k = 0$. We take $F$ as a constant 1 function, i.e., $F = 1$, and using Theorem \ref{thm:bracketing-number-k0} to bound the integral $J$ in this case. Because $F = 1$ has a finite Orlicz norm of order $p = 1$, the asserted tail bound follows from Lemma \ref{lem:concentration}.
\end{proof}

We show how this result implies Theorem \ref{thm:asymptotic-power}.  By Theorem
\ref{thm:null-distribution}, if \smash{$P = Q$}, then \smash{$T_{d,k} =
  O_p(\sqrt{(n+m)/(nm)}$}.  Because \smash{$1/t_{m,n} = o(\sqrt{n+m}) =
  o(\sqrt{(nm)/(n+m)})$}, we see that the test rejects with asymptotic
probability 0.  On the other hand, if \smash{$P \neq Q$}, then by Theorem 
\ref{thm:rtv-ipm-metric}, we have \smash{$\rho(P,Q;\cF_k) > 0$}.  By Theorem
\ref{thm:tail-bounds}, we have \smash{$T_{d,k} \stackrel{p}\to
\rho(P,Q;\cF_k)$}.  Because \smash{$t_{m,n} \to 0$}, this implies that the test
rejects with asymptotic probability 1. 

\newcommand{\pell}{{(\ell)}}
\newcommand{\ph}{{(h)}}
\newcommand{\wb}{{w,b}}
\newcommand{\pwb}{{(\wb)}}
\newcommand{\wbkappa}{{w_\kappa,b_\kappa}}
\newcommand{\pwbkappa}{{(\wbkappa)}}
\newcommand{\mn}{{m+n}}
\newcommand{\pmn}{{(m+n)}}
\newcommand{\Tdk}{{T_{d,k}}}
\newcommand{\Tdkpi}{{T_{d,k}^\pi}}
\newcommand{\Tdkcpi}{{T_{d,k,\Ceps}^\pi}}
\newcommand{\Tdkccircpi}{{T_{d,k,\Ccirc}^\pi}}
\newcommand{\Tdkcgoodpi}{{T_{d,k,\Cgood}^\pi}}
\newcommand{\sumell}{{\sum_{\ell=1}^\mn}}
\newcommand{\unifsmn}{{\mathrm{Unif}(S_\mn)}}
\newcommand{\Ceps}{{\mathcal C_\epsilon}}
\newcommand{\Neps}{{N_\epsilon}}
\newcommand{\Cgood}{{\mathcal C_\mathrm{good}}}
\newcommand{\Ngood}{{N_{\mathrm{good}}}}
\newcommand{\Ccirc}{{\mathcal C_\circ}}
\newcommand{\Ncirc}{{N_\circ}}

\section{Proof of Theorem \ref{thm:asymptotic-perm}}
\label{app:asymptotic-perm}

We will develop the following nonasymptotic statement about type II error
control using permutations, from which the asymptotic statement in Theorem
\ref{thm:asymptotic-perm} can be deduced. 

\begin{theorem}
	\label{thm:non-asymptotic-perm}
	Assume $\popP$ and $\popQ$ have the bounded domains included in the $L$-radius
  Euclidean ball whose center is the origin in $\R^d$. If $k = 0$, additionally
  assume $\popP$ and $\popQ$ each has density satisfying
  \eqref{eq:density-ub-M-infty}. Then, there exists a constant $C$ which only
  depends on $(L,d,k)$ when $k\geq1$ or $(L,d,k,M_\infty)$ when $k=0$, such that
  if $m, n\geq 2$ and $\beta\in(0,1]$ satisfies 
	\begin{equation}
	\label{eq:rho-condition-for-perm-type2}
		\rho(P, Q; \cF_k) \geq C \sqrt{\frac1m+\frac1n} \times \Big(\log(m+n)
    +\log\Big(\frac1\alpha\Big) + \frac{1}{\beta^{1/(2k+1)}}\Big), 
	\end{equation}
	then the permutation test with the p-value defined in \eqref{eq:test-perm}
  with the full permutations controls the type II error, i.e., under the
  alternative hypothesis \smash{$H_1: \popP\neq \popQ$}, the p-value $p$ in 
  \eqref{eq:test-perm} satisfies $\P_{H_1} (p \leq \alpha) \geq 1-\beta$. 
\end{theorem}

We show how this can be used to prove Theorem \ref{thm:asymptotic-perm}. Choose
any $\beta\in(0,1]$. Due to the condition on $\alpha_{m,n}$, we know that for
large enough $m$ and $n$, \eqref{eq:rho-condition-for-perm-type2} is
satisfied. Therefore, 
\[
	\lim_{m,n \rightarrow \infty} \P_{H_1}(p\leq \alpha_{m,n}) \geq 1-\beta.
\]
This implies \smash{$\lim_{m,n\rightarrow\infty} \P_{H_1}(p\leq\alpha_{m,n}) =
  1$} since $\beta$ is arbitrary in $(0,1]$. Combined with the finite-sample
type I error property of permutation tests, and the condition
\smash{$\alpha_{m,n} \to 0$}, this proves Theorem \ref{thm:asymptotic-perm}. The
rest of this section is devoted to proving Theorem
\ref{thm:non-asymptotic-perm}.   

\subsection{Additional notation} 

Define $S_\mn$ as a set of all possible permutations for $(1, \ldots, \mn)$. Recall the notations in \eqref{eq:test-perm}: $\pi$ denotes a permutation such that \smash{$\pi\in S_\mn$}; \smash{$z=(z_\ell)_{\ell=1}^\mn=(x_1, \ldots, x_m, y_1, \ldots, y_n)$}; applying permutation $\pi$ to $z$ outputs \smash{$z^\pi$}. To emphaisze the dependency on the data, we may use the notation \smash{$\Tdk(z)$} and \smash{$\Tdk(z^\pi)$}.
Define a vector 
\[
	a=(a_\ell)_{\ell=1}^\mn=( \underbrace{1/m, \ldots, 1/m}_{m}, \underbrace{-1/n, \ldots, -1/n}_{n}).
\]
Note that \smash{$\sum_{\ell=1}^\mn a_\ell = 0$}. With the permutation $\pi$, we have \smash{$a^\pi = (a_{\pi(\ell)})_{\ell=1}^\mn$}. Define \smash{$\theta\pwb=(\theta\pwb_\ell)_{\ell=1}^\mn$} and \smash{$\theta\pwb_\ell = (w^\T z_\ell -b)_+^k$}. For brevity, we write \smash{$\theta\pwb$} instead of $\theta(w,b;z)$ although \smash{$\theta\pwb$} is a function of $z$. 
We can rewrite \smash{$\Tdk(z)$} with \smash{$\theta\pwb$}:
\[
\Tdk(z) = \max_{\wb\in\nonnegRD} \Big| \frac{1}{m}\sum_{i = 1}^{m} 
(w^{\T}x_i -  b)_{+}^{k} - \frac{1}{n}\sum_{i = 1}^{n} (w^{\T}y_i -
b)_{+}^{k} \Big|
= \max_{\pwb\in\nonnegRD} \Big| \langle \theta\pwb, a\rangle \Big|.
\]
Moreover, we can rewrite the permuted test statistic \smash{$\Tdk(z^\pi)$} in a similar format:
\begin{equation}
	\label{eq:tdkpi-def-1}
	\Tdkpi(z) = \Tdk(z^\pi) = \max_{\pwb\in\nonnegRD} \Big| \langle \theta\pwb, a^\pi\rangle \Big|.
\end{equation}
Note that \smash{$\Tdkpi(z)$} is a function of $\pi$ when $z$ is fixed. Define the $(1-\alpha)$-quantile of all permuted test statistics \smash{$(\Tdkpi(z))_{\pi\in S_\mn}$}, given $z$, as the following:
\begin{equation}
	\label{eq:t-1-alpha-quantile}
	t_{1-\alpha}(z) = \inf\Big\{t \; :\; \P_{\pi\sim\unifsmn} \Big(\Tdkpi(z) \leq t \, \Big| \, z\Big) \geq 1-\alpha\Big\}.
\end{equation}
Define a random variable \smash{$X\pwb= \langle \theta\pwb, a^\pi \rangle$}, which is a function of \smash{$z\sim \popP^m\times\popQ^n$} and \smash{$\pi\sim \unifsmn$}. However, we use the notation $X\pwb$ instead of \smash{$X(w,b; z^\pi)$}, again for brevity. Then, \smash{$\Tdkpi(z)$} can be also written as
\begin{equation}
	\label{eq:tdkpi-def-2}
	\Tdkpi(z) = \max_{\pwb\in\nonnegRD} \Big| X\pwb\Big|.
\end{equation}

\subsection{Proof of Theorem \ref{thm:non-asymptotic-perm}}
\label{apdx:nonasymp-perm-proof}

We need the following result, which will be proved later. 
\begin{lemma}
	\label{lem:tdkpi-upper-bound}
	Assume the same as Theorem \ref{thm:asymptotic-perm}. For $k\geq0$, define a constant
	\begin{equation}
		\label{eq:def-tmna}
	t_{m,n,\alpha} = C_{L,d,k} \sqrt{\frac1m+\frac1n}\Big(\log\Big(\frac1\alpha\Big) + \log\pmn\Big)
	\end{equation}
	where $C_{L,d,k(,M_\infty)}$ is a constant which only depends on $(L,d,k)$ for $k\geq1$ and $(L,d,k,M_\infty)$ for $k=0$.
	Recall the definition of $t_{1-\alpha}(z)$ in \eqref{eq:t-1-alpha-quantile}.
	Then, for any given samples $z$, we have $t_{1-\alpha}(z) \leq t_{m,n, \alpha}$.
\end{lemma}

Assuming Lemma \ref{lem:tdkpi-upper-bound} holds, we can prove Theorem \ref{thm:non-asymptotic-perm}.
Since $P$ and $Q$ have bounded domain, they have finite moments of order $2k+1$. Also, for $k=0$, we have additionally assumed the boundedness of the density $p_{w^\T x}$. Thus, by Theorem \ref{thm:tail-bounds}, there exists a constant $c_{L,d,k(,M_\infty)}$ which only depends on $(L,d,k)$ for $k\geq1$ or $(L,d,k,M_\infty)$ for $k=0$, such that
\[
\Big| T_{d,k}-\rho(P, Q;\cF_k) \Big|
\leq
\frac{c_{L,d,k(,M_\infty)}}{\beta^{1/(2k+1)}} \Big(\frac{1}{\sqrt{m}}+\frac{1}{\sqrt{n}}\Big)
\quad\text{with probability at least}\; 1-\beta.
\]
Note that the randomness from the above equation comes from \smash{$x_i\simiid \popP$} and \smash{$y_j\simiid\popQ$}.
By choosing the constant $C$ in Theorem \ref{thm:asymptotic-perm} as
\smash{$C = \max\{C_{L,d,k} \;,\; c_{L,d,k(,M_\infty)}\}$},
we can easily observe that
\[
t_{m,n,\alpha} \leq \rho(\popP,\popQ; \cF_k)  - \frac{c_{L,d,k(,M_\infty)}}{\beta^{1/(2k+1)}} \Big(\frac{1}{\sqrt{m}}+\frac{1}{\sqrt{n}}\Big).
\]
Therefore, under the alternative hypothesis $H_1$, we have
\begin{equation}
	\label{eq:tdk-lower-bound}
	\begin{aligned}
		\P_{z} \Big(\Tdk(z) < t_{m,n,\alpha}\Big) 
		& \leq
		\P_{z}  \Big(\Tdk(z) \leq \rho(\popP, \popQ; \cF_k) - \frac{c_{L,d,k(,M_\infty)}}{\beta^{1/(2k+1)}} \Big(\frac{1}{\sqrt{m}}+\frac{1}{\sqrt{n}}\Big) \Big) 
		\leq \beta.
	\end{aligned}
\end{equation}
Finally combining \eqref{eq:tdk-lower-bound} with Lemma \ref{lem:tdkpi-upper-bound}, we have
\begin{align*}
	\P_z \Big(t_{1-\alpha}(z)  > \Tdk(z)\Big) 
	& \leq \P\Big( t_{1-\alpha}(z) > t_{m,n,\alpha}\Big) + 
	\P\Big(\Tdk(z) < t_{m,n,\alpha} \Big)
	\leq 0+\beta = \beta,
\end{align*}
which concludes the proof of Theorem \ref{thm:non-asymptotic-perm}.
Now to prove Lemma \ref{lem:tdkpi-upper-bound} is the only remaining part.

\subsection{Proof sketch of Lemma \ref{lem:tdkpi-upper-bound}}
\label{apdx:sketch-lemma-tmna}

The proof Lemma \ref{lem:tdkpi-upper-bound} follows techniques developed in \citet{green2024two}. Since the proof consists of multiple steps, we introduce the outline of the proof here before we dive in.
\begin{enumerate}
	\item Fix $z$. Find a ``good'' finite subset \smash{$\Cgood$}
	of \smash{$\nonnegRD$}, of cardinality $\Ngood = |\Cgood| < \infty$.
	For any \smash{$\mathcal C\subseteq \nonnegRD$}, define 
	\begin{equation}
	\label{eqn:tdkcpi-def}
		T_{d,k,\mathcal C}^\pi (z) = \max_{\pwb\in \mathcal C} \Big| \langle \theta\pwb, a^\pi \rangle \Big| = \max_{\pwb\in\mathcal C} \Big|X\pwb\Big|.
	\end{equation}
	We control $\Ngood$ and the approximation error \smash{$|T_{d,k,\Cgood}^\pi(z) - \Tdkpi(z)|$}.
	\item Find a high-probability upper bound for $X\pwb$ for each $\pwb\in
    \Cgood$. Here $z$ is fixed, and hence the only randomness is rooted from
    $\pi\sim\unifsmn$. 
	\item Find a high-probability upper bound for \smash{$\Tdkcgoodpi(z)$}. 
    Again, $\pi$ is the only source of randomness. 
	\item Lastly, we get a high-probability upper bound for \smash{$\Tdkpi(z)$},
    with the results from the first and third steps. 
\end{enumerate}

\subsection{Proof of Lemma \ref{lem:tdkpi-upper-bound}}

The proof follows the steps explained in Appendix
\ref{apdx:sketch-lemma-tmna}. Throughout the whole proof, we consider $z$ is 
fixed. 

\subsubsection{$\Cgood$ and the approximation error \smash{$|\Tdkcgoodpi(z) -
    \Tdkpi(z)|$}}

Here we will prove the following Lemma.

\begin{lemma}
	\label{lem:cgood}
	We have the following finite subsets of $\Ccirc$ and $\Ceps$ of $\nonnegRD$, whose sizes are $\Ncirc = |\Ccirc|$ and $\Neps = |\Ceps|$.
	\begin{itemize}
		\item Consider $k=0$. There exists $\Ccirc$ such that
		\begin{equation}
			\label{eq:Ccirc-lemma}
			\log \Ncirc  \leq (d+1) \log(1+\mn)
			\quad\text{and}\quad
			\Big|\Tdkpi(z) - T_{d,k,{\mathcal C_\circ}}^\pi(z) \Big| = 0.
		\end{equation}
		\item Consider $k\geq1$. For any $\epsilon > 0$, there exists $\Ceps$ such that
		\begin{equation}
			\label{eq:Ceps-lemma}
			\log N_\epsilon  \leq C'_{L,d,k} \log\Big(1+\frac{\mn}{\epsilon}\Big)
			\quad\text{and}\quad
			\Big| \Tdkpi(z)  - \Tdkcpi(z) \Big| \leq \Big(\frac1m+\frac1n\Big) \epsilon,
		\end{equation}
		where $C'_{L,d,k}$ is a constant only depends on $(L,d,k)$.
	\end{itemize}	
	We will use $\Cgood$ to denote $\Ccirc$ when $k=0$ or $\Ceps$ when $k\geq1$.
\end{lemma}

\paragraph{Case $k=0$.} 

Define $\Ccirc= \{(w_i,b_i): i=1,\ldots, N_\circ\}$ to satisfy the following condition:
\[
\{\theta(w_i,b_i): (w_i, b_i) \in \Ccirc\}= \{\theta\pwb: \pwb \in \nonnegRD\}.
\]
Note that such \smash{$\Ccirc$} exists since \smash{$\theta(w,b)_\ell\in\{0,1\}$} for any \smash{$\pwb\in\nonnegRD$} and \smash{$1\leq\ell\leq m+n$}.

An upper bound for $\Ncirc$ can be obtained from a simple observation about the Vapnik–Chervonenkis (VC) dimension and the shattering number\footnote{It is also called the growth function or the shatter coefficient.} of \smash{$\mathcal H = \{1(w^\T x+b\geq0) : \pwb\in\nonnegRD\}$}.
It is well known that the VC dimension of \smash{$\{1(w^\T x+b\geq0) : \pwb\in\RadonDomain\}$} is at most \smash{$d+1$}.
Thus the VC dimension of $\mathcal H$, saying \smash{$V_\mathcal H$}, is at most \smash{$d+1$} as well. Therfore the shattering number 
\begin{align*}
	S_\mathcal H(\mn)
	&= \max_{v_1, \ldots, v_\mn \in \R^d} 
	\Big| \Big\{ \big(h(v_1), \ldots, h(v_\mn)\big)^\T \in\R^\mn: h\in\mathcal \mathcal H \Big\} \Big| \\
	&= \max_{v_1, \ldots, v_\mn \in \R^d} 
	\Big| \Big\{ \big((w^\T v_1 - b)_+^0, \ldots, (w^\T v_\mn - b)^0\big)^\T \in\R^\mn: \pwb \in \nonnegRD \Big\} \Big| 
\end{align*}
satisfies \smash{$S_\mathcal H(\mn) \leq (1+\mn)^{V_\mathcal H}\leq (1+\mn)^{d+1}$}. This implies that we can find $\mathcal C_\circ$ such that
\begin{equation*}
	\log \Ncirc = \log S_\mathcal H(\mn) \leq (d+1) \log(1+\mn).
\end{equation*}

Also, by the definition of $\mathcal C_\circ$ it is clear that the approximation error with $\mathcal C_\circ$ is zero:
\[
\Big|\Tdkpi(z) - T_{d,k,{\mathcal C_\circ}}^\pi(z) \Big| = 0.
\]

\paragraph{Case $k\geq1$.}

Recall that in Appendix \ref{supp:uclt-prelim}, we introduced the $(\epsilon, \P)$-bracketing number for a class $\cF$ and a measure $\P$.
Define a discrete measure \smash{$R_\mn=\frac1\mn\sum_{\ell=1}^\mn\delta(z_\ell)$}. 
Since \smash{$\|z_i\|_2\leq L$} for any \smash{$i\in\{1,\ldots,\mn\}$}, the
measure $R_\mn$ has finite moments of order $2k+1$ for any \smash{$k\geq 1$}, 
upper bounded by \smash{$L^{2k+1}$}. 

Thus we can apply Lemma \ref{lem:ul-is-a-bracket} and Theorem \ref{thm:bracketing-number} for \smash{$(\epsilon, R_\mn)$}-bracketing number. Following Lemma \ref{lem:ul-is-a-bracket}, we can construct \smash{$(\epsilon, P_m)$}-bracket of $\cG_k$ of the form \smash{$[\ell_{I,b_{j+1}}, u_{I,b_j}]$} as \eqref{eqn:bracket}. 
Denote such bracket $K_\epsilon$.
Theorem \ref{thm:bracketing-number} implies the following upper bound on $|K_\epsilon|$, where \smash{$C'_{L,d,k}$} is a constant only depending on $(L,d,k)$:
\[
	\log |K_\epsilon| \leq  C'_{L,d,k}\log\Big(1 + \frac{1}{\epsilon}\Big).
\]

Using $K_{\epsilon/\pmn}$, we can construct $\Ceps$,
a subset of $\nonnegRD$ which is going to be used instead of $\nonnegRD$ with certain approximation error (which will be controlled later).
For each $\kappa = [\ell_{I,b_{j+1}}, u_{I,b_j}]\in K_{\epsilon/\pmn}$, choose any $w_\kappa \in I$ and $b_\kappa \in (b_j, b_{j+1}]$. Then define $\Ceps= \{\pwbkappa: \kappa\in K_{\epsilon/\mn}\}$. Also, $\Neps = |\Ceps|$ satisfies the following upper bound:
\begin{equation*}
	\log N_\epsilon = \log |K_{\epsilon/\pmn}| \leq C'_{L,d,k} \log\Big(1+\frac{\mn}{\epsilon}\Big).
\end{equation*}

Now we will control the approximation error of using $\mathcal C_\epsilon$ instead of $\nonnegRD$. Recall that \smash{$\Tdkcpi$} is defined as \eqref{eqn:tdkcpi-def}.
Choose any \smash{$\pwb\in\nonnegRD$}. By the definition of the bracket, there exists \smash{$\kappa=[\ell_{I,b_{j+1}}, u_{I,b_j}]\in K_{\epsilon/\pmn}$} such that \smash{$\ell_{I,b_{j+1}}(x) \leq (w^\T x - b)_+^k \leq u_{I,b_j}(x)$} for any \smash{$x\in\R^d$}. Meanwhile, by the construction of \smash{$(w_\kappa, b_\kappa)$}, it is clear that \smash{$\ell_{I,b_{j+1}}(x) \leq (w_\kappa^\T x - b_\kappa)_+^k \leq u_{I,b_j}(x)$}. Therefore, by the definition of the bracket,
\[
	\Big\|(w^\T \cdot - b)_+^k - (w_\kappa^{\T} \cdot - b_\kappa)_+^k \Big\|_{L^2(R_\mn)}^2
	\leq 
	\Big\| \ell_{I,b_{j+1}} - u_{I,b_j} \Big\|_{L^2(R_\mn)}^2 \leq \frac{\epsilon^2}{\pmn^2},
\]
and hence, due to Cauchy-Schwarz inequality,
\begin{align*}
	\bigg(\Big\| \theta\pwb -\theta\pwbkappa \Big\|_1 \bigg)^2
	&=
	\bigg(\sum_{\ell=1}^\mn \Big| \theta\pwb_\ell - \theta\pwbkappa_\ell\Big|\bigg)^2 
	\leq 
	\pmn \sum_{\ell=1}^\mn \Big| \theta\pwb_\ell - \theta\pwbkappa_\ell \Big|^2 \\
	&= \pmn \sum_{\ell=1}^\mn \Big| (w^\T z_\ell - b)_+^k - (w_\kappa^{\T} z_\ell - b_\kappa)_+^k \Big|^2 \\
	&= \pmn^2 \Big\|(w^\T \cdot - b)_+^k - (w_\kappa^{\T} \cdot - b_\kappa)_+^k \Big\|_{L^2(R_\mn)}^2
	\leq\epsilon^2.
\end{align*}
Finally, we can check the upper bound for the approximation error of using $(w_\kappa, b_\kappa)$ instead of $(\wb)$:
\begin{align*}
	\Big| |X\pwb| - |X\pwbkappa| \Big| 
	\leq |X\pwb - X\pwbkappa|
	&= \Big|  \sum\nolimits_{\ell=1}^\mn \Big(\theta\pwb_\ell - \theta\pwbkappa_\ell \Big) a^\pi_\ell \Big| \\
	& \leq \sum\nolimits_{\ell=1}^\mn \Big|\Big(\theta\pwb_\ell - \theta\pwbkappa_\ell \Big) a^\pi_\ell \Big| \\
	& \leq \max\Big(\frac1m, \frac1n \Big) \Big\| \theta\pwb -\theta\pwbkappa \Big\|_1
	\leq \Big(\frac1m+\frac1n\Big) \epsilon.
\end{align*}
This directly implies
\begin{equation*}
	\Big| \Tdkpi(z)  - \Tdkcpi(z) \Big| \leq \Big(\frac1m+\frac1n\Big) \epsilon.
\end{equation*}

\subsubsection{High-probability upper bound for $|X_\wb|$ for each $(\wb)\in
  \Cgood$}

We need a following lemma from Corollary 2.2 in \citet{albert2019concentration}.

\begin{lemma}
\label{lem:bernstein-permuted-sum}
	Let \smash{$\{b_{ij}\}_{i,j=1}^N$} be an \smash{$N \times N$} array of numbers, and $\pi$ be a permutation distributed \smash{$\mathrm{Unif}(S_N)$}. Define a random variable \smash{$U_i = b_{i,\pi(i)}$} for \smash{$i\in\{1,\ldots,N\}$}. Note that $\pi$ is the only randomness in $U_i$. Assume that (i) \smash{$\E[U_i] = 0$}, (ii) \smash{$\Var[\sum_{i=1}^N U_i] \leq \sigma^2$}, and (iii) \smash{$|U_i| \leq b$} with probability $1$. Then for any \smash{$t\geq 0$},
	\[
		\P\bigg(\Big| \sum_{i=1}^n U_i\Big| \geq t\bigg)
		\leq 16 e^{1/16} \exp\Big(-\frac{t^2}{256(\sigma^2+bt)}\Big).
	\]
	This implies that
	\[
		\Big| \sum_{i=1}^n U_i\Big|  \leq 
		\sigma \sqrt{256\cdot \log\Big(\frac{16e^{1/16}}{\delta}\Big)} + 256\cdot b \log\Big(\frac{16e^{1/16}}{\delta}\Big)
		\quad\text{with probabtility at least}\;
		1-\delta.
	\]
\end{lemma}

In this subsection, we only consider $X\pwb$ with $\wb\in \Cgood$. Recall the definition \smash{$X\pwb = \sumell \theta\pwb_\ell a^\pi_\ell$}, and note that
\smash{$\E_{\pi\sim \unifsmn} [a^\pi_\ell\,|\, z] = m\cdot(1/m)+n\cdot(-1/n)=0$}. Write \smash{$X\pwb_\ell = \theta\pwb_\ell a^\pi_\ell$}.
To use the lemma we will check the following three statements, where $z$ is fixed:
\begin{enumerate}
	\item $\E_{\pi\sim\mathrm{Unif}(S_\mn)} [X\pwb_\ell \,|\, z] = 0$ for any $\ell$.
	\item $\Var(X\pwb\,|\, z)\leq \sigma^2$ where $\sigma=L^{k} \sqrt{2(\frac1m + \frac1n)}$.
	\item $| X\pwb_\ell | \leq b$ where $b=L^k(\frac1m+\frac1n)$ for any $\ell$.
\end{enumerate}

\paragraph{First statement.}

Note that $z$ is fixed. Easily observe that 
\[
	\E_{\pi\sim \unifsmn} [X\pwb_\ell \,|\, z] = \theta\pwb_\ell \cdot \E_{\pi\sim \unifsmn} [a^\pi_\ell \,|\, z] = 0.
\]

\paragraph{Second statement.}

We need to first calculate $\Var(a^\pi_\ell \,|\, z)$ and $\Cov(a^\pi_\ell, a^\pi_h \,|\, z)$ for $\ell\neq h$.
\begin{equation}
\label{eq:a-pi-var-cov}
	\begin{aligned}
	\Var(a^\pi_\ell \,|\, z) 
	&= \E_{\pi\sim \unifsmn} [(a^\pi_\ell)^2]  = \frac{m}{\mn}\cdot \frac{1}{m^2} + \frac{n}{\mn}\cdot \frac{1}{n^2} 
	= \frac{1}{mn}
	\\
	\Cov(a^\pi_\ell, a^\pi_h \,|\, z)
	&= \E_{\pi\sim \unifsmn} [a^\pi_\ell a^\pi_h \,|\, z]  
	 = \frac{{m\choose 2}}{{\mn \choose 2}} \cdot \frac{1}{m^2} +
		\frac{{mn}}{{\mn \choose 2}}\cdot \frac{-1}{mn} +
		\frac{{n\choose 2}}{{\mn \choose 2}}\cdot \frac{1}{n^2}\\
	& = -\frac{1/(\mn-1)}{mn}
	\end{aligned}
\end{equation}
We are now ready to observe $\Var(X\pwb \,|\, z)$.
\begin{align*}
	\Var(X\pwb \,|\, z)
	& = \Var\bigg(\sumell \theta\pwb_\ell a^\pi_\ell \,\Big|\, z\bigg) 
	 = \sumell \Var\Big(\theta\pwb_\ell a^\pi_\ell\Big) + \sum_{\ell \neq h} 	
		\Cov\Big(\theta\pwb_\ell  a^\pi_\ell,  \theta\pwb_h a^\pi_h \Big) \\
	& =\sumell (\theta\pwb_\ell)^2 \frac{1}{mn}
		- \sum_{\ell \neq h} \theta\pwb_\ell   \theta\pwb_h  \frac{1/(\mn-1)}{mn}
		\quad (\text{by \eqref{eq:a-pi-var-cov}}) \\
	&= \frac{1}{mn(m+n-1)} \bigg((m+n) \Big(\sumell (\theta\pwb^\ell)^2 \Big) - \Big(\sumell \theta\pwb_\ell\Big)^2 \bigg) \\
	&\leq \frac{\mn}{mn(m+n-1)} \bigg(\sumell (\theta\pwb_\ell)^2 \bigg)\\
	&= \frac{\mn}{mn(m+n-1)} (m+n) L^{2k} 
		\quad(\text{since $\theta\pwb_\ell = (w^\T z_\ell -b)_+^k \leq
   \|z_\ell\|_2^k$, as $b \geq 0$}) \\
	&\leq 2L^{2k}\Big(\frac1m + \frac1n\Big) \quad(\text{since $m,n\geq 1$}).
\end{align*}

\paragraph{Third statement.}

Easily check that \smash{$|X\pwb_\ell| = |\theta\pwb_\ell a^\pi_\ell| \leq (w^\T z_\ell -b)_+^k  \cdot \max(\frac1m, \frac1n) \leq  L^k(\frac1m+\frac1n)$}.

\paragraph{Conclusion.}

With the above three statements and Lemma \ref{lem:bernstein-permuted-sum}, we obtain the following high-probability upper bound for $|X\pwb|$ for each $\pwb\in \Cgood$, where $z$ is fixed: with probability at least $1-\delta$,
\begin{equation}
\label{eq:Xwb-high-prob-ub}
	\Big|X\pwb\Big| = \Big|\sumell X\pwb_\ell \Big| 
	\leq L^k \sqrt{\frac1m+\frac1n} \sqrt{512 \cdot \log\Big(\frac{16e^{1/16}}{\delta}\Big)}
	+ 256 L^k \Big(\frac1m+\frac1n\Big) \log\Big(\frac{16e^{1/16}}{\delta}\Big).
\end{equation}
As a reminder, the randomness in the above statements only comes from \smash{$\pi\sim\mathrm{Unif}(S_\mn)$}.

\subsubsection{High-probability upper bound for \smash{$\Tdkcpi$}}

First, for each \smash{$\pwb\in \Cgood$}, apply \eqref{eq:Xwb-high-prob-ub} with \smash{$\delta \leftarrow \delta / \Ngood$}. Then, apply union bound over all \smash{$\pwb\in \Cgood$} to get an upper bound for \smash{$\Tdkcgoodpi(z) = \max_{\pwb \in \Cgood} | X\pwb|$}. This gives us the following, where $z$ is fixed:
with probability at least $1-\delta$,
\[
	\Tdkcgoodpi(z) 
	\leq L^k \sqrt{\frac1m+\frac1n} \sqrt{512 \cdot \Big\{\log\Big(\frac{16e^{1/16}}{\delta}\Big) +\log \Ngood \Big\}}
	+ 256 L^k \Big(\frac1m+\frac1n\Big) \Big\{\log\Big(\frac{16e^{1/16}}{\delta}\Big) +\log \Ngood \Big\}.
\]
Recall that we have already got the upper bound for $\Ngood$ in Lemma \ref{lem:cgood}. Plugging in this upper bound, since \smash{$\frac1m + \frac1n \leq 1$}, we have the following statement:
there exists a constant $C''_{L,d,k}$ only depending on $(L,d,k)$ such that with probability $1-\delta$,
\begin{itemize}
	\item for $k=0$, 
	\begin{equation}
		\label{eq:tdkccircpi-high-prob-ub}
		\Tdkccircpi(z) \leq C''_{L,d,k} \sqrt{\frac1m + \frac1n} \Big\{\log\Big(\frac{1}{\delta}\Big) +\log \Big(1+\mn\Big) \Big\},
	\end{equation}
	\item and for $k\geq1$,
	\begin{equation}
		\label{eq:tdkcpi-high-prob-ub}
		\Tdkcpi(z) \leq C''_{L,d,k} \sqrt{\frac1m + \frac1n} \Big\{\log\Big(\frac{1}{\delta}\Big) +\log \Big(1+\frac\mn\epsilon\Big)\Big\}.
	\end{equation}
\end{itemize}

\subsubsection{High-probability upper bound for $\Tdkpi$}

Still $z$ is fixed. 
Combining the approximation error in Lemma \ref{lem:cgood} and \eqref{eq:tdkcpi-high-prob-ub} gives us that, with probability at least $1-\delta$,
\begin{itemize}
	\item for $k=0$,
	\[
		\Tdkpi(z) \leq 0 + C''_{L,d,k} \sqrt{\frac1m + \frac1n} \Big\{\log\Big(\frac{1}{\delta}\Big) +\log \Big(1+\frac\mn\epsilon\Big)\Big\}.
	\]
	\item for $k=1$,
	\[
		\Tdkpi(z) \leq \Big(\frac1m+\frac1n\Big)\epsilon + C''_{L,d,k} \sqrt{\frac1m + \frac1n} \Big\{\log\Big(\frac{1}{\delta}\Big) +\log \Big(1+\frac\mn\epsilon\Big)\Big\}.
	\]
	When we choose \smash{$\epsilon = (\frac1m + \frac1n)^{-1/2}$}, the above inequality is rewritten as
	\[
		\Tdkpi(z) \leq \sqrt{\frac1m+\frac1n} + C''_{L,d,k} \sqrt{\frac1m + \frac1n} \Big(\log\Big(\frac1\delta\Big) + \log(1+\mn)\Big).
	\]
\end{itemize}
Then, by plugging in $\delta = \alpha$, we can rewrite the above two bullet points in one statement as the following: there exists a constant $C_{L,d,k}$ only depending on $(L,d,k)$ such that 
\[
	\Tdkpi(z) \leq C_{L,d,k} \sqrt{\frac1m + \frac1n} \Big(\log\Big(\frac1\alpha\Big) + \log(m+n)\Big) 
	= t_{m,n,\alpha}	\quad\text{with probability at least}\; 1-\alpha,
\]
where we recall the definition of \smash{$t_{m,n,\alpha}$} in \eqref{eq:def-tmna}.
In other words, $t_{1-\alpha}(z) \leq t_{m,n,\alpha}$, which concludes the proof of Lemma \ref{lem:tdkpi-upper-bound}.
Therefore, this completes the proof of Theorem \ref{thm:non-asymptotic-perm}, presented in Appendix \ref{apdx:nonasymp-perm-proof}.


\section{Auxilliary technical lemmas}

\begin{lemma}
\label{lem:bound-power-of-relu}
    For $a,b \in \reals$ and integer $k \geq 1$,
    \begin{align*}
        &\big|a_+^k - b_+^k\big| \leq k(a_+^{k-1} + b_+^{k-1})|a-b|,
    \end{align*}
    with the convention that when $k-1 = 0$, we set $0^0 = 1$.
\end{lemma}

\begin{proof}[Proof of Lemma \ref{lem:bound-power-of-relu}]
	Without loss of generality, assume $a \leq b$.
	The function $x \mapsto x_+^k$ has derivative $kx_+^{k-1}$ almost everywhere (even under our convention for when $k-1=0$).
	For every point $x \in [a,b]$,
	we have $|kx_+^{k-1}| \leq k(a_+^{k-1} + b_+^{k-1})$, because $|x_+| \leq \max\{a_+, b_+\}$.
	Thus, integrating the derivative along the path between $a$ and $b$ gives the result.
\end{proof}

\begin{lemma}
\label{lem:rbv-ftn-continuous}
	Assume $\mu \in \cM(\nonnegRD)$.
	Then, for any $0 \leq m \leq k - 1$ and any function $c(w,b):\nonnegRD\rightarrow\R$ which is bounded by \smash{$|c(w,b)| \leq C$} for some constant $C$, we have
	\begin{equation}
		x  \mapsto \int_\nonnegRD \Big(c(w,b)(w^\T x-b)_+^{k-m}\Big)\;\de\mu(w,b)
	\end{equation} 
	is continuous at every $x \in \reals^d$.
\end{lemma}

\begin{proof}[Proof of Lemma \ref{lem:rbv-ftn-continuous}]
	Denote $f(x) = \int_\nonnegRD c(w,b)(w^\T x-b)_+^{k-m}\;\de\mu(w,b)$.
	Let $|\mu|$ be the total variation measure of the signed measure $\mu$. Note that $\|\mu\|_\TV = \|\,|\mu|\,\|_\TV$.
	For $x,y \in \reals^d$, 
	\begin{align*}
	    \big|f&(x)-f(y)\big| \le \int_\nonnegRD |c(w,b)|\, \big|(w^\T x - b)_+^{k-m} - (w^\T y - b)_+^{k-m} \big| \;\de|\mu|(w,b)
	    \\
	    &\le (k-m)\int_\nonnegRD |c(w,b)| \, \big((w^\T x - b)_+^{k-m-1} + (w^\T y - b)_+^{k-m-1}\big)|w^\T(x-y)|\;\de|\mu|(w,b)
	    \\
		&\le (k-m)\| x - y \|\,\int_\nonnegRD |c(w,b)|\, \big((\|w\|_2\|x\|_2)^{k-m-1} + (\|w\|_2\|y\|_2)^{k-m-1}\big) \;\de\mu(w,b) 
		\\
		&\le (k-m)\| x - y \|\, C\, \|\mu\|_\TV\,  \big(\|x\|_2^{k-m-1} + \|y\|_2^{k-m-1}\big) ,
	\end{align*}
	where in the second inequality we have used Lemma \ref{lem:bound-power-of-relu}.
	Since the right-hand side goes to $0$ as $y \rightarrow x$, the proof is complete.
\end{proof}


\section{Sensitivity analysis: log transform}
\label{supp:sensitivity-log}

In Equation \eqref{eq:opt-problem}, which is the basis for the experiments in the main  
paper, we use a log transform of the IPM term in the criterion. The current
section compares the performance of (local) optimizers of this problem, which we
call the ``log'' problem, for short, to those of
\begin{equation}
\label{eq:opt-problem-2}
\min_{\substack{f = f_{(a_j,w_j,b_j)_{j=1}^N} \\ b_j \geq 0, \, j = 1,\dots,N}}
\, - \frac{1}{kN} \big| \empP(f) - \empQ(f) \big| 
+ \frac{\lambda}{k} \bigg( \frac{1}{N} \sum_{i=1}^N |a_i| \|w_i\|_2^k \bigg)^2, 
\end{equation}
which we call the ``no-log'' problem, for short. In the above display, we can
see that there is no log transform on the IPM term, but in addition, the penalty
on the \smash{$\RTVk$} seminorm has been squared. This is done because without
such a transformation, one can show that the problem (IPM term plus
\smash{$\RTVk$} seminorm) does not attain its infimum, unless $\lambda$ is
chosen very carefully so that the infimum in zero. By squaring the penalty, the
criterion in \eqref{eq:opt-problem-2} is coercive, which guarantees (since it is
also continuous) that it will attain its infimum. In both
\eqref{eq:opt-problem} and \eqref{eq:opt-problem-2}, we fix $\lambda = 1$, and 
use \texttt{torch.optim.Adam} with \texttt{betas} parameter $(0.9, 0.99)$. We
explore the behavior for different learning rates and numbers of
iterations. Throughout, we focus on the ``var-one'' setting described in 
Table \ref{table:experiments-setup}, and use $N=10$ neurons. The summary of our
findings is as follows.    

\begin{itemize}
\item Optimizing the ``log'' problem typically results in a larger IPM value
  compared to the ``no-log'' problem, especially for larger $k,d$, and
  especially under the null. 

\item Optimizing the ``log'' and ``no-log'' problems generally results in
  similar ROC curves, however, for larger $k,d$, the ``no-log'' ROC curves can
  actually be slightly better. This actually appears to be driven by the fact
  that ``no-log'' optimization is \emph{relatively worse} in terms of the IPM
  values obtained under the null, which gives it a slightly greater separation
  and slightly higher power. 

\item The ``log'' problem is more robust to the choice of learning rate, both
  in terms of the IPM values and ROC curves obtained, especially for larger
  $k,d$. 

\item The ``log'' problem typically shows faster convergence (number of
  iterations required to approach a local optimum), especially for larger
  $k,d$. 
\end{itemize}

Figures \ref{fig:mmd-comparison-lognolog}--\ref{fig:roc-curves-diff-iter}
display the results from our sensitivity analyses. In each one, the figure
caption explains the salient points about the setup and takeaways. We remark
that results for $k=4$ (not included for brevity) show qualitatively similar but
even more extreme behavior.

\begin{figure}[p]
\centering
\vspace{-10pt}
\begin{subfigure}[b]{\textwidth}
\centering
\includegraphics[width=0.75\textwidth]{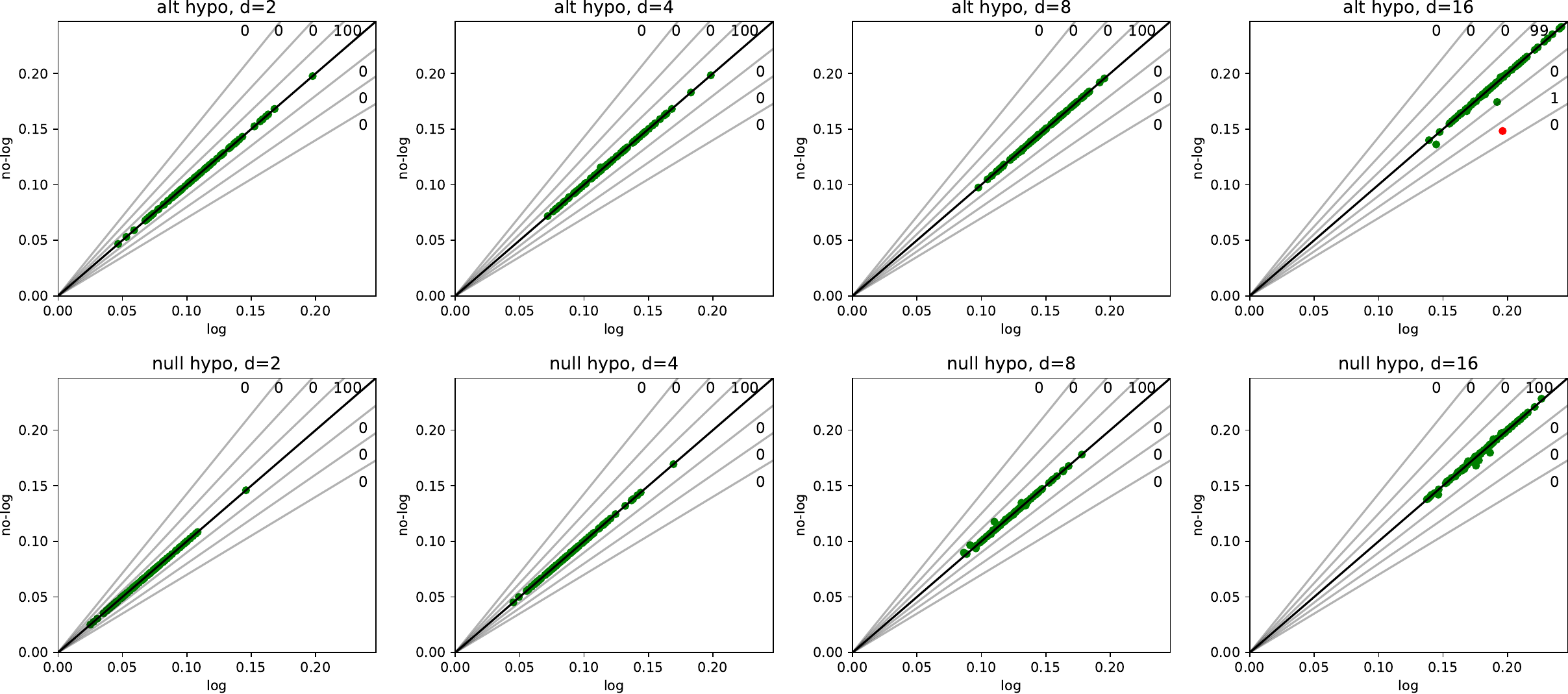}
\caption{$k=1$}
\label{fig:mmd-scatter-lognolog-k1-iter1200}
\end{subfigure}
	
\medskip
\begin{subfigure}[b]{\textwidth}
\centering
\includegraphics[width=0.75\textwidth]{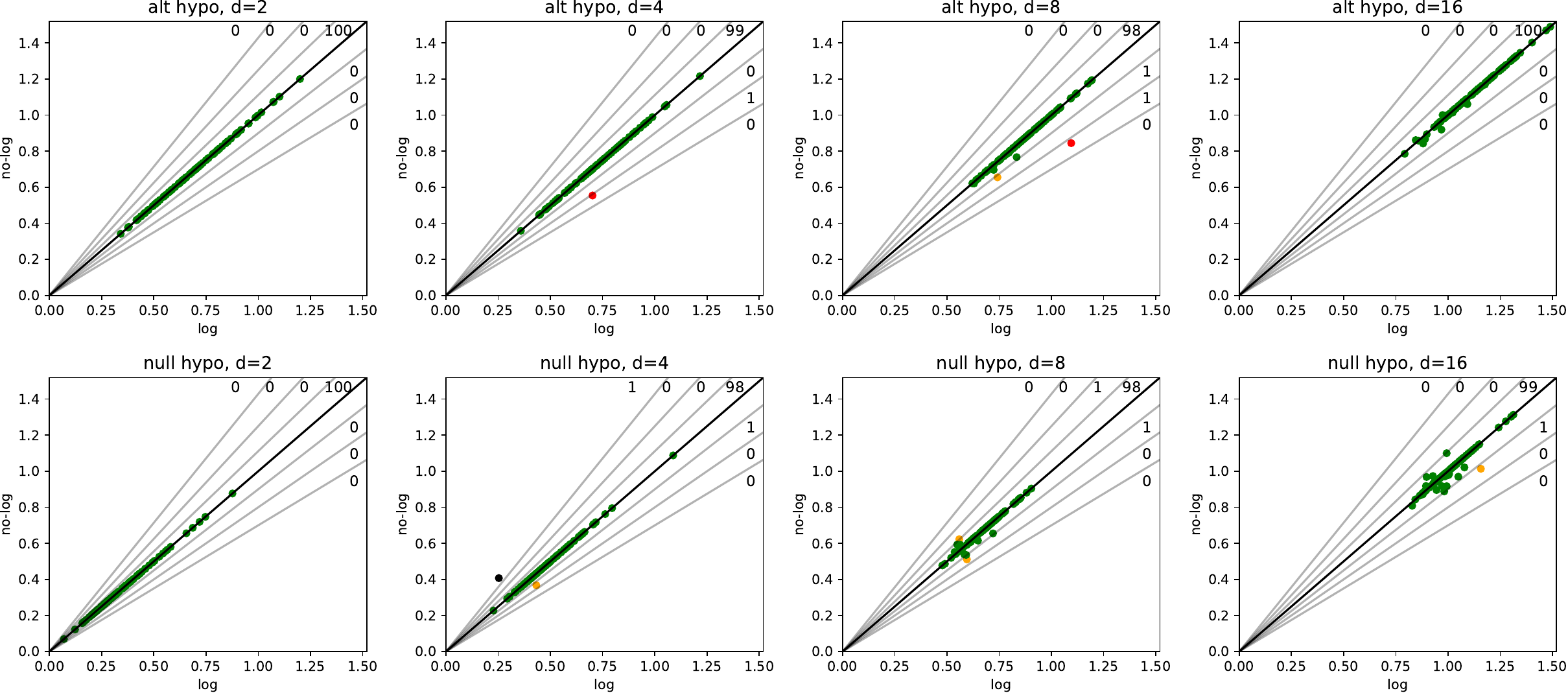}
\caption{$k=3$}
\label{fig:mmd-scatter-lognolog-k3-iter1200}
\end{subfigure}
\caption{IPM values obtained by optimizing the ``log'' (x-axis) and ``no-log''
  (y-axis) problems. Each point represents a set of samples drawn from
  $\popP,\popQ$, and the result of running $T=1200$ iterations with learning
  rate $0.01$, for each criterion. (This learning rate was chosen to be
  favorable to the ``no-log'' problem.) Points below the diagonal mean that the
  ``log'' criterion results in a larger IPM value, which we see is especially
  prominent for larger $k,d$, and more prominent under the null.} 
\label{fig:mmd-comparison-lognolog}

\bigskip
\includegraphics[width=0.75\textwidth]{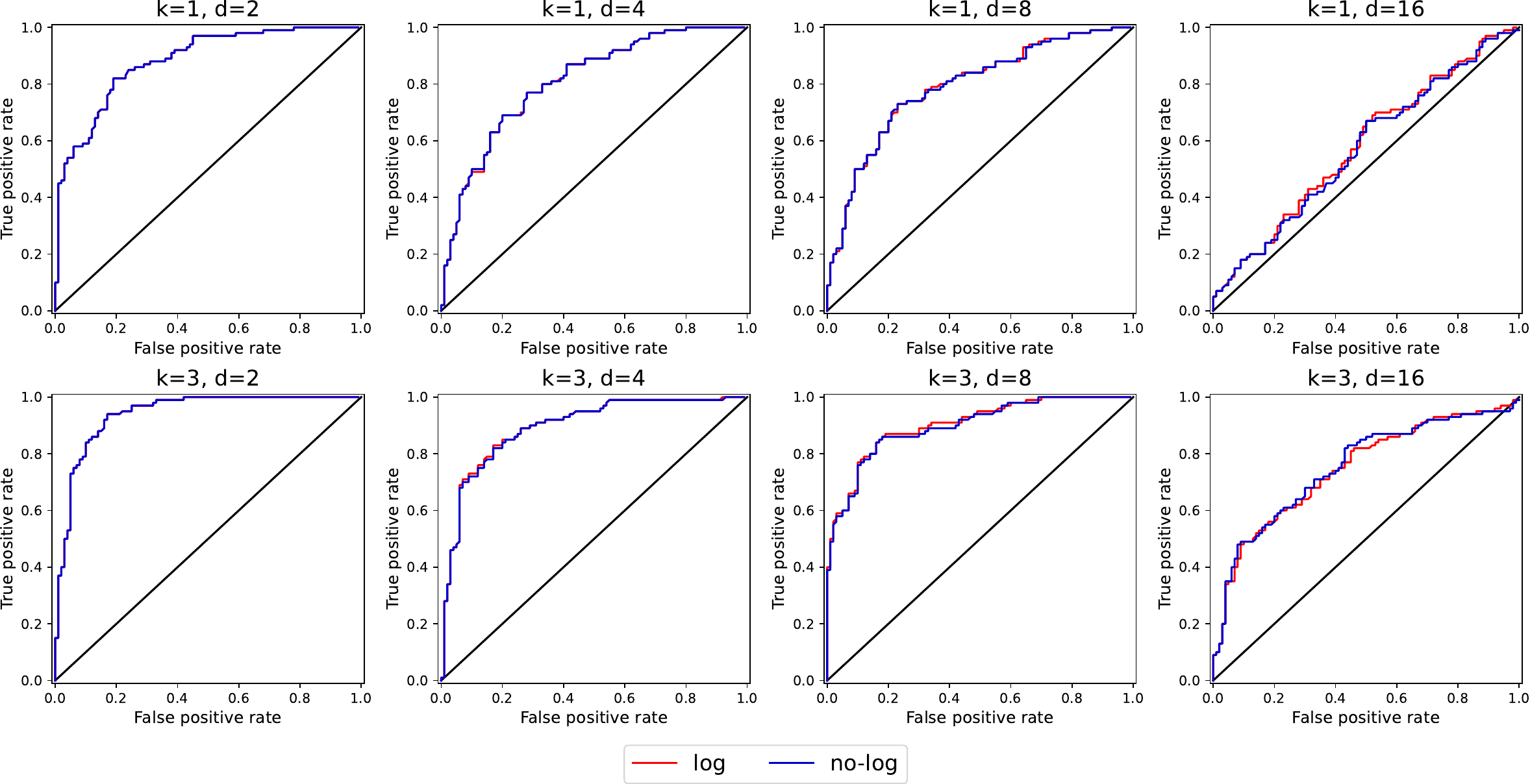}
\caption{ROC curves from the ``log'' and ``no-log'' problems, for the same data
  as in Figure \ref{fig:mmd-comparison-lognolog}. They are very similar except
  for the largest $k,d$ pair, where the ``no-log'' ROC curve is slightly
  better. Inspecting Figure \ref{fig:mmd-scatter-lognolog-k3-iter1200}, this is
  likely due to the fact that the null IPM values here are relatively smaller
  (poorer optimization).}
	\label{fig:roc-lognolog-k13-lr001-iter1200}
\end{figure}

\begin{figure}[p]
\centering
\begin{subfigure}[b]{\textwidth}
\centering
\includegraphics[width=0.8\textwidth]{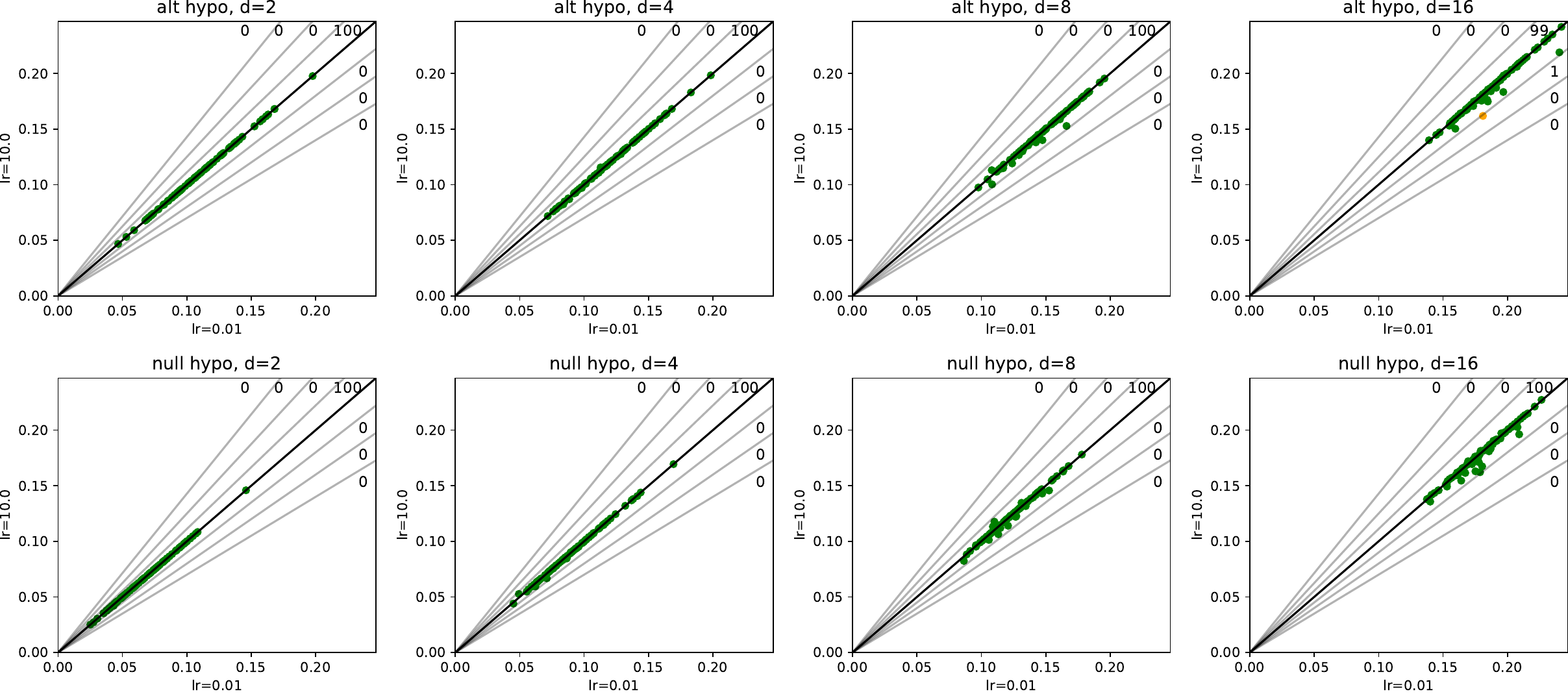}
\caption{$k=1$, log}
\end{subfigure}

\medskip
\begin{subfigure}[b]{\textwidth}
\centering
\includegraphics[width=0.8\textwidth]{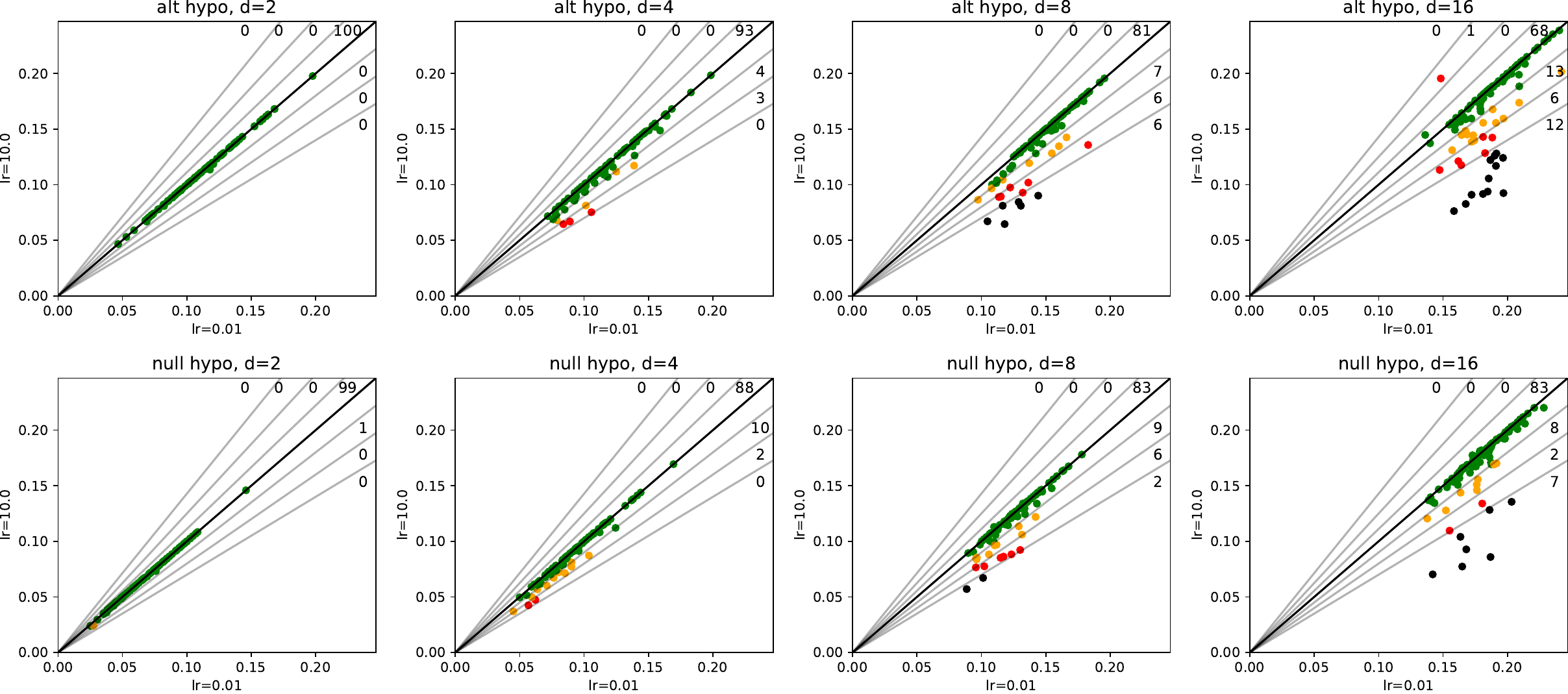}
\caption{$k=1$, no-log}
\end{subfigure}

\caption{IPM values corresponding to learning rates $0.01$ (x-axis) and $10$
  (y-axis), for the same data as in Figure
  \ref{fig:mmd-comparison-lognolog}. There is little to no difference for the
  ``log'' problem (points near the diagonal), and a much bigger difference for
  the ``no-log'' problem, where in most cases the larger learning rate is worse 
  (points below the diagonal).}  
\label{fig:robust-lr-scatter-1}
\end{figure}

\begin{figure}[p]
\begin{subfigure}[b]{\textwidth}
\centering
\includegraphics[width=0.8\textwidth]{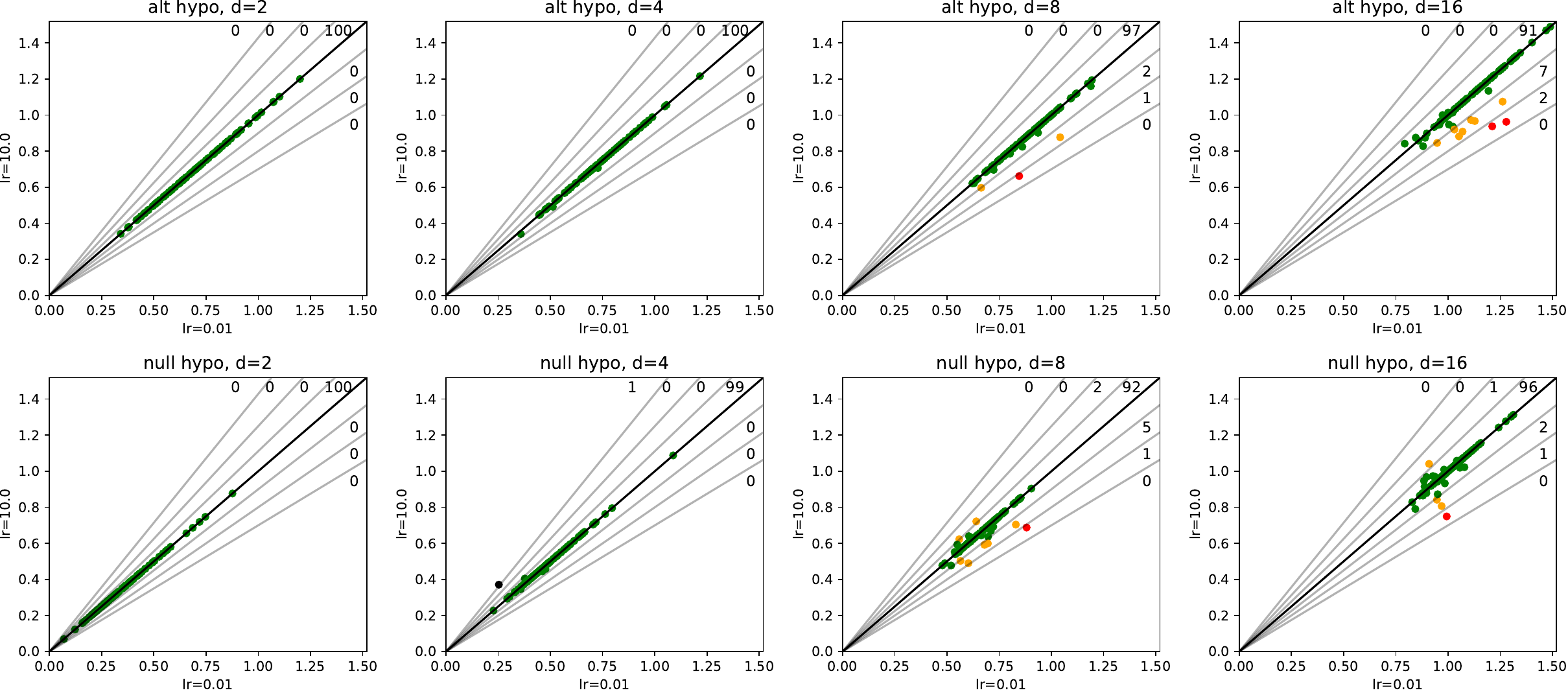}
\caption{$k=3$, log}
\end{subfigure}	

\medskip
\begin{subfigure}[b]{\textwidth}
\centering
\includegraphics[width=0.8\textwidth]{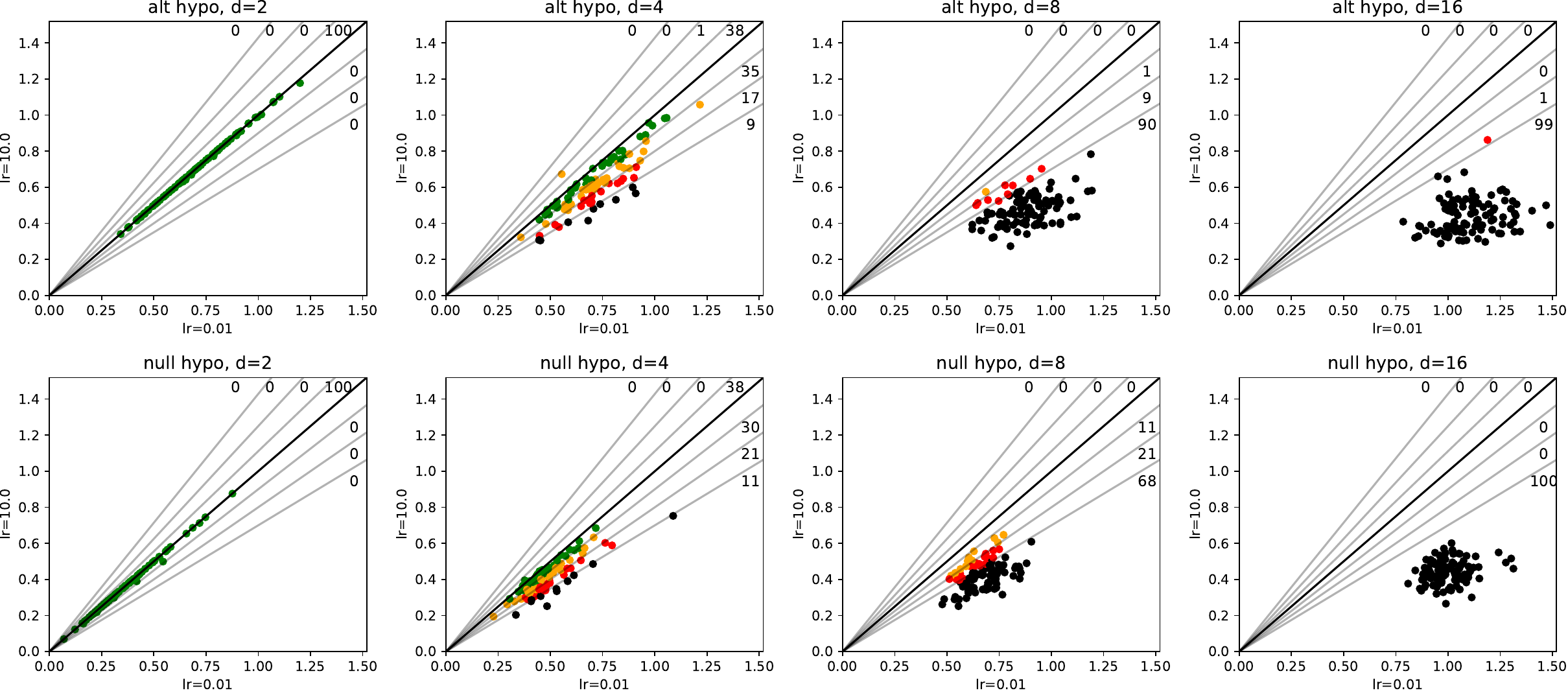}
\caption{$k=3$, no-log}
\end{subfigure}	
	
\caption{IPM values corresponding to learning rates $0.01$ (x-axis) and $10$
  (y-axis), for the same data as in Figure
  \ref{fig:mmd-comparison-lognolog}. There is little to moderate for the ``log''
  problem (points around the diagonal), and a much bigger difference for the
  ``no-log'' problem, where in nearly all cases the larger learning rate is
  clearly much worse (points below the diagonal).}     
\label{fig:robust-lr-scatter-2}
\end{figure}

\begin{figure}[p]
\centering
\begin{subfigure}[b]{\textwidth}
\centering
\includegraphics[width=0.8\textwidth]{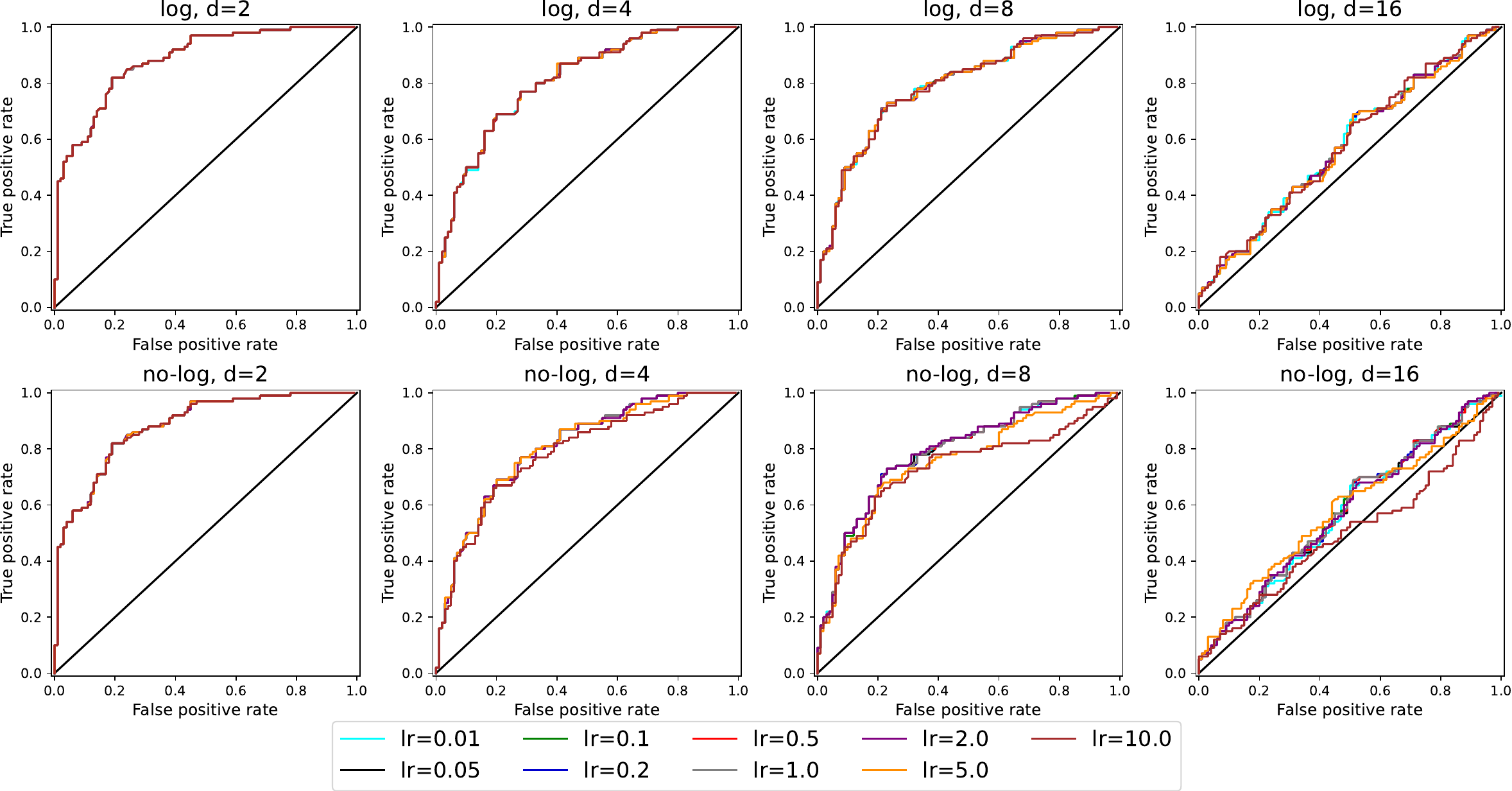}
\caption{$k=1$}
\label{fig:robust-lr-k1-iter1200}
\end{subfigure}
	
\medskip
\begin{subfigure}[b]{\textwidth}
\centering
\includegraphics[width=0.8\textwidth]{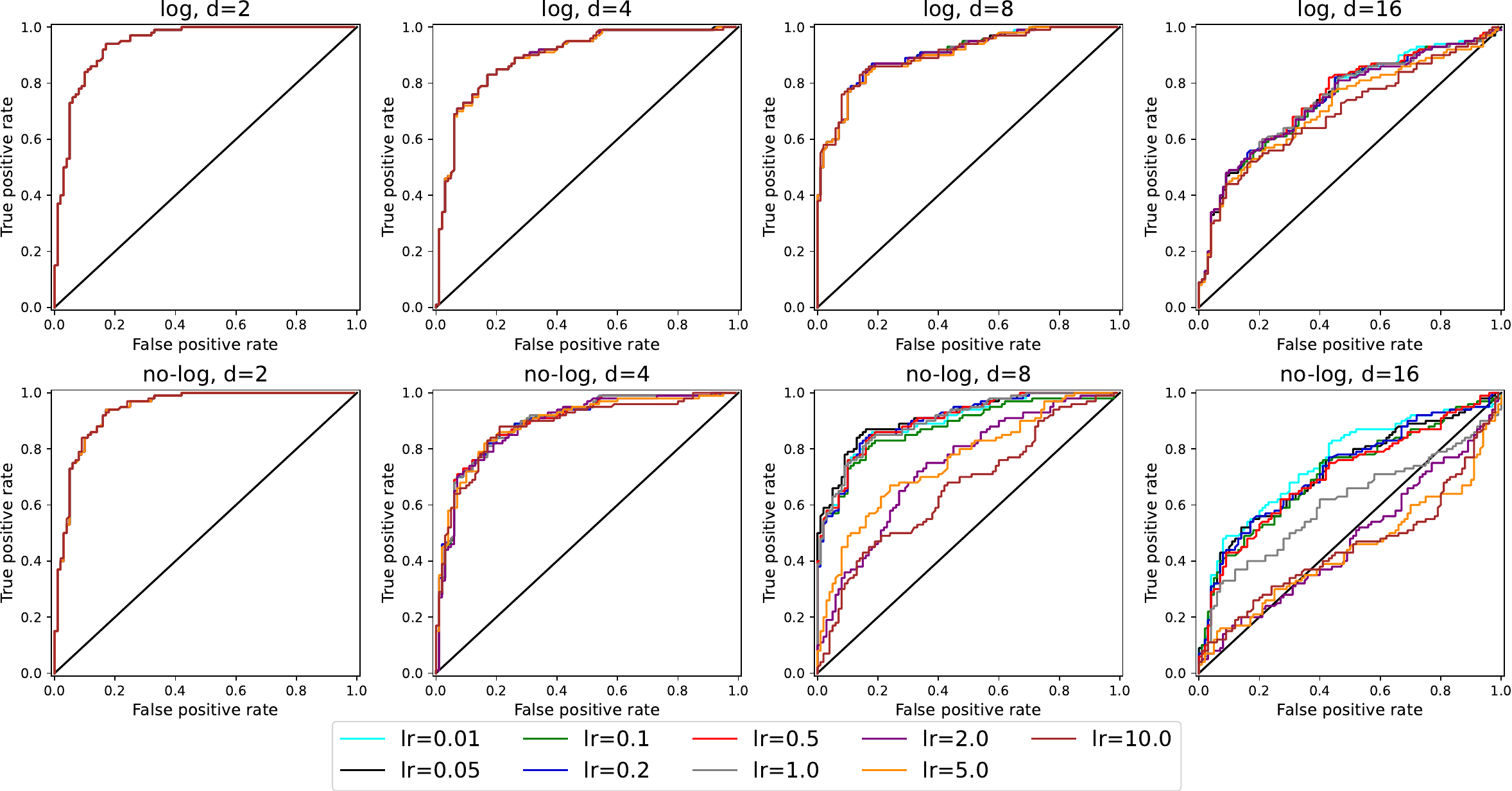}
\caption{$k=3$}
\label{fig:roubst-lr-k3-iter1200}
\end{subfigure}
	
\caption{IPM values corresponding to learning rates from $0.01$ to $10$, for
  the same data as in Figure \ref{fig:mmd-comparison-lognolog}. In all cases,
  the ``log'' problem results in much more stable ROC curves with respect to the
  choice of learning rate, especially for larger $k,d$.}
\label{fig:robust-lr-roc}
\end{figure}

\begin{figure}[p]
\centering
\begin{subfigure}[b]{\textwidth}
\centering
\includegraphics[width=0.8\textwidth]{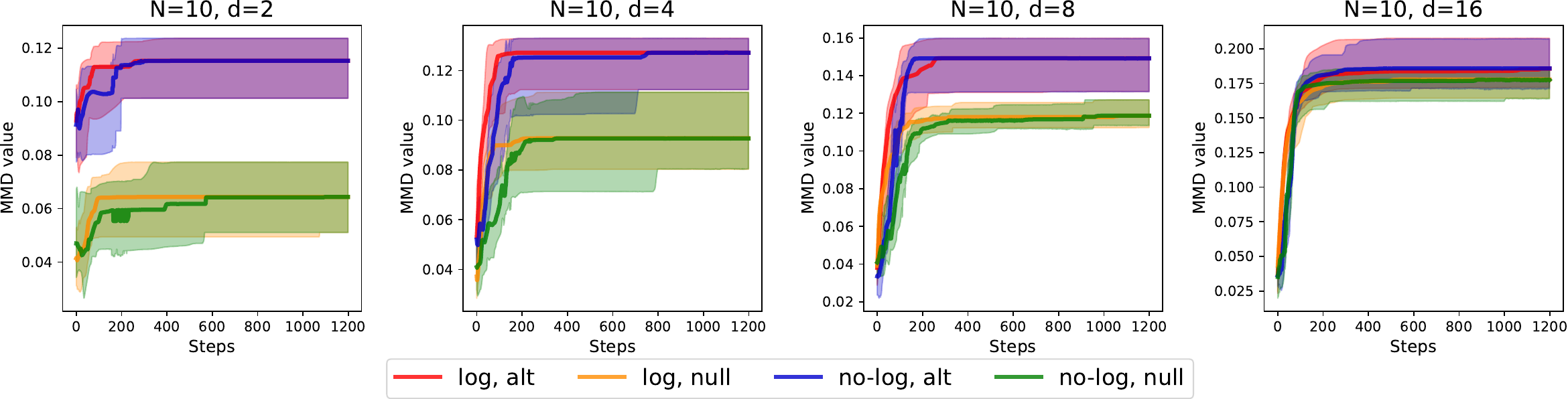}
\caption{$k=1$, $\text{lr}=0.01$}
\end{subfigure}

\medskip
\begin{subfigure}[b]{\textwidth}
\centering
\includegraphics[width=0.8\textwidth]{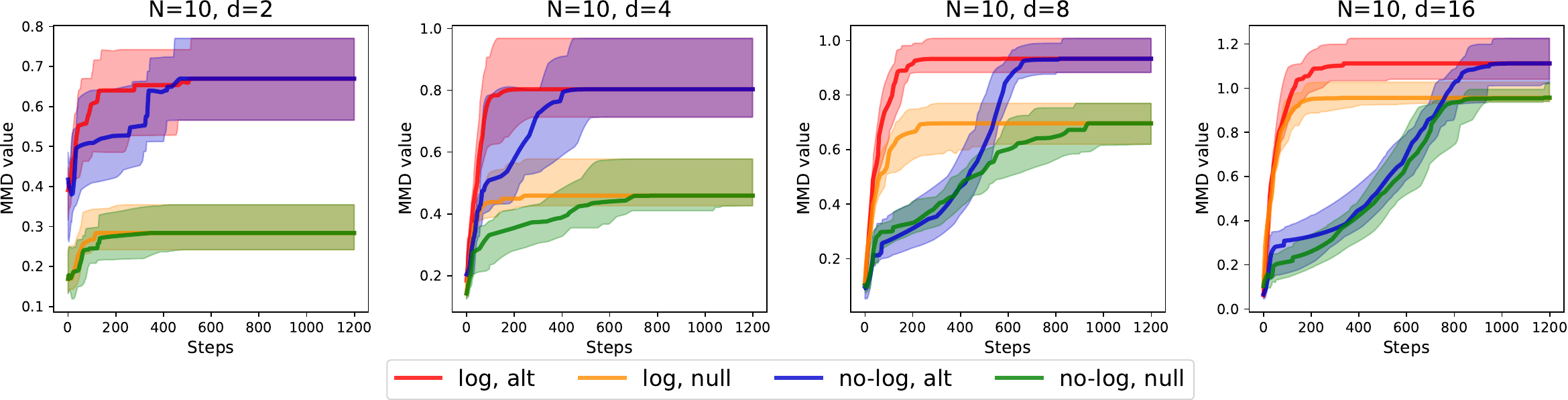}
\caption{$k=3$, $\text{lr}=0.01$}
\end{subfigure}

\medskip
\begin{subfigure}[b]{\textwidth}
\centering
\includegraphics[width=0.8\textwidth]{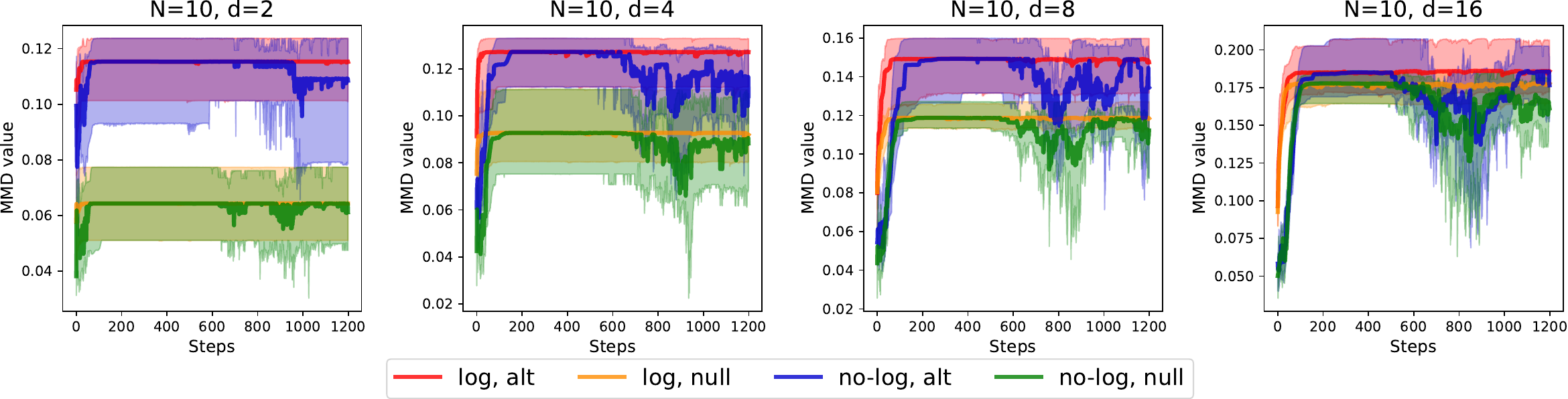}
\caption{$k=1$, $\text{lr}=0.5$}
\end{subfigure}

\medskip
\begin{subfigure}[b]{\textwidth}
\centering
\includegraphics[width=0.8\textwidth]{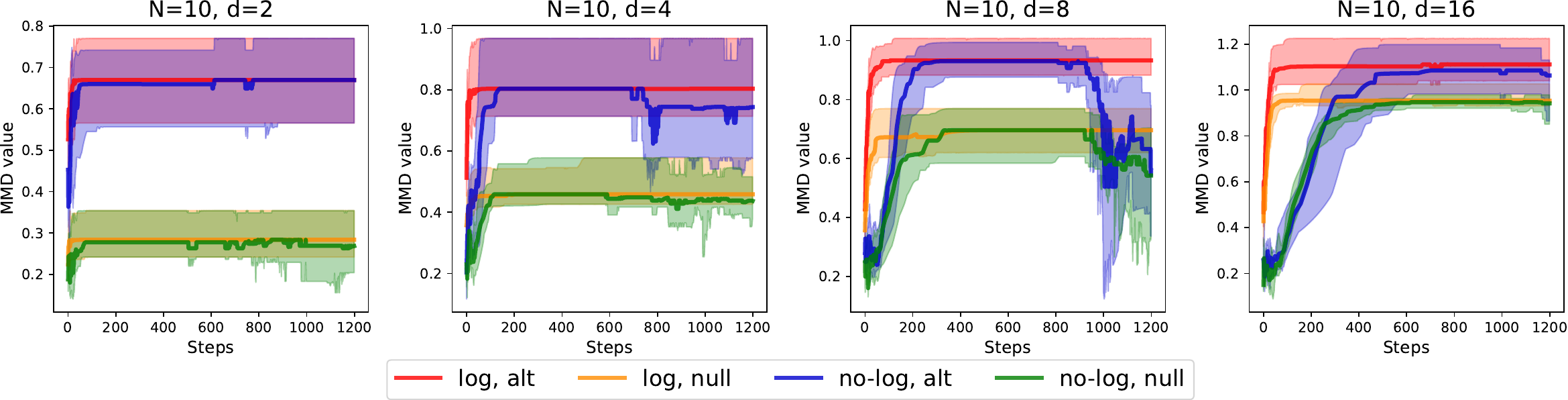}
\caption{$k=3$, $\text{lr}=0.5$}
\end{subfigure}

\caption{IPM value as a function of iteration for the ``log'' and ``no-log''
  problems. (cf. the $y$-axis says the MMD instead of the IPM.) This is done over 20 repetitions (draws of samples from
  $\popP,\popQ$); solid lines represent the median and shaded areas the
  interquartile range (computed over repetitions), at each iteration. For
  smaller $k,d$, the ``log'' and ``no-log'' convergence speeds are fairly
  similar, but for larger $k,d$, the ``no-log'' convergence speed is clearly
  slower. For the larger learning rate, we can also see that in many instances
  the ``no-log'' iterations fail to converge, and bounce up and down in IPM
  value, without staying at a local optimum.}
\label{fig:convergence-curve-k-lr}
\end{figure}

\begin{figure}[p]
\centering
\begin{subfigure}[b]{\textwidth}
\centering
\includegraphics[width=0.8\textwidth]{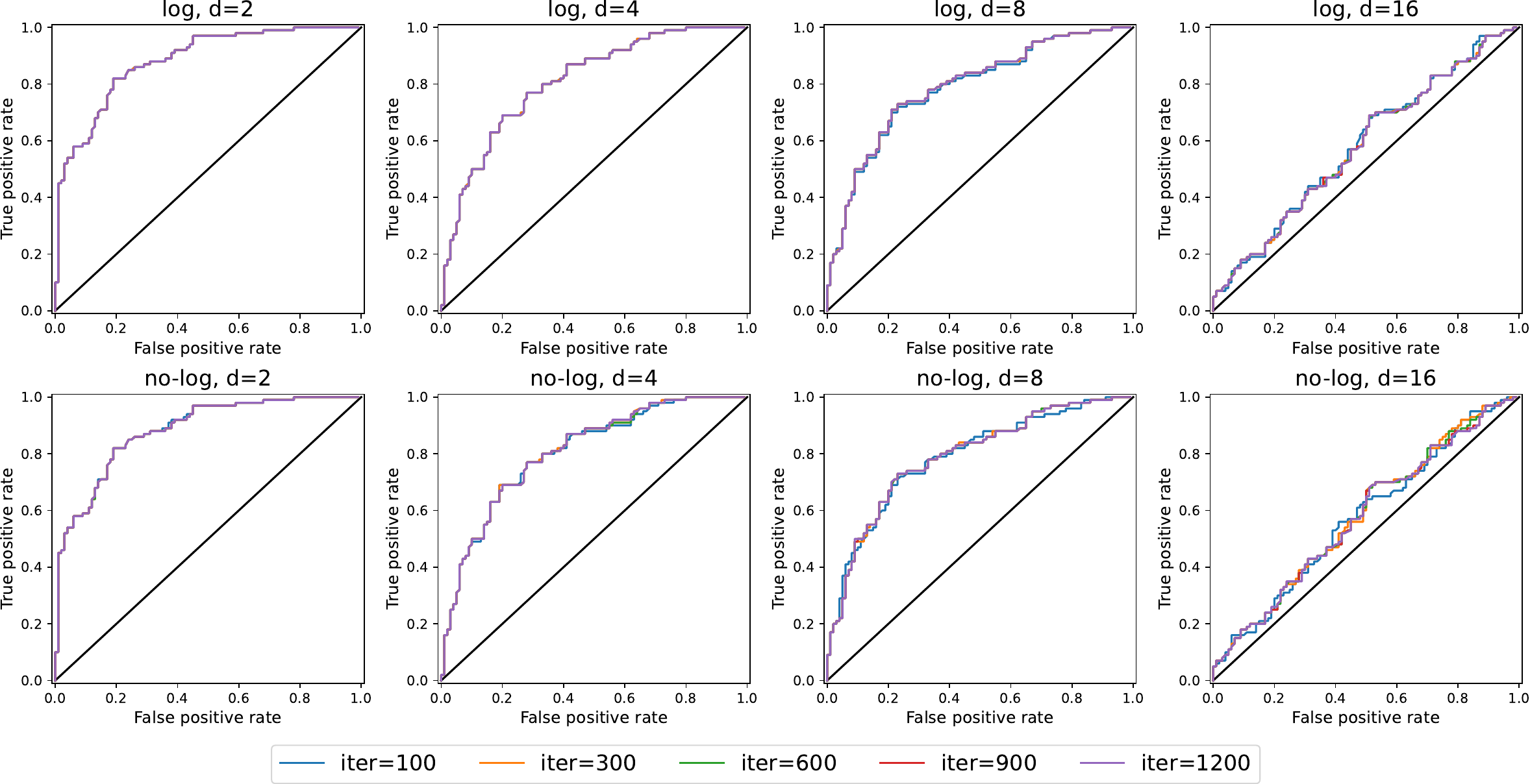}
\caption{$k=1$}
\end{subfigure}

\medskip
\begin{subfigure}[b]{\textwidth}
\centering
\includegraphics[width=0.8\textwidth]{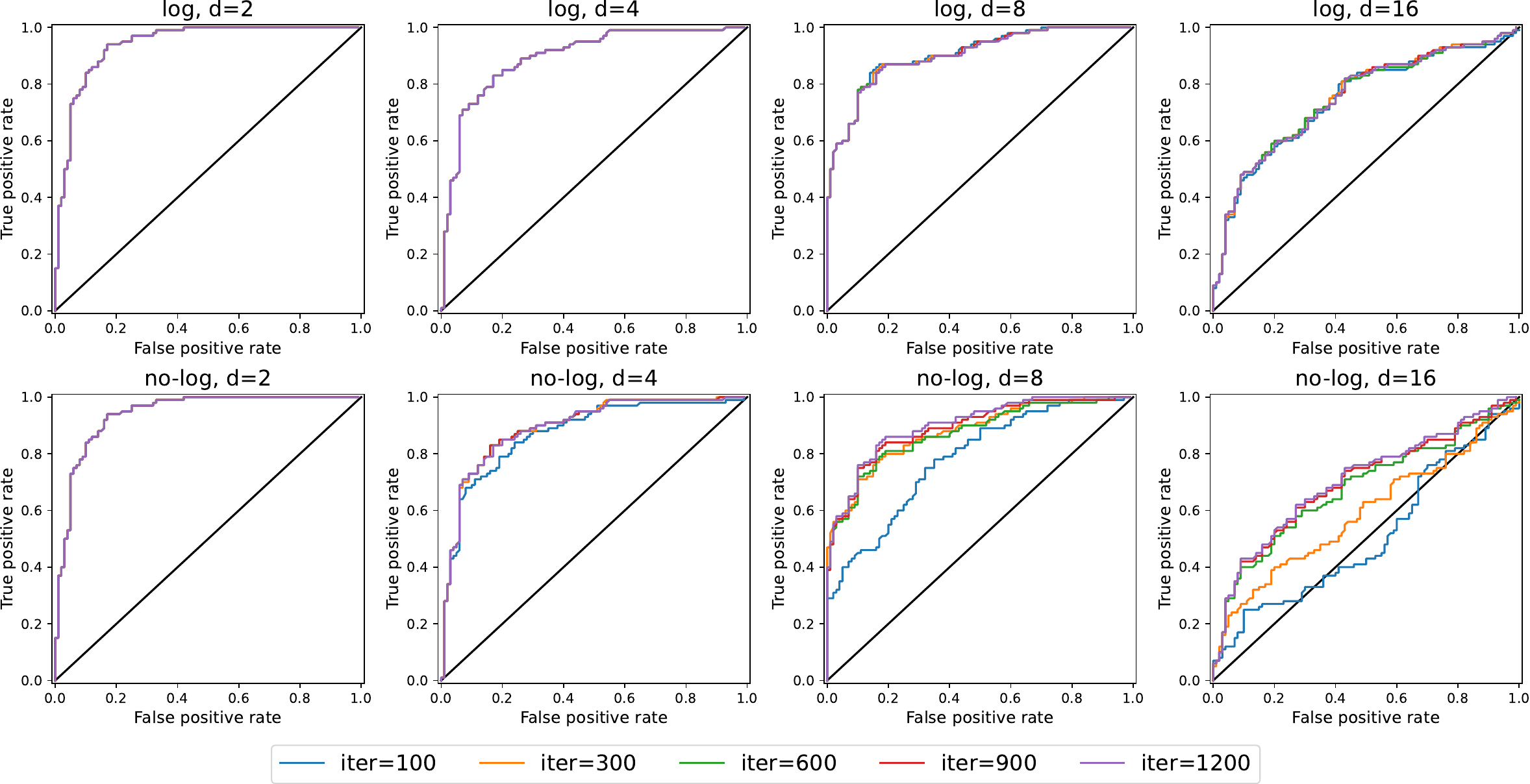}
\caption{$k=3$}
\end{subfigure}

\caption{ROC curves for the ``log'' and ``no-log'' problems, at different
  iteration numbers, over the same data in Figure
  \ref{fig:convergence-curve-k-lr}. The ``log'' curves are quite robust to the
  choice of the number of iterations, whereas the ``no-log'' ones are not, especially for
  larger $k,d$. In other words, the ``log'' problem typically shows faster convergence (number of
  iterations required to approach a local optimum), especially for larger $k,d$. } 
\label{fig:roc-curves-diff-iter}
\end{figure}

\section{Sensitivity analysis: number of neurons}
\label{supp:sensitivity-N}

In this section, we now turn to investigate the role of the number of neurons
$N$. We stick with the ``log'' problem \eqref{eq:opt-problem} as in the
experiments in the main text (with the previous section providing evidence that
the log transform leads to greater robustness in many regards). The setup is
essentially the same as that used in the last section, except that we vary the
number of neurons from $N=1$ to $N=20$. The summary of our findings is as
follows.  

\begin{itemize}
\item Using a larger number of neurons often improves the achieved IPM value,
  especially for larger $k,d$; however, going past $N=10$ does not seem to make
  a big difference in our experiments.

\item Using a larger number of neurons generally results in a better ROC curve, 
  especially for larger $k,d$; again, going past $N=10$ does not seem to make 
  a big difference in our experiments.

\item Using a larger number of neurons can result in faster convergence, though
  not dramatically.

\item Using a larger number of neurons is typically more robust to the choice of 
  learning rate, both in terms of the IPM values and ROC curves obtained,
  especially for larger $k,d$. 
\end{itemize}

Figure \ref{fig:N-scatter-k3-lr05-N210}--\ref{fig:robust-lr-varyingN} display  
the results from our sensitivity analyses. As before, each figure caption
presents the salient points about the setup and takeaways, and we remark that
results for $k=4$ (not included for brevity) show qualitatively similar but even
more extreme behavior.  
	
\begin{figure}[p]
\centering
\includegraphics[width=0.8\textwidth]{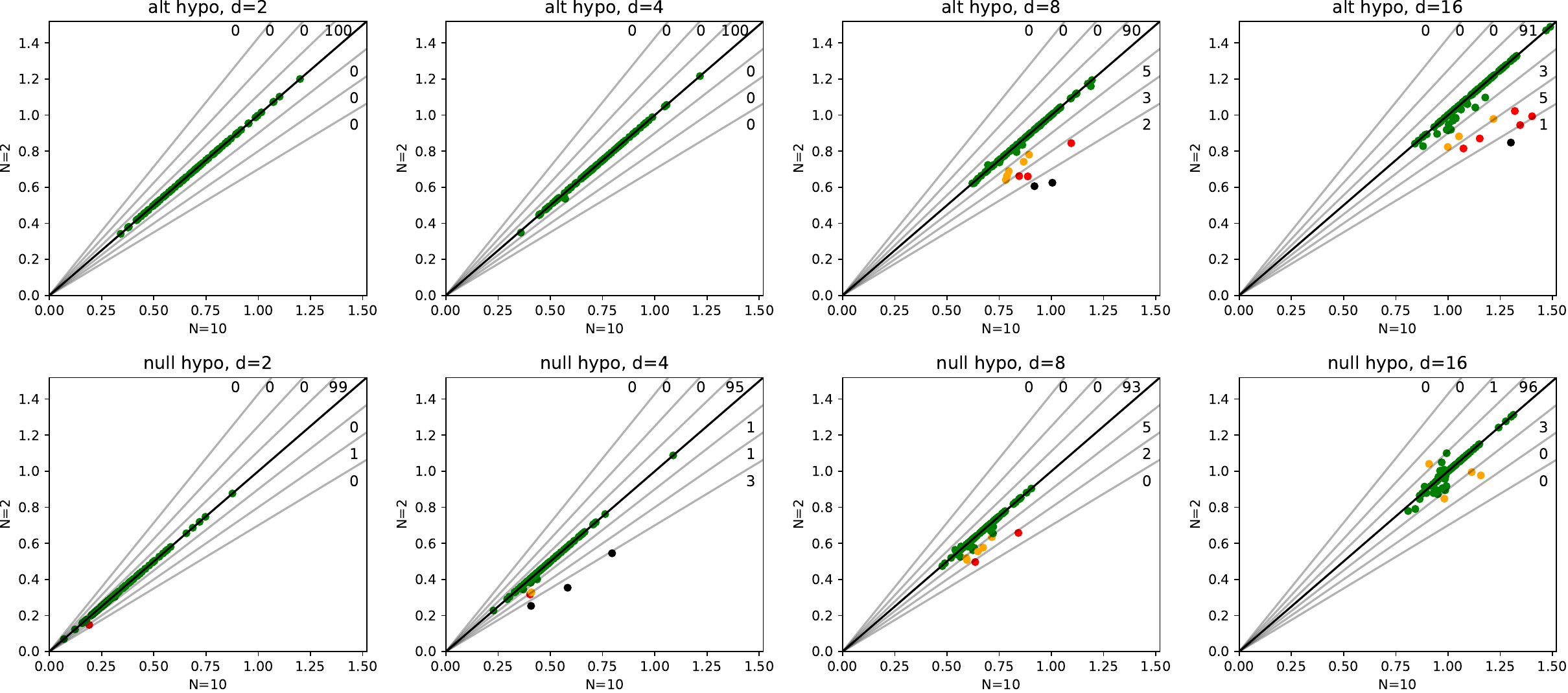}
\caption{IPM values when $N=10$ (x-axis) and $N=2$ (y-axis). Each point
  represents a set of samples drawn from $\popP,\popQ$, and the result of
  running $T=1200$ iterations with learning rate $0.5$, when $k=3$. Points below
  the diagonal mean that $N=10$ results in a larger IPM value, 
  which we see is especially prominent for larger $k,d$, and more prominent
  under the null.}  
\label{fig:N-scatter-k3-lr05-N210}

\bigskip
\includegraphics[width=0.8\textwidth]{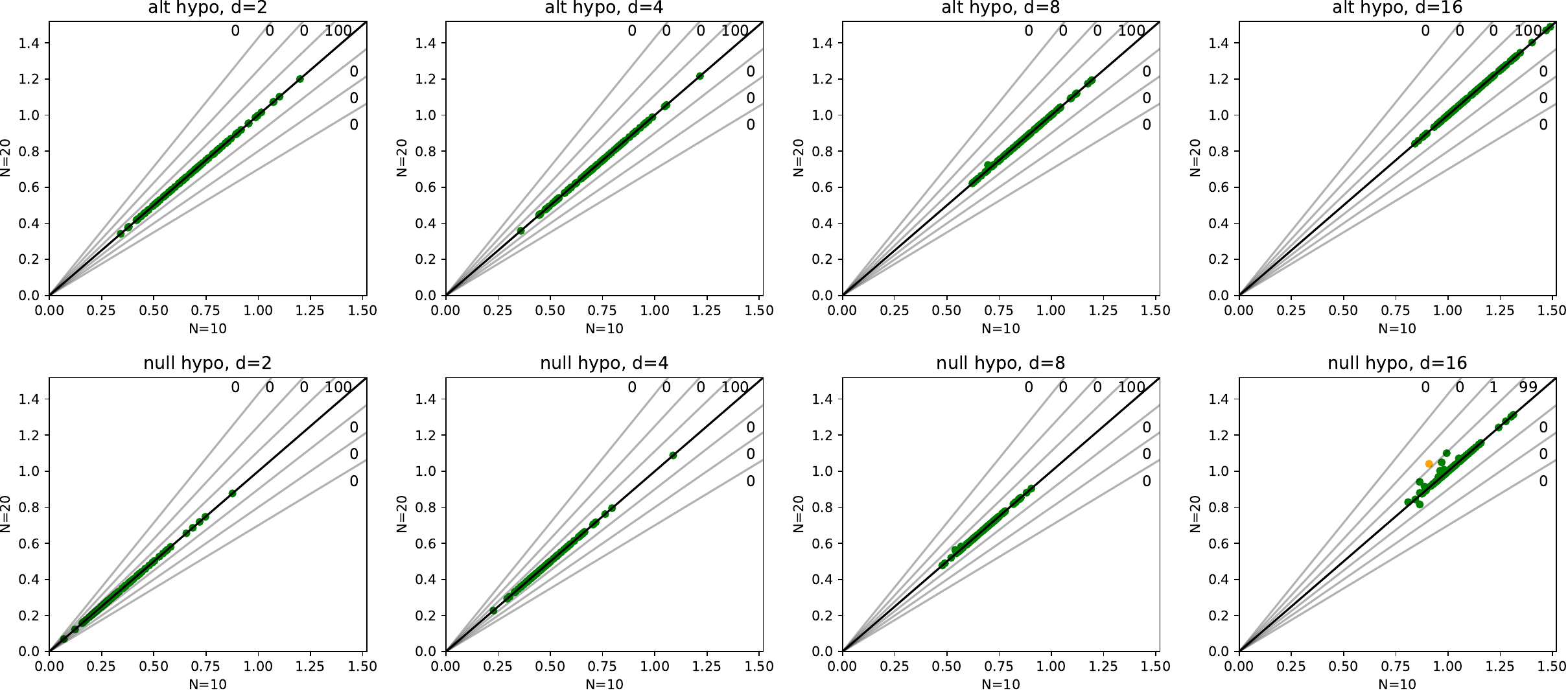}
\caption{IPM values when $N=10$ (x-axis) and $N=20$ (y-axis). The setup is as 
  above. Here we see no clear improvement in moving from $N=10$ to $N=20$
  neurons.} 
\end{figure}

\begin{figure}[p]
\centering
\includegraphics[width=0.8\textwidth]{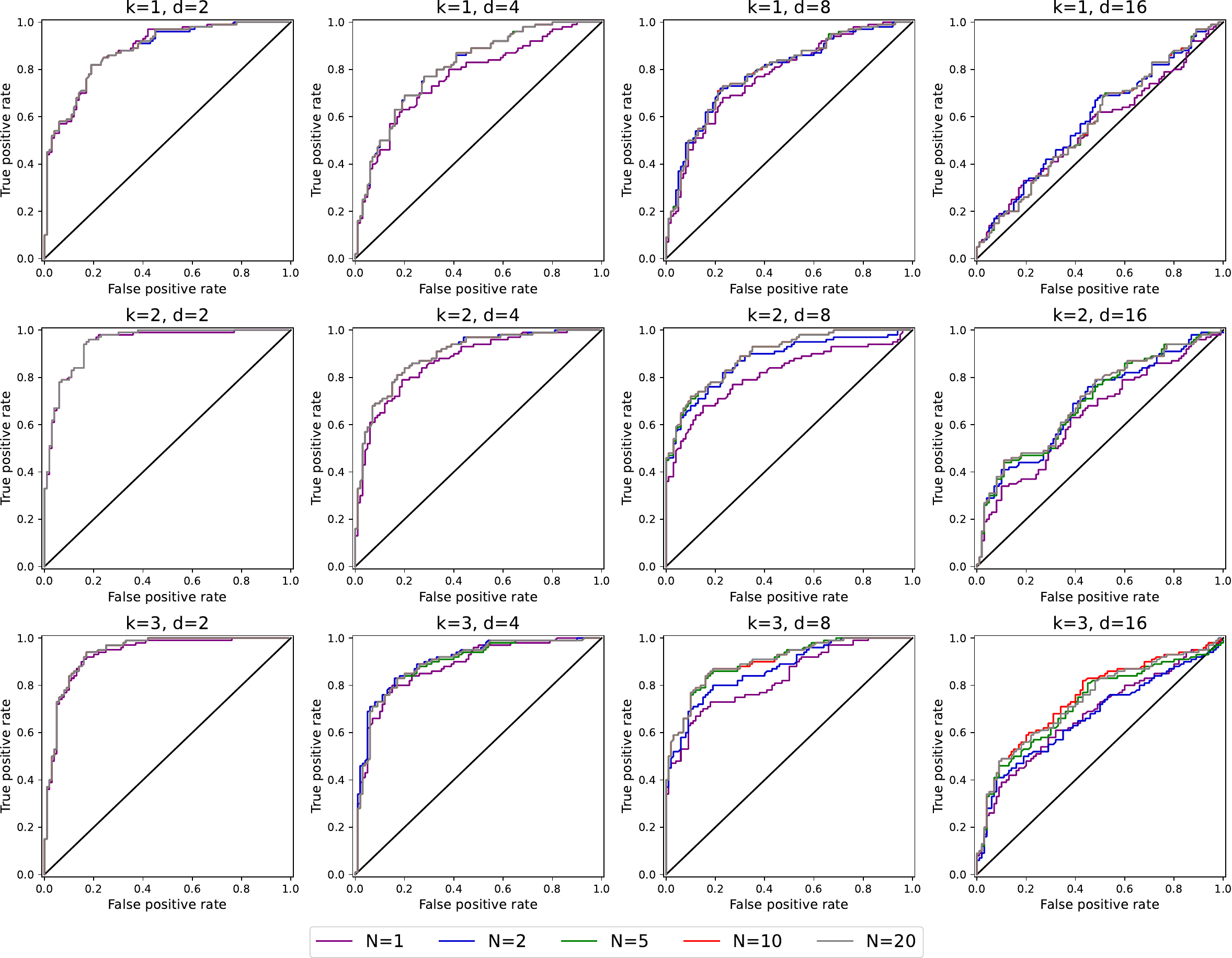}
\caption{ROC curves as the number of neurons varies from $N=1$ to $N=20$, for
  the same data as in Figure \ref{fig:N-scatter-k3-lr05-N210}. The ROC curves
  look generally better for a larger number of neurons, especially for larger
  $k,d$.} 
\label{fig:N-ROC-lr05-iter1200}
\end{figure}

\begin{figure}[p]
\centering
\begin{subfigure}[b]{\textwidth}
\centering
\includegraphics[width=0.75\textwidth]{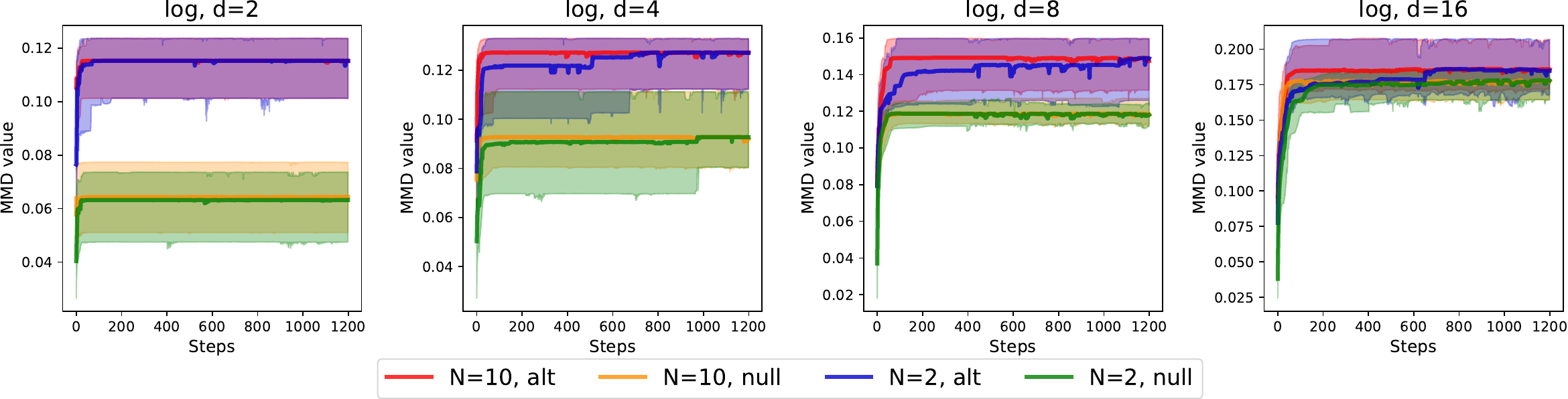}
\caption{$k=1$, $\text{lr}=0.5$}
\end{subfigure}

\medskip
\begin{subfigure}[b]{\textwidth}
\centering
\includegraphics[width=0.75\textwidth]{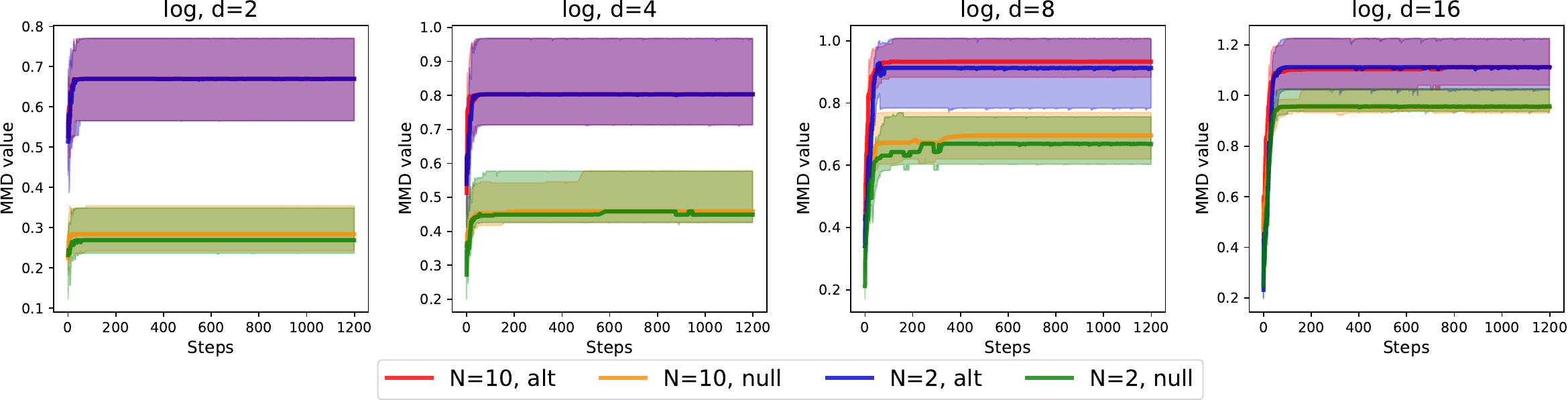}
\caption{$k=3$, $\text{lr}=0.5$}
\end{subfigure}
\caption{IPM value as a function of iteration for $N=10$ and $N=2$ neurons. This
  is carried out over 20 repetitions (draws of samples from $\popP,\popQ$);
  solid lines represent the median and shaded areas the interquartile range
  (computed over repetitions), at each iteration. The convergence speed is
  sometimes slightly faster with $N=10$ neurons but the differences are not
  large.}  
\label{fig:convergence-curve-k-lr-N210}
\end{figure}

\begin{figure}[p]
\centering
\includegraphics[width=0.75\textwidth]{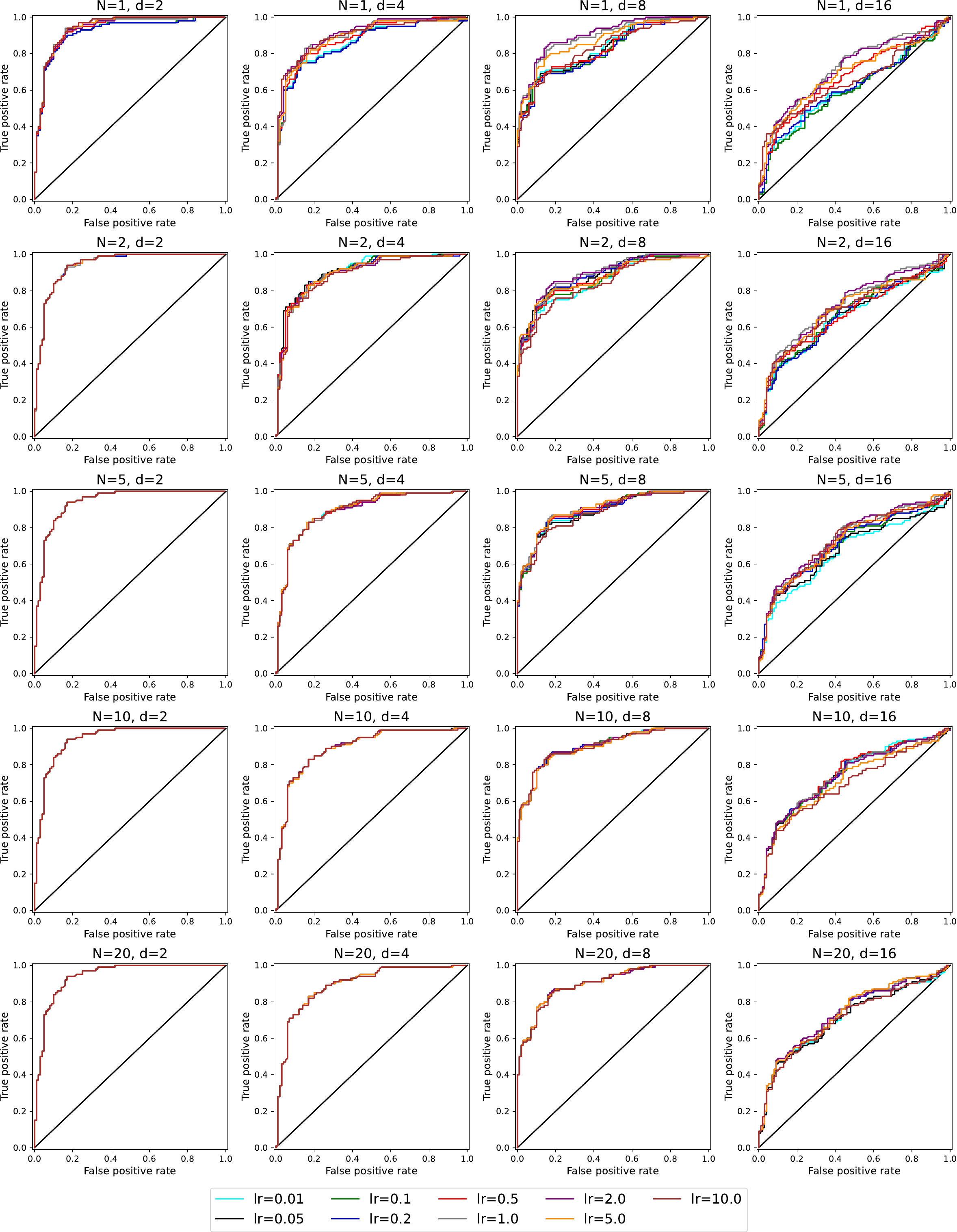}
\caption{ROC curves as the number of neurons varies from $N=1$ to $N=20$, and
  for varying learning rates, for the same data as in Figure
  \ref{fig:N-scatter-k3-lr05-N210}. A larger number of neurons is typically more
  robust to the choice of learning rate, especially for larger $d$.} 
\label{fig:robust-lr-varyingN}
\end{figure}


\section{Comparison to polynomial kernel MMD}
\label{supp:mmd-poly}

Figure \ref{fig:roc-mmd-poly} compares the RKS tests to the polynomial kernel
MMD tests using the data as in Figure \ref{fig:roc-setup-per-dimension} in the
main paper. We have chosen linear, quadratic, and cubic degree polynomials by
(rough) analogy to the use of $k=1,2,3$ in RKS. The colors of the RKS test of
degree $k$ and kernel MMD with a polynomial kernel of degree $k$ are chosen to
match. However, we can see that these two frameworks produce tests with
generally quite different behaviors: for example, compare the behavior of RKS
with $k=1$ to kernel MMD with linear kernel, in the ``t-coord'' setting.   

Altogether, the behavior of kernel MMD with polynomial kernel is as expected: it 
performs quite well when $P \neq Q$ differ according to the alternative (moment
difference) that is anticipated by the test, and is otherwise essentially
powerless. 

\begin{figure}[p]
\vspace{-10pt}
\includegraphics[width=\textwidth]{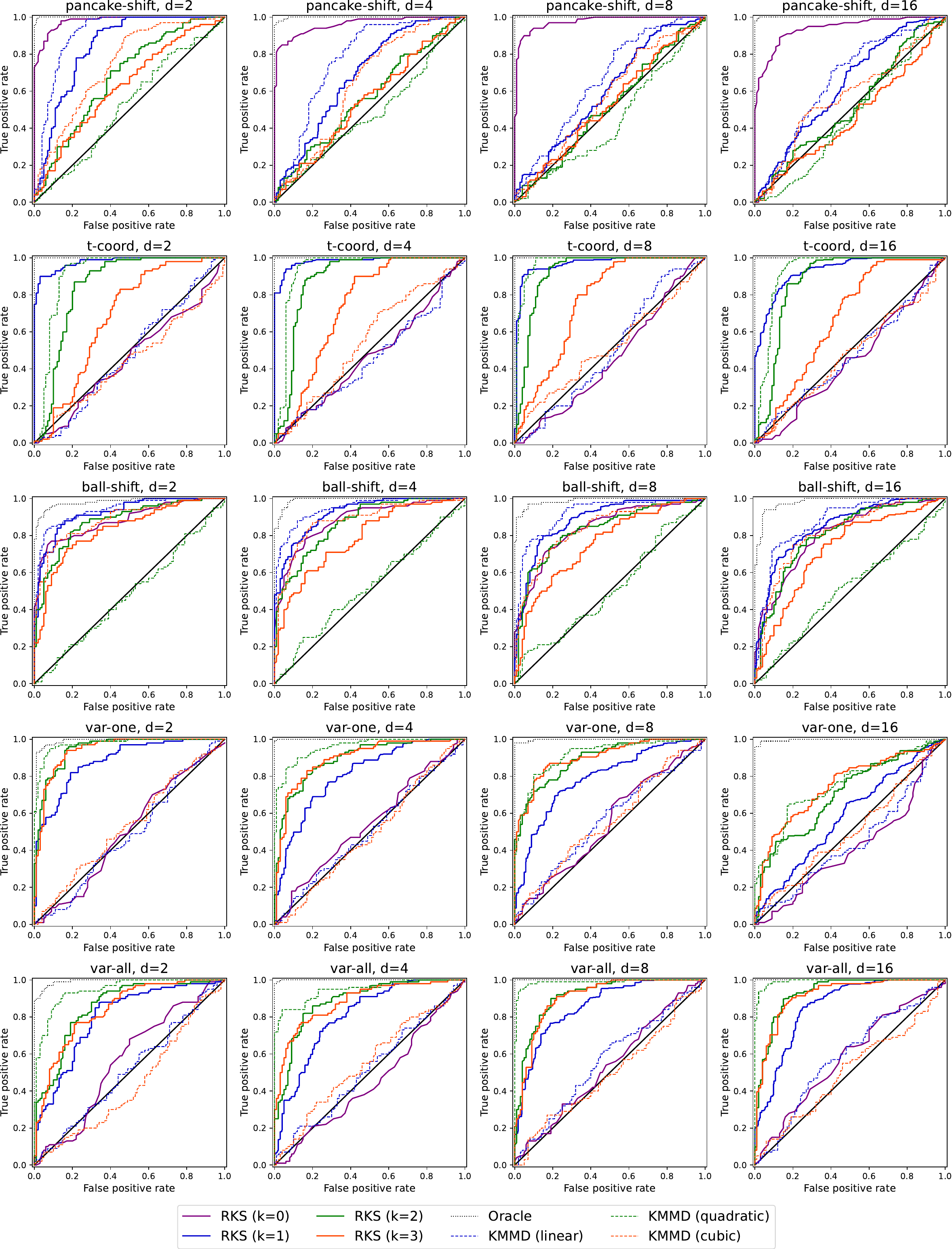}
\caption{ROC curves across the same experimental settings as in Figure 
  \ref{fig:roc-setup-per-dimension}, but now with kernel MMD using a polynomial 
  kernel of linear, quadratic, and cubic degrees. The ROC curve from the
  likelihood ratio test is also shown, indicating the best possible performance
  in each case (it depends on oracle knowledge of $\popP, \popQ$).}          
\label{fig:roc-mmd-poly}
\end{figure}

\end{document}